\documentclass{article}

\usepackage[colorlinks=True,citecolor=blue,urlcolor=blue,pagebackref=true,backref=true,linkcolor=blue,filecolor=blue]{hyperref}
\renewcommand*{\backrefalt}[4]{%
    \ifcase #1 \footnotesize{(Not cited.)}%
    \or        \footnotesize{(Cited on page~#2.)}%
    \else      \footnotesize{(Cited on pages~#2.)}%
    \fi}

\usepackage[accepted]{icml2024}
\usepackage{url}
\usepackage{natbib}
\usepackage{xcolor}
\usepackage{hyperref}
\definecolor{mydarkblue}{rgb}{0,0.08,0.45}
\hypersetup{
  colorlinks = true,
  urlcolor = blue, 
  linkcolor = blue,
  citecolor = blue,
  pdfborder={0 0 0},
}
\usepackage{url}
\usepackage{amsmath,amssymb,amsthm}
\usepackage{mathtools}
\usepackage{physics}
\usepackage{xspace}
\usepackage[capitalize,noabbrev]{cleveref}
\usepackage{nicefrac}
\usepackage{comment}
\usepackage{enumitem}
\usepackage{bm}
\usepackage{bbm}
\usepackage{autonum}
\usepackage{scalerel}
\usepackage{stmaryrd}
\usepackage{wrapfig}
\usepackage{booktabs} %
\usepackage{multirow}
\usepackage{etoc}
\usepackage{subcaption}
\usepackage{graphicx} %

\Crefname{algocf}{Alg.}{Algs.}

\crefname{appendix}{App.}{Apps.}
\crefname{subsubsubappendix}{App.}{Apps.}
\crefname{equation}{}{}
\crefname{lemma}{Lem.}{Lems.}
\crefname{theorem}{Thm.}{Thms.}
\crefname{corollary}{Cor.}{Cors.}
\crefname{algorithm}{Alg.}{Algs.}

\crefname{example}{Ex.}{Exs.}
\crefname{section}{Sec.}{Secs.}
\crefname{table}{Tab.}{Tabs.}
\crefname{remark}{Rem.}{Rems.}
\crefname{definition}{Def.}{Defs.}
\crefname{proposition}{Prop.}{Props.}
\crefname{myremark}{Rem.}{Rems.}
\crefname{mylemma}{Lem.}{Lems.}
\crefname{mydefinition}{Def.}{Defs.}
\crefname{myproposition}{Prop.}{Props.}
\crefname{mycorollary}{Cor.}{Cors.}
\crefname{myassumption}{Assum.}{Assums.}
\crefname{assumption}{Assum.}{Assums.}
\crefname{figure}{Fig.}{Figs.}
\crefname{myexample}{Ex.}{Exs.}
\crefname{enumi}{}{}
\crefname{name}{}{} %

\newcommand{\defeq}{\triangleq}
\DeclareMathOperator{\supp}{supp}

\providecommand{\argmin}{\mathop\mathrm{arg min}}

\newcommand{\hess}{{\nabla^2}}

\newtheorem{theorem}{Theorem}
\newtheorem{corollary}{Corollary}
\newtheorem{proposition}{Proposition}
\newtheorem{lemma}{Lemma}
\newtheorem{assumption}{Assumption}
\newtheorem{definition}{Definition}
\theoremstyle{definition}
\newtheorem{remark}{Remark}
\newtheorem{example}{Example}

\newcommand{\eps}{\epsilon}
\newcommand{\vareps}{\varepsilon}

\newcommand{\cB}{\mathcal{B}}
\newcommand{\cC}{\mathcal{C}}
\newcommand{\cD}{\mathcal{D}}
\newcommand{\cE}{\mathcal{E}}

\newcommand{\cH}{\mathcal{H}}

\newcommand{\cL}{\mathcal{L}}

\newcommand{\cN}{\mathcal{N}}

\newcommand{\cS}{\mathcal{S}}

\newcommand{\cX}{\mathcal{X}}
\newcommand{\cY}{\mathcal{Y}}

\newcommand{\bR}{\mathbb{R}}
\newcommand{\bN}{\mathbb{N}}
\newcommand{\bE}{\mathbb{E}}

\newcommand{\bP}{\mathbb{P}}
\newcommand{\bS}{\mathbb{S}}

\newcommand{\fg}{\mathfrak{g}}
\newcommand{\fp}{\mathfrak{p}}
\newcommand{\fq}{\mathfrak{q}}
\newcommand{\ft}{\mathfrak{t}}

\newcommand{\uC}{\textup{C}}
\newcommand{\uH}{\textup{H}}
\newcommand{\uT}{\textup{T}}
\newcommand{\uCpp}{\textup{C++}}
\newcommand{\tC}{\mathtt{C}}
\newcommand{\tI}{\mathtt{I}}
\newcommand{\tJ}{\mathtt{J}}
\newcommand{\tS}{\mathtt{S}}
\renewcommand{\P}{\mathbb{P}}

\newcommand{\simplex}[1][n-1]{\Delta_{#1}}
\newcommand{\poly}{\textup{poly}}

\newcommand{\out}{\textup{out}}
\newcommand{\bQ}{\mathbb{Q}}
\newcommand{\wis}{\xi}

\newcommand{\Mis}{M}
\newcommand{\iid}{\textrm{i.i.d.\@}\xspace}
\newcommand{\iidup}{\textup{i.i.d.\@}\xspace}
\newcommand{\trunc}{{\textup{resid}}}

\newcommand{\truncstrat}{{\textup{sr}}}

\newcommand{\Var}{\operatorname{Var}}
\newcommand{\perm}{\operatorname{perm}}
\newcommand{\snis}{\textup{SNIS}}
\newcommand{\funcspace}{\mathfrak{F}}
\newcommand{\ellinf}{\ell_\infty}
\newcommand{\boundedop}{\mathfrak{B}}
\newcommand{\ball}[1]{\cB_{#1}}
\newcommand{\balleuc}{\cB_{2}}
\newcommand{\idmap}[2]{{#1}\hookrightarrow{#2}}
\newcommand{\scorebound}{\mathfrak{S}_{p}}

\newcommand{\cvrnum}[1]{\cN_{#1}}
\newcommand{\cvrC}{\mathfrak{C}_d}

\newcommand{\cvrd}{{\beta}}
\newcommand{\cvrw}{{\alpha}}

\newcommand{\ind}[1]{\mathbbm{1}_{#1}}

\newcommand{\Eequi}{\cE_{\textup{equi}}}
\newcommand{\unif}{\textup{Uniform}\xspace}
\newcommand{\inv}{\textup{inv}}
\newcommand{\inverse}{^{-1}}
\newcommand{\diag}{\operatorname{diag}}
\newcommand{\loggrowth}{\hyperref[assum:kernel_growth]{\color{black}{\textsc{LogGrowth}}}\xspace}
\newcommand{\polygrowth}{\hyperref[assum:kernel_growth]{\color{black}{\textsc{PolyGrowth}}}\xspace}
\newcommand{\db}{{\diamond}} %
\newcommand{\bestapprox}[1]{{\llbracket {#1} \rrbracket}}
\newcommand{\kev}{d_{\kp}} %
\newcommand{\krcd}{L} %

\newcommand{\rpcq}{\mathfrak{U}}
\newcommand{\slg}{\gamma}

\DeclareMathOperator{\mmd}{MMD}

\newcommand{\kernel}{{\bm k}}
\renewcommand{\k}{\kernel}
\newcommand{\K}{{\bm K}}
\newcommand{\kergold}[1][\bP]{{\kernel_{#1}}}
\newcommand{\kp}{\kergold}
\newcommand{\kpsq}[1][\bP]{{\kernel^2_{#1}}}
\newcommand{\kq}{\kergold[\bQ]}
\newcommand{\impw}{\frac{\dd\P}{\dd\bQ}} %
\newcommand{\ksplit}{{\kernel_{\textup{split}}}}
\newcommand{\rkhs}[1][\kernel]{{\cH_{#1}}}

\newcommand{\ks}[1][p]{{\kernel_{#1}}}
\newcommand{\ksm}{{\kernel_{p}}}

\newcommand{\steinop}[1][p]{\cS_{#1}}

\newcommand{\knormSn}{\staticnorm{\k}_{n}}

\newcommand{\kpnormSn}{\staticnorm{\kp}_{n}}
\newcommand{\kpnorm}[2][\kp]{\staticnorm{#2}_{#1}}
\renewcommand{\grad}{\nabla}
\renewcommand{\div}{\nabla\cdot}

\newcommand{\divmk}{\nabla^{\otimes 2}\cdot(M\k)}
\newcommand{\kc}{{\k_{c}}}
\newcommand{\wopt}{{w_{\textsc{OPT}}}}
\newcommand{\mmdopt}{\mmd_{\textsc{OPT}}}

\newcommand{\halve}{\textsc{Halve}\xspace}
\newcommand{\thin}{\textsc{Thin}\xspace}

\newcommand{\compress}{\textsc{Compress}\xspace}
\newcommand{\compresspp}{\textup{Compress++}\xspace}
\newcommand{\compresspptarget}{\hyperref[alg:compresspp_target]{\textup{KT-Compress++}}\xspace}
\newcommand{\compresspptargetnotag}{{\textup{KT-Compress++}}\xspace}

\newcommand{\rpc}{\hyperref[alg:rpc]{
\textup{WeightedRPCholesky}}\xspace}
\newcommand{\rpcnotag}{\textup{WeightedRPCholesky}\xspace}
\newcommand{\symmetrize}{\textup{symmetrized}\xspace}

\newcommand{\kt}{\textsc{KT}\xspace}
\newcommand{\ktsplit}{\textsc{kt-split}\xspace}
\newcommand{\ktswap}{\textsc{kt-swap}\xspace}
\newcommand{\kttarget}{\hyperref[alg:kt_target]{\textup{KT}}\xspace}
\newcommand{\kttargetnotag}{{\textup{KT}}\xspace}
\newcommand{\kttargetfull}{\hyperref[alg:kt_target]{\textup{KernelThinning}}\xspace}
\newcommand{\kttargetfullnotag}{{\textup{KernelThinning}}\xspace}
\newcommand{\ktswaptarget}{\hyperref[alg:kt_swap_target]{\textup{KT-Swap}}\xspace}
\newcommand{\ktswaptargetnotag}{{\textup{KT-Swap}}\xspace}
\newcommand{\ktswapls}{\hyperref[alg:kt_swap_ls]{\textup{KT-Swap-LS}}\xspace}

\newcommand{\agm}{\hyperref[alg:agm]{\textup{AMD}}\xspace}
\newcommand{\agmnotag}{{\textup{AMD}}\xspace}

\newcommand{\alrbc}{\hyperref[alg:alrbc]{\textup{LD}}\xspace}
\newcommand{\alrbcfull}{\hyperref[alg:alrbc]{\textup{Low-rankDebiasing}}\xspace}
\newcommand{\alrbcnotag}{{\textup{LD}}\xspace}

\newcommand{\recombbfs}{\hyperref[alg:recombbfs]{\textup{Recombination}}\xspace}
\newcommand{\recombbfsnotag}{{\textup{Recombination}}\xspace}
\newcommand{\stratresamp}{\hyperref[alg:stratified_resample]{\textup{Resample}}\xspace}
\newcommand{\stratresampnotag}{{\textup{Resample}}\xspace}
\newcommand{\recombfull}{\hyperref[alg:recomb]{\textup{RecombinationThinning}}\xspace}
\newcommand{\recomb}{\hyperref[alg:recomb]{\textup{RT}}\xspace}
\newcommand{\recombnotag}{{\textup{RT}}\xspace}
\newcommand{\cholthinfull}{\hyperref[alg:cholthin]{\textup{CholeskyThinning}}\xspace}
\newcommand{\cholthin}{\hyperref[alg:cholthin]{\textup{CT}}\xspace}
\newcommand{\cholthinnotag}{{\textup{CT}}\xspace}

\newcommand{\GBCnotag}{{\textup{ST}}\xspace}
\newcommand{\GBCfull}{\hyperref[alg:GBC]{\textup{SteinThinning}}\xspace}
\newcommand{\GBCfullnotag}{{\textup{SteinThinning}}\xspace}
\newcommand{\GSKT}{\hyperref[alg:GSKT]{\textup{SKT}}\xspace}
\newcommand{\GSKTnotag}{{\textup{SKT}}\xspace}
\newcommand{\LSKT}{\hyperref[alg:LSKT]{\textup{LSKT}}\xspace}
\newcommand{\LSKTnotag}{{\textup{LSKT}}\xspace}
\newcommand{\GSR}{\hyperref[alg:SR]{\textup{SR}}\xspace}
\newcommand{\LSR}{\hyperref[alg:SR]{\textup{LSR}}\xspace}
\newcommand{\GSRnotag}{\textup{SR}\xspace}
\newcommand{\LSRnotag}{\textup{LSR}\xspace}
\newcommand{\GSC}{\hyperref[alg:SC]{\textup{SC}}\xspace}
\newcommand{\LSC}{\hyperref[alg:SC]{\textup{LSC}}\xspace}
\newcommand{\GSCnotag}{\textup{SC}\xspace}
\newcommand{\LSCnotag}{\textup{LSC}\xspace}

\newcommand{\mrm}[1]{\mathrm{#1}}

\newcommand{\dirac}{\mbf \delta}

\usepackage{marginnote}

\renewcommand{\real}{\mbb R}
\newcommand{\half}{\frac12}

\newcommand{\E}{\mbb E}

\newcommand{\sless}[1]{\stackrel{#1}{\leq}}

\newcommand{\normalparenth}[1]{( #1 )}
\newcommand{\sparenth}[1]{\normalparenth{#1}}

\newcommand{\bigparenth}[1]{\big( #1 \big)}

\newcommand{\sbraces}[1]{\{ #1 \}}

\newcommand{\angles}[1]{\left\langle #1 \right \rangle}

\newcommand{\ceil}[1]{\left\lceil #1 \right \rceil}
\newcommand{\sceil}[1]{\lceil #1 \rceil}
\newcommand{\floor}[1]{\left\lfloor #1 \right \rfloor}

\newcommand{\tp}{^\top}

\newcommand{\qtext}[1]{\quad\text{#1}\quad} 
\newcommand{\stext}[1]{\ \text{#1}\ }

\def\defeq{\triangleq} %

\def\norm#1{\big\|{#1}\big\|} %
 \def\snorm#1{\|{#1}\|} %
\newcommand{\twonorm}[1]{\norm{#1}_2} %
\newcommand{\stwonorm}[1]{\snorm{#1}_2} %
\newcommand{\sinfnorm}[1]{\snorm{#1}_{\infty}} %

\def\mbi#1{\boldsymbol{#1}} %
\def\mbf#1{\mathbf{#1}}
\def\mbb#1{\mathbb{#1}}

\def\mrm#1{\mathrm{#1}}

\def\tbf#1{\textbf{#1}}

\def\til#1{\widetilde{#1}}

\def\balign#1\ealign{\begin{align}#1\end{align}}
\def\baligns#1\ealigns{\begin{align*}#1\end{align*}}
\def\balignat#1\ealign{\begin{alignat}#1\end{alignat}}
\def\balignats#1\ealigns{\begin{alignat*}#1\end{alignat*}}
\def\bitemize#1\eitemize{\begin{itemize}#1\end{itemize}}
\def\benumerate#1\eenumerate{\begin{enumerate}#1\end{enumerate}}

\newenvironment{talign*}
 {\csname align*\endcsname}
 {\endalign}
\newenvironment{talign}
 {\csname align\endcsname}
 {\endalign}

\def\balignst#1\ealignst{\begin{talign*}#1\end{talign*}}
\def\balignt#1\ealignt{\begin{talign}#1\end{talign}}

\def\blue#1{\textcolor{blue}{{#1}}}

\def\red#1{\textcolor{red}{{#1}}}

\def\reals{\mathbb{R}} %
\def\R{\mathbb{R}}
\def\Q{\mathbb{Q}}
\def\naturals{\mathbb{N}} %
\def\N{\mathbb{N}}

\newcommand{\Otilde}{\til{O}}

\def\norm#1{\left\|{#1}\right\|} %
\def\staticnorm#1{\|{#1}\|} %
\newcommand{\statictwonorm}[1]{\staticnorm{#1}_2} %
\def\Matern{Mat\'ern\xspace}
\def\polylog{\operatorname{poly-log}}
\def\mbi#1{\boldsymbol{#1}} %
\newcommand{\para}[1]{\tbf{#1}\ \ }
\newcommand{\Hess}{\nabla^2} %
\newcommand{\seq}[1]{\stackrel{#1}{=}}

\newcommand{\ncref}[1]{\cref{#1}: \nameref*{#1}} %

\newcommand{\pcref}[1]{Proof of \ncref{#1}} %

\newcommand{\mytitle}{Debiased Distribution Compression}
\title{\mytite}

\icmltitlerunning{\mytitle}
\begin{document}

\twocolumn[
\icmltitle{
\mytitle}

\begin{icmlauthorlist}
\icmlauthor{Lingxiao Li}{mit}
\icmlauthor{Raaz Dwivedi}{cornell}
\icmlauthor{Lester Mackey}{msr}
\end{icmlauthorlist}

\icmlaffiliation{mit}{MIT CSAIL}
\icmlaffiliation{cornell}{Cornell Tech}
\icmlaffiliation{msr}{Microsoft Research New England}

\icmlcorrespondingauthor{Lingxiao Li}{lingxiao@mit.edu}
\icmlcorrespondingauthor{Raaz Dwivedi}{dwivedi@cornell.edu}
\icmlcorrespondingauthor{Lester Mackey}{lmackey@microsoft.com}

\icmlkeywords{Machine Learning, ICML}

\vskip 0.3in
]
\printAffiliationsAndNotice{} %

\etoctocstyle{1}{Table of contents}
\etocdepthtag.toc{mtchapter}
\etocsettagdepth{mtchapter}{section}

\begin{abstract}
Modern compression methods can summarize a target distribution $\mathbb{P}$ more succinctly than i.i.d.\ sampling but require access to a low-bias input sequence like a Markov chain converging quickly to $\mathbb{P}$.
We introduce a new suite of compression methods suitable for compression with biased input sequences.  
Given $n$ points targeting the wrong distribution and quadratic time, 
Stein Kernel Thinning (SKT)  
returns $\sqrt{n}$ equal-weighted points with $\widetilde{O}(n^{-1/2})$ maximum mean discrepancy (MMD) to $\mathbb {P}$.
For larger-scale compression tasks, Low-rank SKT achieves the same feat in 
sub-quadratic time 
using an adaptive low-rank debiasing procedure that may be of independent interest.
For downstream tasks that support simplex or constant-preserving weights, Stein Recombination and Stein Cholesky achieve even greater parsimony, matching the guarantees of SKT with as few as $\operatorname{poly-log}(n)$ weighted points. 
Underlying these advances are new guarantees for the quality of simplex-weighted coresets, 
the spectral decay of kernel matrices, and the covering numbers of Stein kernel Hilbert spaces. In our experiments, our techniques provide succinct and accurate posterior summaries while overcoming biases due to burn-in, approximate Markov chain Monte Carlo, and tempering.
\end{abstract}

\section{Introduction}
\begin{table*}[ht]
  \caption{\label{tab:summary}\textbf{Methods for debiased distribution compression.}
  For each method, we report the smallest coreset size $m$ and running time, up to logarithmic factors, sufficient to guarantee  $\Otilde(n^{-1/2})$ $\mmd_{\kp}$ to $\bP$ 
  given a \loggrowth kernel $\kp$ and $n$ slow-growing input points $\cS_n = (x_i)_{i=1}^n$ from a fast-mixing Markov chain targeting  $\bQ$  with tails no lighter than $\bP$  (see \cref{thm:mcmc_convex,exam:well_controlled}).
  For generic slow-growing $\cS_n$, identical guarantees hold for excess $\mmd_{\kp}$ \eqref{eq:excess_mmd} relative to the best simplex reweighting of $\cS_n$.}
  \centering
  \resizebox{\textwidth}{!}{
  \begin{tabular}{ccccc}
    \toprule
    \tbf{Method} 
    & \tbf{Compression Type} 
    & \def\arraystretch{1.2}\begin{tabular}{c}\tbf{Coreset Size $\mbi{m}$} 
    \end{tabular} 
    & \tbf{Runtime}
    & \tbf{Source}
    \\ \midrule
    \tbf{Stein Thinning} \small \citep{riabiz2022optimal}
    & equal-weighted 
    & $\red{n}$ 
    & $\kev n^2$ 
    & \cref{sec:st}
    \\
    \midrule 
    \textbf{Stein Kernel Thinning} \ $\begin{cases}\textup{Greedy} \\ \textup{Low-rank}\end{cases}$
    \hspace*{-0.2in}\def\arraystretch{1.2}\begin{tabular}{c}(\cref{alg:GSKT}) \\
    (\cref{alg:LSKT})
    \end{tabular} 
     & equal-weighted 
     & $\blue{\sqrt{n}}$ 
     & \def\arraystretch{1.2}\begin{tabular}{c}$\kev n^2$\\
  $\ \ \,\kev n^{1.5}$
 \end{tabular}  
    & \def\arraystretch{1.2}\begin{tabular}{c}
    \cref{thm:GSKT_guarantee}\\
    \cref{thm:LSKT_guarantee} 
    \end{tabular}
    \\ \midrule
    \textbf{Stein Recombination} \  
     $\begin{cases}\textup{Greedy} \\ \textup{Low-rank}\end{cases}$
      \hspace*{-0.12in}\textup{(\cref{alg:SR})}
     & simplex-weighted 
     & $\blue{\polylog(n)}$ 
     & \def\arraystretch{1.2}\begin{tabular}{c}$\kev n^2$\\
 $\ \ \,\kev n+n^{1.5}$
 \end{tabular}
    & \cref{thm:SR_guarantee}
    \\ \midrule
    \textbf{Stein Cholesky} \ $\begin{cases}\textup{Greedy} \\ \textup{Low-rank}\end{cases}$ 
    \hspace*{-0.15in} \textup{(\cref{alg:SC})}
     & constant-preserving 
     & $\blue{\polylog(n)}$
     & \def\arraystretch{1.2}\begin{tabular}{c}$\kev n^2$\\
  $\ \ \,\kev n + n^{1.5}$
 \end{tabular} 
    & \cref{thm:SC_guarantee}
    \\ \bottomrule
  \end{tabular}}
\end{table*}

Distribution compression is the problem of summarizing a target probability distribution $\bP$ with a small set of representative points. 
Such compact summaries are particularly valuable for tasks that incur substantial downstream computation costs per summary point, like organ and tissue modeling in which each simulation consumes thousands of CPU hours \citep{niederer2011simulating}. 

Remarkably, modern compression methods can summarize a distribution more succinctly than \iid sampling. 
For example, kernel thinning (KT) \citep{dwivedi2021kernel, dwivedi2022generalized}, Compress++ \citep{shetty2022distribution}, recombination \citep{hayakawa2023sampling}, 
and randomly pivoted Cholesky \citep{epperly2024kernel} all provide $\Otilde(1/m)$ approximation error using $m$ points, a significant improvement over the $\Omega(1/\sqrt{m})$ approximation provided by \iid sampling from $\bP$.
However, each of these constructions relies on access to an accurate input sequence, like an \iid sample from $\bP$ or a Markov chain converging quickly to $\bP$. 

Much more commonly, one only has access to $n$  \emph{biased} sample points  approximating a 
distribution $\bQ\neq\bP$. 
Such biases are a common occurrence in Markov chain Monte Carlo (MCMC)-based inference due to tempering \citep[where one targets a less peaked and more dispersed distribution to achieve faster convergence,][]{gramacy2010importance}, burn-in \citep[where the initial state of a Markov chain biases the distribution of chain iterates,][]{cowles1996markov}, or approximate MCMC \citep[where one runs a cheaper approximate Markov chain to avoid the prohibitive costs of an exact MCMC algorithm, e.g.,][]{ahn2012bayesian}. 
The Stein thinning (ST) method of \citet{riabiz2022optimal} was developed to provide accurate compression even when the input sample sequence provides a poor approximation to the target. ST operates by greedily thinning the input sample to minimize the maximum mean discrepancy \citep[MMD,][]{gretton2012kernel} to $\bP$.
However, ST is only known to provide an $O(1/\sqrt{m})$ approximation to $\bP$; this guarantee is no better than that of \iid sampling and a far cry from the $\Otilde(1/m)$ error achieved with unbiased coreset constructions.

\newcommand{\greedy}{SKT\xspace}
\newcommand{\lrtag}{Low-rank\xspace}
\newcommand{\lrgreedy}{\lrtag \greedy\xspace}
\newcommand{\recom}{Stein Recombination\xspace}
\newcommand{\lrrecom}{\lrtag\recom}
\newcommand{\chol}{Stein Cholesky\xspace}
\newcommand{\lrchol}{\lrtag\chol}
In this work, we address this deficit by developing new, efficient coreset constructions that provably yield better-than-\iid error even when the input sample is biased. 
For $\P$ on $\reals^d$, our primary contributions are fourfold and summarized in \cref{tab:summary}. %
First, for the task of equal-weighted compression, we introduce \emph{Stein Kernel Thinning} (\greedy, \cref{alg:GSKT}), a strategy 
    that combines the greedy bias correction properties of ST with the unbiased compression of KT to produce $\sqrt{n}$ summary points with error $\Otilde(n^{-1/2})$ in $O(n^2)$ time. In contrast, ST would require $\Omega(n)$ points to guarantee this error.
Second, for larger-scale compression problems, we propose \emph{Low-rank SKT} 
(\cref{alg:LSKT}), a strategy that combines the scalable summarization of Compress++ with a new low-rank debiasing procedure (\cref{alg:alrbc}) to match the SKT guarantees in sub-quadratic $o(n^2)$ time.

Third, for the task of simplex-weighted compression, in which summary points are accompanied by weights in the simplex, we propose greedy and low-rank \emph{\recom} (\cref{alg:SR}) constructions that match the guarantees of SKT with as few as $\polylog(n)$ points. 
Finally, for the task of constant-preserving compression, in which summary points are accompanied by real-valued weights summing to $1$, 
we introduce greedy and low-rank \emph{\chol} (\cref{alg:SC}) constructions that again match the guarantees of SKT using as few as $\polylog(n)$ points. 

Underlying these advances are new guarantees for the quality of simplex-weighted coresets (\cref{thm:mcmc_convex,thm:iid_convex}), the spectral decay of kernel matrices (\cref{cor:Kp_eigval_bound}), 
and
the covering numbers of Stein kernel Hilbert spaces (\cref{prop:ks_growth}) 
that may be of independent interest. 
In \cref{sec:experiments}, we employ our new procedures to produce compact summaries of complex target distributions given input points biased by burn-in, approximate MCMC, or tempering.

\par
\para{Notation}
We assume Borel-measurable sets and functions and define $[n] \defeq \{1,\ldots, n\}$, $\simplex \defeq \{w \in \bR^n: w \ge 0, \bm{1}^\top w = 1\}$, $\staticnorm{x}_0\defeq\abs{\{i:x_i\neq 0\}}$, and $\staticnorm{x}_p^p\defeq\sum_i |x_i|^p$ for $x \in \bR^d$ and $p\geq 1$. 
For $x \in \bR^d$, $\delta_x$ denotes the delta measure at $x$.
We let $\rkhs[\k]$ denote the reproducing kernel Hilbert space (RKHS) of a kernel $\k: \bR^d\times\bR^d \to \bR$ \citep{aronszajn1950theory} and $\norm{f}_{\k}$ denote the RKHS norm of $f \in \rkhs[\k]$.
For a measure $\mu$ and separately $\mu$-integrable $\k$ and $f$, we write $\mu f \!\defeq\! \int f(x) \dd\mu(x)$ and $\mu \k(x) \!\defeq\! \int \k(x,y)\dd\mu(y)$. 
The divergence of a differentiable matrix-valued function $A$ is $(\grad_x\cdot A(x))_j = \sum_i \partial_{x_i} A_{ij}(x)$. 
For random variables $(X_n)_{n\in\N}$, we say $X_n \!=\! O(f(n,\delta))$ holds with probability $\ge 1\!-\!\delta$ if $\Pr(X_n \!\le\! C f(n,\delta)) \ge 1\!-\!\delta$ for a constant $C$ independent of $(n,\delta)$ and all $n$ sufficiently large. When using this notation, we view all algorithm parameters except $\delta$ as functions of $n$.
For $A \!\in\! \real^{n\times n}$ and $v \in \bR^n$, $\diag(A)$ and $\diag(v)$ are $n\times n$ diagonal matrices with $A_{ii}$ and $v_i$ respectively as the $i$-th diagonal entry.

\section{Debiased Distribution Compression}\label{sec:prelim}
Throughout, we aim to summarize  
a fixed target distribution $\P$ on $\reals^d$ using a sequence $\cS_n \defeq (x_i)_{i=1}^n$ of potentially biased candidate points in $\reals^d$.\footnote{Our coreset constructions will in fact apply to any sample space, but our analysis will focus on $\reals^d$.}
Correcting for unknown biases in $\cS_n$ requires some auxiliary knowledge of $\P$.
For us, this knowledge comes in the form of a kernel function $\kp$ with known expectation under $\P$. 
Without loss of generality, we can take this kernel mean to be identically zero.\footnote{For $\P\kp\not\equiv0$, the kernel $\kp'(x,y) = \kp(x,y) -\P\kp(x)-\P\kp(y) + \P\P\kp$ satisfies $\P\kp'\equiv 0$ and $\mmd_{\kp'}=\mmd_{\kp}$.}
\begin{assumption}[Mean-zero kernel]
\label{assum:mean_zero_p}
For some $\fp\geq 1/2$, %
$\bE_{x\sim\bP}[\kp(x,x)^\fp]<\infty$ and 
    $\P\kp\equiv 0$. 
\end{assumption}
Given %
a target compression size $m$, our goal is to output an weight vector $w\in\R^n$ with $\norm{w}_0 \le m$, $\bm{1}_n^\top w=1$, and $o(m^{-1/2})$ (better-than-\iid) maximum mean discrepancy (MMD) to $\P$:
\begin{talign}
  \mmd_{\kp}(\sum_{i=1}^n\!w_i\dirac_{x_i}, \P) \defeq\sqrt{\sum_{i,j=1}^n\!w_iw_j\kp(x_i, x_j)}.
  \label{eqn:mmd_sup}
\end{talign}
We consider three standard compression tasks with $\norm{w}_0\leq m$.
In \emph{equal-weighted compression} one selects $m$ possibly repeated points from $\cS_n$ and assigns each a weight of $\frac{1}{m}$; because of repeats, the induced weight vector over $\cS_n$ satisfies $w \in  \simplex \cap (\frac{\bN_0}{m})^n$.
In \emph{simplex-weighted compression} we allow any $w\in\simplex$, and in \emph{constant-preserving compression} we simply enforce $\bm{1}_n^\top w=1$. %
All three coreset types exactly preserve constants and hence satisfy standard coherence constraints.

When making big-O statements, we will 
treat $\cS_n$ as the prefix of an infinite sequence $\cS_\infty \defeq (x_i)_{i\in\bN}$.
We also write $\kp(\cS_n[\tJ],\cS_n[\tJ]) \defeq[\kp(x_i, x_j)]_{i,j\in\tJ}$ for the principal kernel submatrix with indices $\tJ \subseteq [n]$.

\subsection{Kernel assumptions}
Many practical \emph{Stein kernel} constructions are available for generating mean-zero kernels for a target $\P$~\citep{chwialkowski16kernel,liu2016kernelized,gorham2017measuring, gorham2019measuring, barp2019minimum, yang2018goodness, afzali2023gradient}. 
We will use the most prominent of these Stein kernels as a running example:
\begin{definition}[Stein kernel]
\label{def:stein_kernel}
Given a differentiable base kernel $\k$ and a symmetric positive semidefinite matrix $M$,  the \emph{Stein kernel} $\ksm: \bR^d \times \bR^d \to \bR$ for $\P$ with positive differentiable Lebesgue density $p$ 
is defined as
\begin{talign}
    &\ksm(x,y) \defeq 
    \frac{1}{p(x) p(y)}{\grad_x \cdot \grad_y \cdot (p(x)M\k(x,y) p(y))}.
\end{talign}
\end{definition}

While our algorithms apply to any mean zero kernel, including the aforementioned Stein kernels, the Sobolev kernels used in quasi-Monte Carlo \citep{kuo2003component}, and the popular centered Gaussian kernel \citep{chen2010super,lacoste2015sequential}, our guarantees adapt to the underlying smoothness of the kernel. Our next definition and assumption make this precise.
\begin{definition}[Covering number]
For %
    a kernel $\k: \bR^d \times \bR^d \to \bR$ with $\ball{\k} \defeq \sbraces{f\in\rkhs[\k]:\norm{f}_{\k} \leq 1}$, a set $A \subset \bR^d$, and $\vareps > 0$, the covering number $\cvrnum{\k}(A, \vareps)$ is the minimum cardinality of all sets $\cC \subset \ball{\k}$ satisfying
  \begin{talign}
      \ball{\k} \subset \bigcup_{h \in \cC}\{g \in \ball{\k}: \sup_{x\in A}\abs{h(x)-g(x)}\le \vareps\}.
  \end{talign}
\end{definition}
\renewcommand\theassumption{($\mbi{\cvrw,\cvrd}$)-kernel}
\begin{assumption}
\label{assum:kernel_growth}
For some $\cvrC > 0$, all $r>0$ and $\vareps\in(0,1)$, and $\balleuc(r)\defeq\sbraces{x\in\real^d:\twonorm{x} \leq r}$, a kernel $\k$ is either $\polygrowth(\cvrw,\cvrd)$, i.e., 
  \begin{talign}
      \log \cvrnum{\k}(\balleuc(r),\vareps)\le\cvrC(1/\vareps)^{\cvrw}(r+1)^{\cvrd}\label{eqn:poly_growth_k},
  \end{talign}
  with $\alpha < 2$ 
  or $\loggrowth(\cvrw,\cvrd)$, i.e., 
  \begin{talign}
      \log \cvrnum{\k}(\balleuc(r),\vareps)\le\cvrC\log(e/\vareps)^{\cvrw}(r+1)^{\cvrd}.\label{eqn:log_growth_k}
  \end{talign}
\end{assumption}
\renewcommand\theassumption{\arabic{assumption}}
In \cref{cor:Kp_eigval_bound} we show that the eigenvalues of kernel matrices with \polygrowth and \loggrowth kernels have polynomial and exponential decay respectively.
\citet[Prop.~2]{dwivedi2022generalized} showed that all sufficiently differentiable kernels satisfy the $\polygrowth$ condition 
and that bounded radially analytic kernels are $\loggrowth$. 
Our next result, proved in \cref{subsec:stein_spectral}, shows that a Stein kernel $\ksm$ can inherit the growth properties of its base kernel even if $\ksm$ is itself unbounded and non-smooth.
\begin{proposition}[Stein kernel growth rates]\label{prop:ks_growth}
  A Stein kernel $\ksm$ with 
   $\sup_{\statictwonorm{x} \le r}\statictwonorm{\nabla \log p(x)} = O(r^{d_{\ell}})$ for $d_{\ell}\geq 0$ is %
  \begin{enumerate}[label=(\alph*), leftmargin=*]
  \itemsep0em
  \item $\loggrowth(d+1, 2d+\delta)$ for any $\delta > 0$ if the base kernel $\k$ is radially analytic (\cref{def:analytic_k}) and
  \item $\polygrowth(\frac{d}{s-1}, (1+\frac{d_{\ell}}{s})d)$ if the base kernel $\k$ is $s$-times continuously differentiable (\cref{def:diff_k}) for $s\!>\!1$.
\end{enumerate}
\end{proposition}
Notably, the popular Gaussian (\cref{exam:cvrnum_gauss}) and inverse multiquadric (\cref{exam:cvrnum_imq}) base kernels satisfy the \loggrowth preconditions, while \Matern, B-spline, sinc, sech, and Wendland's compactly supported kernels satisfy the \polygrowth precondition \citep[Prop.~3]{dwivedi2022generalized}.
To our knowledge, \cref{prop:ks_growth} and \cref{cor:Kp_eigval_bound} provide the first covering number bounds and eigenvalue decay rates for the typically unbounded Stein kernels $\ksm$.

\subsection{Input point desiderata}
\label{subsec:benchmark}
Our primary desideratum for the input points is that they can be debiased into an accurate estimate of $\bP$.  Indeed, our high-level strategy for debiased compression is to first use $\kp$ to debias the input points into a more accurate approximation of $\P$ and then compress that approximation into a more succinct representation.
Fortunately, even when the input $\cS_n$ targets a distribution $\Q\neq\P$, effective debiasing is often achievable via simplex reweighting, i.e., by solving the convex optimization problem 
\begin{talign}
    &\wopt \in \arg\min_{w\in\Delta_{n-1}} \sum_{i,j=1}^n w_iw_j\kp(x_i, x_j) \label{eqn:w_opt}\\
    &\text{with }\, \mmdopt \defeq \mmd_{\kp}(\sum_{i=1}^n\!\wopt_i\dirac_{x_i}, \P).
    \label{eqn:mmd_opt}
\end{talign}
For example, \citet[Thm.~1b]{hodgkinson2020reproducing} showed that simplex reweighting can correct for biases due to off-target \iid or MCMC sampling.
Our next result (proved in \cref{sec:thm:mcmc_convex_proof}) significantly relaxes their conditions.
\begin{theorem}[Debiasing to \iid quality via simplex reweighting]\label{thm:mcmc_convex}
  Consider a kernel $\kp$ satisfying \cref{assum:mean_zero_p} with $\rkhs[\kp]$ separable, and %
  suppose $(x_i)_{i=1}^\infty$ are the iterates of a homogeneous $\phi$-irreducible geometrically ergodic Markov chain \citep[Thm.~1]{gallegosherrada2023equivalences} with stationary distribution $\bQ$ and initial distribution absolutely continuous with respect to $\bP$.
  If $\E_{x\sim\bP}[\frac{\dd\P}{\dd\bQ}(x)^{2\fq-1}\kp(x,x)^\fq] < \infty$ for some $\fq > 1$ then %
  $\mmdopt = O(n^{-1/2})$
  in probability.
\end{theorem}
\begin{remark}
$\rkhs[\kp]$ is separable whenever $\kp$ is continuous \citep[Lem.~4.33]{steinwart2008support}.
\end{remark}
Since $n$ points sampled \iid from $\bP$ have $\Theta(n^{-1/2})$ root mean squared MMD (see \cref{prop:iid_guarantee}),
\cref{thm:mcmc_convex} shows that a debiased off-target sample can be as accurate as a direct sample from $\bP$.   %
Moreover, \cref{thm:mcmc_convex}  applies to many practical examples. 
The simplest example of a geometrically ergodic chain is \iid sampling from $\bQ$, but geometric ergodicity has also been established for a variety of popular  Markov chains including random walk Metropolis \citep[Thm.~3.2]{roberts1996geometric}, independent Metropolis-Hastings  \citep[Thm.~2.2]{atchade2007geometric},
the unadjusted Langevin algorithm \citep[Prop.~8]{durmus2017nonasymptotic}, the Metropolis-adjusted Langevin algorithm \citep[Thm.~1]{durmus2022geometric},  Hamiltonian Monte Carlo \citep[Thm.~10 and Thm.~11]{durmus2020irreducibility},  stochastic gradient Langevin dynamics \citep[Thm.~2.1]{li2023geometric}, and the Gibbs sampler \citep{johnson2009geometric}.
Moreover, for $\bQ$ absolutely continuous with respect to $\bP$, the importance weight $\impw$ is typically bounded or slowly growing when the tails of $\bQ$ are not much lighter than those of $\bP$.

Remarkably, under more stringent conditions, \cref{thm:iid_convex} (proved in \cref{sec:thm:iid_convex_proof}) shows that simplex reweighting can decrease MMD to $\P$ at an even-faster-than-\iid rate.
\begin{theorem}[Better-than-\iid debiasing via simplex reweighting]\label{thm:iid_convex}
  Consider a kernel $\kp$ satisfying \cref{assum:mean_zero_p} with $\fp=2$ and points  $(x_i)_{i=1}^\infty$ drawn \iid from a distribution $\bQ$ with $\frac{\dd\P}{\dd\bQ}$ bounded. 
  If $\E[\kp(x_1,x_1)^\fq] < \infty$ for some $\fq >3$, then %
  $\bE[\mmdopt^2] = o(n^{-1}).$
\end{theorem}
The work of  \citet[Thm.~3.3]{liu2017black} also established $o(n^{-1/2})$ MMD error for simplex reweighting but only under a uniformly bounded eigenfunctions assumption that is often violated (\citealp[Thm.~1]{minh2010some}, \citealp[Ex.~1]{zhou2002covering}) and difficult to verify \citep{steinwart2012mercer}.

We highlight that our remaining results make no particular assumption about the input points but rather upper bound the excess MMD
\begin{talign}    \label{eq:excess_mmd}
    \Delta\!\mmd_{\kp}(w) &\defeq \mmd_{\kp}(\sum_{i\in[n]} \! w_i\dirac_{x_i}, \P) \\ 
    &-\mmdopt
\end{talign}
of a candidate weighting $w$ %
in terms of the input point radius $R_n \defeq \max_{i\in[n]}\twonorm{x_i}\vee 1$ and kernel radius $\norm{\kp}_n \defeq \max_{i\in[n]}\kp(x_i, x_i)$.
While these results apply to \emph{any} input points, we consider the following running example of \emph{slow-growing} input points throughout the paper.
\begin{definition}[Slow-growing input points]\label{exam:well_controlled}
We say $\cS_n$ is \emph{$\slg$-slow-growing} if $R_n = O((\log n)^\slg)$ for some $\slg \ge 0$ and $\kpnormSn = \Otilde(1)$.
\end{definition}
Notably, $\cS_n$ is $1$-slow-growing with probability $1$ when $\kp(x,x)$ is polynomially bounded by $\twonorm{x}$ and the input points are drawn from a homogeneous $\phi$-irreducible geometrically ergodic Markov chain  with a sub-exponential target $\bQ$, i.e., $\bE[e^{c\twonorm{x}}]<\infty$ for some $c > 0$ \citep[Prop.~2]{dwivedi2021kernel}.
For a Stein kernel $\ksm$ (\cref{def:stein_kernel}), by \cref{prop:ks_alt_form}, $\ksm(x,x)$ is polynomially bounded by $\twonorm{x}$ if $\k(x,x)$, $\twonorm{\nabla_x\nabla_y \k(x,x)}$, and $\twonorm{\nabla\log p(x)}$ are all polynomially bounded by $\twonorm{x}$.
Moreover,  $\twonorm{\nabla\log p(x)}$ is automatically polynomially bounded by $\twonorm{x}$ when $\nabla \log p$ is Lipschitz or, more generally, pseudo-Lipschitz \citep[Eq.~(2.5)]{erdogdu2018global}.

\subsection{Debiased compression via Stein Kernel Thinning}\label{sec:quad_time}
Off-the-shelf solvers based on mirror descent and Frank Wolfe can solve the convex debiasing program~\cref{eqn:w_opt} in $O(n^3)$ time by generating weights with $O(n^{-1/2}\kpnormSn)$ excess MMD~\citep{liu2017black}. 
We instead employ a more efficient, greedy debiasing strategy based on Stein thinning (ST). After $n$ rounds, ST outputs an equal-weighted coreset of size $n$ with $O(n^{-1/2}\kpnormSn)$ excess MMD \citep[Thm.~1]{riabiz2022optimal}.  
Moreover, while the original implementation of \citet{riabiz2022optimal} has cubic runtime,  our implementation (\cref{alg:GBC}) based on sufficient statistics improves the runtime to $O(n^2\kev)$ where $\kev$ denotes the runtime of a single kernel evaluation.\footnote{Often, $\kev=\Theta(d)$ as in the case of Stein kernels (\cref{subsec:O_d_ks_eval}).}

The equal-weighted output of ST serves as the perfect input for the kernel thinning (KT) algorithm which compresses an equal-weighted sample of size $n$ into a coreset of any target size $m\leq n$ in $O(n^2\kev)$ time. 
We adapt the target KT algorithm slightly to target MMD error to $\bP$ and to include a baseline ST coreset of size $m$ in the \ktswap step (see \cref{alg:kt_swap_target}).
Combining the two routines 
we obtain  Stein Kernel Thinning (\GSKT), our first solution for equal-weighted debiased distribution compression:

\begin{algorithm}[htb]
  \caption{Stein Kernel Thinning (\GSKTnotag)\label{alg:GSKT}}
  \begin{algorithmic}\itemindent=-.7pc
    \STATE {\bfseries Input:} mean-zero kernel $\kp$, points $\cS_n$,  output size $m$, KT failure probability $\delta$
    \STATE $n' \gets m \, 2^{\sceil{\log_2 \frac{n}{m}}}$
    \STATE $w \gets \GBCfull(\kp, \cS_n, n')$ 
    \STATE $w_{\GSKTnotag} \gets \kttargetfull(\kp, \cS_n, n', w, m, \delta)$
    \STATE {\bf Return:} $w_{\GSKTnotag}\in\!\Delta_{n-1}\cap(\frac{\N_0}{m})^{n}$ \hfill\COMMENT{hence $\norm{w_{\GSKTnotag}}_{0}\leq m$}
  \end{algorithmic}
\end{algorithm}
Our next result, proved in \cref{sec:thm:GSKT_guarantee_proof}, shows that \GSKT yields better-than-\iid excess MMD whenever the radii ($R_n$ and $\norm{\kp}_n$) and kernel covering number exhibit slow growth.
\begin{theorem}[MMD guarantee for \GSKT]\label{thm:GSKT_guarantee}
  Given a kernel $\kp$ satisfying \cref{assum:mean_zero_p,assum:kernel_growth}, 
    Stein Kernel Thinning (\cref{alg:GSKT}) outputs $w_{\GSKTnotag}$ in $O(n^2\kev)$ time  %
    satisfying 
    \begin{talign}
       &\Delta\!\mmd_{\kp}(w_{\GSKTnotag}) \!=\!O\bigparenth{\frac{\sqrt{\norm{\kp}_n \ell_{\delta}  \cdot \log n \cdot  R_n^{\cvrd} G_{m}^{\cvrw}} }{\min(m, \sqrt{n})}}
    \end{talign}
    with probability at least $1-\delta$, where $\ell_{\delta} \defeq \log^2(\frac{e}{\delta})$  and
  \begin{talign}
    G_m \!\defeq\! 
    \begin{cases}
    1\!+\!\log m &  \loggrowth(\cvrw, \cvrd),  \\
    m & \polygrowth(\cvrw, \cvrd). \\
    \end{cases}
  \label{eqn:quad_time_dtt_g}
  \end{talign}
\end{theorem}

\begin{example}\label{exam:SKT_ex}
Under the assumptions of \cref{thm:GSKT_guarantee} with $\slg$-slow-growing input points (\cref{exam:well_controlled}), \loggrowth $\kp$, 
and a coreset size $m\leq \sqrt{n}$, 
\GSKTnotag with high probability delivers  $\Otilde(m\inverse)$ excess MMD, a nearly minimax optimal rate for equal-weighted coresets \citep[Thm.~3.1]{phillips2020near} and a significant improvement over the $\Omega(m^{-1/2})$ error rate of \iid sampling.
\end{example}

\begin{remark}\label{rmk:standard_thin_n_n0}
When $m \!\!<\!\!\sqrt{n}$, we can uniformly subsample or, in the case of MCMC inputs, \emph{standard thin} (i.e., keep only every $\frac{n}{m^2}$-th point of) the input sequence down to size $m^2$ before running \GSKT to reduce runtime from $O(n^2)$ to $O(m^4)$ while incurring only $O(m^{-1})$ excess error. A similar remark applies to \LSKT algorithm introduced in \cref{sec:sub_quad_time}. %
\end{remark}

\newcommand{\rg}{\gamma}
\section{Accelerated Debiased Compression}\label{sec:sub_quad_time}
To enable larger-scale debiased compression, we next introduce a sub-quadratic-time version of \GSKTnotag built via a new low-rank debiasing scheme and the near-linear-time compression algorithm of \citet{shetty2022distribution}.

\subsection{Fast bias correction via low-rank approximation}
At a high level, our approach to accelerated debiasing involves four components. 
First, we form a rank-$r$ approximation $FF^\top$ of  the kernel matrix $K = \kp(\cS_n, \cS_n)$ in $O(nr\kev+nr^2)$ time using a weighted extension (\rpc, \cref{alg:rpc}) of the randomly pivoted Cholesky algorithm of  \citet[Alg.~2.1]{chen2022randomly}.
Second, we correct the diagonal to form   
$K'=FF^\top + \diag(K-FF^\top)$.
Third, we solve the reweighting problem~\cref{eqn:w_opt} with  $K'$ substituted for $K$ using $T$ iterations of accelerated entropic mirror descent \citep[\agm,][Alg.~14 with $\phi(w)=\sum_i w_i\log w_i$]{wang2023no}. 
The acceleration ensures $O(1/T^2)$ suboptimality after $T$ iterations, and 
each iteration takes only $O(nr)$ time thanks to the low-rank plus diagonal approximation.
Finally, we repeat this three-step procedure $Q$ times, each time using the weights outputted by the prior round to update the low-rank approximation $\widehat K$. 
On these subsequent adaptive rounds, \rpc approximates the leading subspace of a \emph{weighted} kernel matrix $\diag(\sqrt{\tilde{w}}) K \diag(\sqrt{\tilde{w}})$ for $\tilde{w} \in \simplex$ before undoing the row and column reweighting.  
Since each round's weights are closer to optimal, this adaptive updating has the effect of upweighting more relevant subspaces for subsequent debiasing.  
For added sparsity, we prune the weights outputted by the prior round using stratified residual resampling  \citep[\stratresamp, \cref{alg:stratified_resample},][]{douc2005comparison}.
Our complete Low-rank Debiasing (\alrbc) scheme, summarized in \cref{alg:alrbc}, enjoys $o(n^2)$ runtime whenever $r = o(n^{1/2})$, $T = O(n^{1/2})$, and $Q=O(1)$.

\begin{algorithm}[htb]
  \caption{Low-rank Debiasing (\alrbcnotag)\label{alg:alrbc}}
  \begin{algorithmic}\itemindent=-.7pc
    \STATE {\bfseries Input:} mean-zero kernel $\kp$, points $\cS_n=(x_i)_{i=1}^n$, rank $r$, \agmnotag steps $T$, adaptive rounds $Q$
    \STATE $w^{(0)} \gets (\frac{1}{n},\ldots,\frac{1}{n}) \in \real^n$
    \FOR{$q=1$ {\bfseries to} $Q$}\itemindent=-.7pc
    \STATE $\tilde w \gets \stratresamp(w^{(q-1)}, n)$
    \STATE $\tI, F \gets \rpc(\kp, \cS_n, \tilde w, r)$
    \STATE $K' \gets FF^\top + \diag(\kp(\cS_n,\cS_n)) - \diag(FF^\top)$
    \STATE $w^{(q)} \!\gets\! \agm(K', T, \tilde w, \texttt{AGG}=\ind{q>1})$ 
    \STATE \textbf{if} $(w^{(q)})^\top K' w^{(q)} > \tilde w^\top K' \tilde w$ \textbf{then} $w^{(q)} \gets \tilde w$ 
    \ENDFOR
    \STATE {\bfseries Return:} $w_{\alrbcnotag}\gets w^{(Q)}\in\!\Delta_{n-1}$
  \end{algorithmic}
\end{algorithm}

Moreover, our next result, proved in \cref{subsec:LD_analysis}, shows that \alrbc provides \iid-level precision whenever $T\geq\sqrt{n}$, $Q=O(1)$, and $r$ grows appropriately with the input radius and kernel covering number.

\renewcommand\theassumption{$(\mbi{\cvrw,\!\cvrd})$-params}
\begin{assumption}\label{assum:lr}
  The kernel $\kp$ satisfies \cref{assum:kernel_growth,assum:mean_zero_p}, 
  the output size and rank $m,r \!\geq\! (\frac{\cvrC R_n^{\cvrd}+1}{\sqrt{\log 2}}+2\sqrt{\log 2})^2$, 
 the AMD step count $T \!\geq\! \sqrt{n}$, and the adaptive round count $Q\!=\!O(1)$.\footnote{To unify the presentation of our results, \cref{assum:lr} constrains all common algorithm input parameters with the understanding that the conditions are enforced only when the input is relevant to a given algorithm.}
\end{assumption}

\begin{theorem}[Debiasing guarantee for \alrbc]\label{thm:LD_guarantee}
Under \cref{assum:lr}, 
Low-rank Debiasing (\cref{alg:alrbc}) takes $O((\kev\!+\!r\!+\!T)nr)$ time to output $w_{\alrbcnotag}$ satisfying %
  \begin{talign}
    &\Delta\!\mmd_{\kp}(w_{\alrbcnotag}) \!=\!O\bigg(\!\sqrt{\frac{\snorm{\kp}_n \max(\log n, 1/\delta)}{n}} + \sqrt{\frac{n H_{n,r}}{\delta}} \bigg) 
    \label{eqn:lrbc_guarantee}
  \end{talign}
   with probability at least $1 - \delta$, for any $\delta \in (0, 1)$ and $H_{n,r}$ defined in \eqref{eqn:lrbc_h} that satisfies
    \begin{talign}
    \!H_{n,r}\!=\!
    \begin{cases}
    \!O\big(\!\sqrt{r} (\frac{R_n^{2\beta}}{r})^{\!\frac{1}{\alpha}}\!\big)\!& \!\polygrowth(\alpha,\beta), \\
    \!O\big(\!\sqrt{r} \exp\sparenth{\!-\!\big(\!\frac{0.83\!\sqrt{r}-2.39}{\cvrC R_n^\beta}\!\big)^{\!\frac{1}{\alpha}\!}}\!\big)\!&\!\loggrowth(\alpha,\beta).
    \end{cases}
    \label{eqn:H_nr_big_O}
    \end{talign}
\end{theorem}
\begin{example}\label{exam:LD_ex}
Under the assumptions of \cref{thm:LD_guarantee} with $\gamma$-slow-growing input points (\cref{exam:well_controlled}), \loggrowth $\kp$, $T=\Theta(\sqrt{n})$, and $r = (\log n)^{2(\cvrw+\cvrd\slg)+\eps}$ for any $\eps> 0$, 
\alrbc delivers $\Otilde(n^{-1/2})$ excess MMD with high probability  in $\Otilde(n^{1.5})$ time. 
\end{example}

\subsection{Fast debiased compression via Low-rank Stein KT}

To achieve debiased compression in sub-quadratic time, we next propose Low-rank SKT (\cref{alg:LSKT}).
\LSKT debiases the input using \alrbc, converts the \alrbc output into an equal-weighted coreset using \stratresamp, and 
finally combines \kttarget with the divide-and-conquer \compresspp framework \citep{shetty2022distribution} to compress $n$ equal-weighted points into $\sqrt{n}$ in near-linear time. 
\begin{algorithm}[htb]
  \caption{Low-rank Stein Kernel Thinning (\LSKTnotag)}
  \label{alg:LSKT}
  \begin{algorithmic}\itemindent=-.7pc
    \STATE {\bfseries Input:} mean-zero kernel $\kp$, points $\cS_n=(x_i)_{i=1}^n$, rank $r$, AGM steps $T$, adaptive rounds $Q$, oversampling parameter $\fg$, failure prob. $\delta$
    \STATE $w \gets \alrbcfull(\kp, \cS_n, r, T, Q)$
    \STATE $n' \gets 4^{\ceil{\log_4 n}}, m \gets \sqrt{n'}$ \COMMENT{output size $\sqrt{n} \leq m < 2\sqrt{n}$}
    \STATE $w \gets \stratresamp(w, n')$
    \STATE $w_{\LSKTnotag} \gets \compresspptarget(\kp, \cS_n, n', w, \fg, \frac{\delta}{3})$
    \STATE {\bf Return:} $w_{\LSKTnotag}\in\!\Delta_{n-1}\cap(\frac{\N_0}{m})^{n}$ \COMMENT{hence $\norm{w_{\LSKTnotag}}_{0}\leq m$}
  \end{algorithmic}
\end{algorithm}

Our next result (proved in \cref{sec:app_sub_quad_time}) shows that \LSKT can provide better-than-\iid excess MMD in $o(n^2)$ time.
\begin{theorem}[MMD guarantee for \LSKT]\label{thm:LSKT_guarantee}
Under \cref{assum:lr}, Low-rank SKT (\cref{alg:LSKT}) with $\fg\! \in [\log_2\log(n+1)+3.1, \log_4(\sqrt{n}/\log n)]$ and $\delta \in (0,1)$ 
outputs $w_{\LSKTnotag}$ in $O((\kev\!+\!r\!+\!T)nr\!+\!\kev n^{1.5})$ time satisfying, with probability at least $1-\delta$, 
  \begin{talign}
    &\Delta\!\mmd_{\kp}(w_{\LSKTnotag}) \\
    &\quad= O\left(\!\sqrt{\frac{\kpnormSn \max(1/\delta,\,\ell_\delta (\log n) n^{\rg\cvrd} G_{\sqrt{n}}^\cvrw)}{n}}+\sqrt{\frac{nH_{n,r}}{\delta}}   
    \right),
  \end{talign}
  for $G_m$, $H_{n,r}$ as in \cref{thm:GSKT_guarantee,thm:LSKT_guarantee}.
\end{theorem}

\begin{example}
Under the assumptions of \cref{thm:LSKT_guarantee} with $\slg$-slow-growing input points (\cref{exam:well_controlled}), \loggrowth $\kp$, $T=\Theta(\sqrt{n})$, and $r = (\log n)^{2(\cvrw+\cvrd\slg)+\eps}$ for any $\eps> 0$, 
\LSKT delivers, with high probability,
$\Otilde(n^{-1/2})$ excess MMD %
in $\Otilde(n^{1.5})$ time 
using a coreset of size $m \in [\sqrt{n}, 2\sqrt{n})$.
\end{example}

\section{Weighted Debiased Compression}\label{sec:beyond_unweighted}
The prior sections developed debiased equal-weighted coresets with better-than-\iid compression guarantees.
Equal-weighted coresets are typically chosen for their compatibility with unweighted downstream tasks and easy visualization.
For downstream tasks that support weights, we next match the equal-weighted  guarantees with significantly smaller simplex-weighted or constant-preserving coresets.

\subsection{Simplex-weighted coresets via Stein Recombination\!}
\begin{algorithm}[htb]
  \caption{Recombination Thinning (\recombnotag)}
  \label{alg:recomb}
  \begin{algorithmic}\itemindent=-.7pc
    \STATE {\bfseries Input:} mean-zero kernel $\kp$, points $\cS_n=(x_i)_{i=1}^n$, weights $w \in \simplex$, output size $m$ 
    \STATE $\tilde w \gets \stratresamp(w, n)$
    \STATE $\tI, F \gets \rpc(\kp, \cS_n, \tilde w, m-1)$
    \STATE $w'\gets \recombbfs([F, \bm 1_n]^\top, \tilde w)$ \COMMENT{$[F, \bm 1_n]^\top\!\in \bR^{m\times n}$} %
    \STATE \COMMENT{$F^\top \tilde w = F^\top w'$, $w' \in \simplex$ , and $\norm{w'}_0\le m$}
    \STATE $w'' \!\gets\! \ktswapls(\kp, \cS_n, w', \textup{SPLX})$;
    \ $\tJ \gets\! \{i\!:w''_i > 0\}$
    \STATE $w''[\tJ] \gets \argmin_{w'\in\simplex[\abs{\tJ}-1]}w'^\top \kp(\cS_n[\tJ], \cS_n[\tJ]) w'$ \COMMENT{use any $O(\abs{\tJ}^3)$ quadratic programming solver} 
    \STATE {\bfseries Return:}  $w_{\recombnotag} \gets w'' \in \Delta_{n-1}$ with $\snorm{w_{\recombnotag}}_0 \leq m$
  \end{algorithmic}
\end{algorithm}
Simplex-weighted coresets automatically preserve the constraints of convex integrands, support straightforward direct sampling, and, when compared with schemes involving negative weights, offer improved numerical stability in the presence of integral evaluations errors \citep{karvonen2019positivity}.
Inspired by the coreset constructions of \citet{hayakawa2022positively,hayakawa2023sampling}, we first introduce a simplex-weighted compression algorithm, \recombfull(\recomb, \cref{alg:recomb}), suitable for summarizing a debiased input sequence. %
To produce a coreset given input weights $w \in \simplex$, 
\recomb first prunes small weights using \stratresamp and then 
uses \rpc to identify $m\!-\!1$ test vectors that capture most of the variability in the weighted kernel matrix.
Next, \recombbfs (\cref{alg:recombbfs}) \citep[Alg.~1]{tchernychova2016caratheodory} identifies a sparse simplex vector $w'$ with $\norm{w'}_0 \le m$ that exactly matches the inner product of its input with each of the test vectors.
Then, we run \ktswapls (\cref{alg:kt_swap_ls}), a new, line-search version of \ktswap \citep[Alg. 1b]{dwivedi2021kernel} that greedily improves MMD to $\bP$ while maintaining both the sparsity and simplex constraint of its input.
Finally, we optimize the weights of the remaining support points using any cubic-time quadratic programming solver.

In \cref{prop:recomb_guarantee} we show that \recomb runs in time $O((\kev+m)nm+m^3\log n)$ and nearly preserves the MMD of its input whenever $m$ grows appropriately with the kernel covering number.
Combining \recomb with \GBCfull or \alrbcfull in \cref{alg:SR}, we obtain Stein Recombination (\GSR) and Low-rank SR (\LSR), our approaches to debiased simplex-weighted compression.
Remarkably, \GSR and \LSR can match the MMD error rates established for \GSKT and \LSKT using substantially fewer coreset points, as our next result (proved in \cref{sec:thm:SR_guarantee_proof}) shows.
\begin{algorithm}[htb]
  \caption{(Low-rank) Stein Recombination (\GSRnotag\ / \LSRnotag)}
  \label{alg:SR}
  \begin{algorithmic}\itemindent=-.7pc
    \STATE {\bfseries Input:}  mean-zero kernel $\kp$, points $\cS_n$, output size $m$, 
    rank $r$, AGM steps $T$, adaptive rounds $Q$ %
    \STATE $w \gets 
    \begin{cases}
    \alrbcfull(\kp, \cS_n, r, T, Q) & \textup{if low-rank}\\
     \GBCfull(\kp, \cS_n) & \textup{otherwise}
   \end{cases}$ \\ 
    \STATE $w_{\GSRnotag} \gets \recombfull(\kp, \cS_n, w, m)$
    \STATE {\bf Return:} $w_{\GSRnotag} \in \Delta_{n-1}$ with $\snorm{w_{\GSRnotag}}_0 \leq m$
  \end{algorithmic}
\end{algorithm}

\begin{theorem}[MMD guarantee for \GSR /\LSR]\label{thm:SR_guarantee}
Under \cref{assum:lr}, Stein Recombination (\cref{alg:SR}) takes  $O(\kev n^2\!+\!(\kev\!+\!m)nm\!+\! m^3\log n)$ to output $w_{\GSRnotag}$, and Low-rank SR takes $O((\kev\!+\!r\!+\! T)nr\!+ \!(\kev\!+\!m)nm\!+\!m^3\log n)$ 
time to output $w_{\LSRnotag}$. Moreover, for any $\delta\in(0, 1)$ and $H_{n,r}$ as in \cref{thm:LD_guarantee}, each of the following bounds holds (separately) with probability at least $1-\delta$:
    \begin{talign}
     &\Delta\!\mmd_{\kp}\!(w_{\GSRnotag}) \!=\!O\Big(\sqrt{\frac{\norm{\kp}_n(\log n \vee \frac{1}{\delta})}{n} + \frac{nH_{n,m}}{\delta}}\Big) \stext{and}\\ 
     &\Delta\!\mmd_{\kp}\!(w_{\LSRnotag}) \!=\!O\Big(\!\sqrt{\frac{\norm{\kp}_n(\log n \vee \frac{1}{\delta})}{n} 
     + \frac{n(H_{n,m}+H_{n,r})}{\delta}}\Big)\!.
  \end{talign}
\end{theorem}

\begin{example}
Instantiate the assumptions of \cref{thm:SR_guarantee} with $\slg$-slow-growing input points (\cref{exam:well_controlled}), \loggrowth $\kp$, 
and a heavily compressed coreset size $m = (\log n)^{2(\cvrw+\cvrd\slg)+\eps}$ for any $\eps> 0$.
Then \GSR delivers $\Otilde(n^{-1/2})$
excess MMD with high probability in $O(n^2)$ time, and \LSR with $r\!=\!m$ and $T=\Theta(\sqrt{n})$  achieves the same in $\Otilde(n^{1.5})$ time.
\end{example}

\subsection{Constant-preserving coresets via Stein Cholesky}\label{sec:stein_chol}
For applications supporting negative weights, 
constant-preserving coresets 
offer the possibility of higher accuracy at the cost of potential numerical instability and poorer robustness to errors in integral evaluation \citep{karvonen2019positivity}.
We introduce a constant-preserving compression algorithm, \cholthinfull (\cholthin, \cref{alg:cholthin}), suitable for summarizing a debiased input sequence. 
\cholthin first applies \rpc to a \emph{constant-regularized kernel} $\kp(x,y) + c$ to select an initial coreset and then uses a combination of \ktswapls and closed-form optimal constant-preserving reweighting to greedily refine the support and weights. 
The regularized kernel ensures that the coreset output by \rpc is of high quality when paired with the best constant-preserving weights, and our \cholthin standalone analysis  (\cref{prop:cholthin_guarantee})  improves upon the runtime and error guarantees of \recomb.
In \cref{alg:SC}, we combine \cholthin with \GBCfull or \alrbcfull to obtain Stein Cholesky (\GSC) and Low-rank SC (\LSC), our approaches to debiased constant-preserving compression.
Our MMD guarantees for \GSC and \LSC  (proved in \cref{sec:thm:SC_guarantee_proof}) improve upon the rates of \cref{thm:SR_guarantee}. %
\begin{algorithm}[htb]
  \caption{Cholesky Thinning (\cholthinnotag)}
  \label{alg:cholthin}
  \begin{algorithmic}\itemindent=-.7pc
    \STATE {\bfseries Input:} mean-zero kernel $\kp$, points $\cS_n=(x_i)_{i=1}^n$, weights $w \in \simplex$, output size $m$ 
    \STATE $c \gets \textsc{Average}(\text{Largest $m$ entries of }(\kp(x_i, x_i))_{i=1}^n)$%
    \STATE $\tI, F \gets \rpc(\kp + c, \cS_n, w, m)$;  %
     $w \gets \bm{0}_n$
    \STATE{$w[\tI] \gets \argmin_{w'\in\bR^{\abs{\tI}}:\sum_i w_i'=1}w'^\top \kp(\cS_n[\tI], \cS_n[\tI]) w'$}
    \STATE $w \gets \ktswapls(\kp, \cS_n, w, \textup{CP})$;
    \ \  $\tI \gets \{i:w_{i} \neq 0\}$
    \STATE $w[\tI] \gets \argmin_{w'\in\bR^{\abs{\tI}}:\sum_i w_i'=1}w'^\top \kp(\cS_n[\tI], \cS_n[\tI]) w'$ 
    \STATE {\bfseries Return:} $w_{\cholthinnotag} \gets w\in \real^{n}$ with $\snorm{w_{\cholthinnotag}}_0\leq m$, $\bm 1_n\tp w_{\cholthinnotag}=1$
  \end{algorithmic}
\end{algorithm}

\begin{algorithm}[htb]
  \caption{(Low-rank) Stein Cholesky (\GSCnotag\ / \LSCnotag)}
  \label{alg:SC}
  \begin{algorithmic}\itemindent=-.7pc
    \STATE {\bfseries Input:}  mean-zero kernel $\kp$, points $\cS_n$, output size $m$,
    rank $r$, AGM steps $T$, adaptive rounds $Q$\\
    \STATE $w \gets 
    \begin{cases}
    \alrbcfull(\kp, \cS_n, r, T, Q) & \textup{if low-rank}\\
     \GBCfull(\kp, \cS_n) & \textup{otherwise}
   \end{cases}$ \\ 
    \STATE $w_{\GSCnotag} \gets \cholthinfull(\kp, \cS_n, w, m)$
    \STATE {\bf Return:} $w_{\GSCnotag} \in \real^{n}$ with $\snorm{w_{\GSCnotag}}_0\leq m$ and $\mbf{1}_n\tp w_{\GSCnotag}=1$
  \end{algorithmic}
\end{algorithm}

\begin{figure*}[tbh]
  \centering
  \raisebox{-0.5\height}{
  \begin{subfigure}[b]{0.33\textwidth}
  \includegraphics[width=\textwidth]{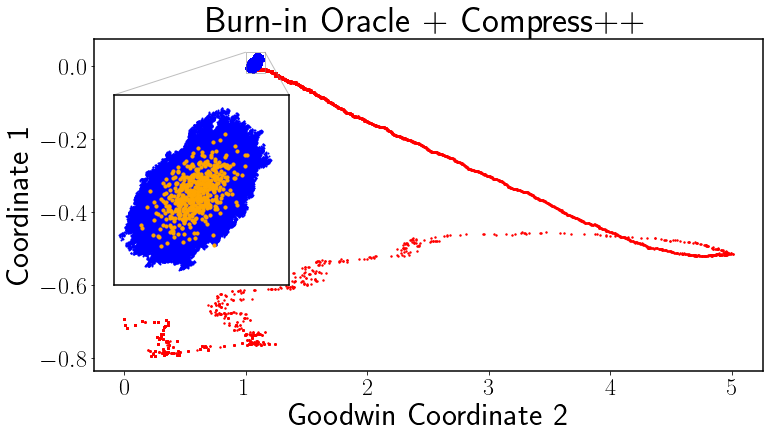}
  \includegraphics[width=\textwidth]{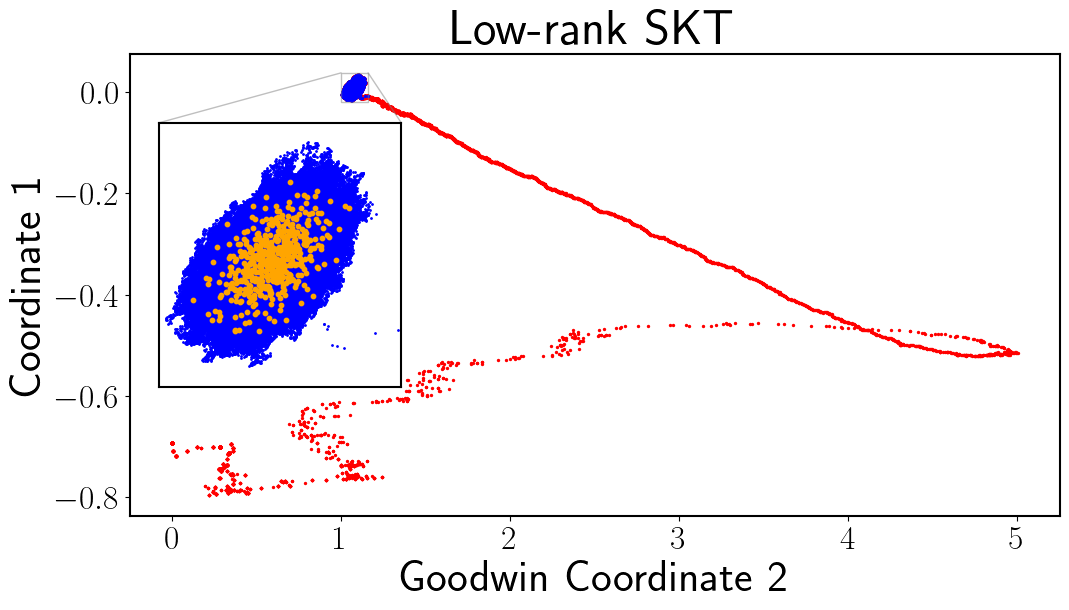}
\end{subfigure}}
  \raisebox{-0.5\height}{\includegraphics[height=3.0in]{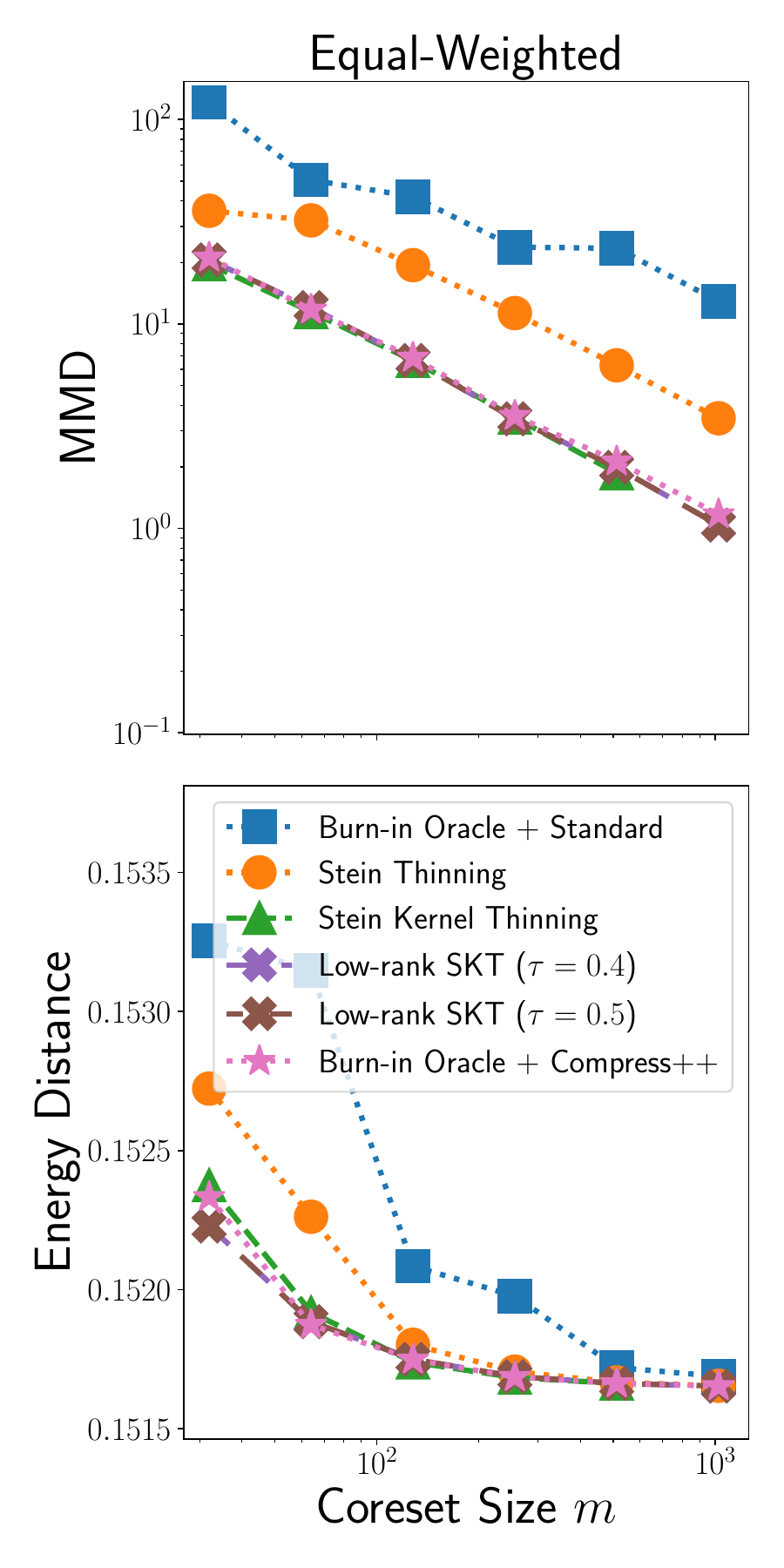}}
  \hspace*{-0.1in}
  \raisebox{-0.5\height}{\includegraphics[height=3.0in]{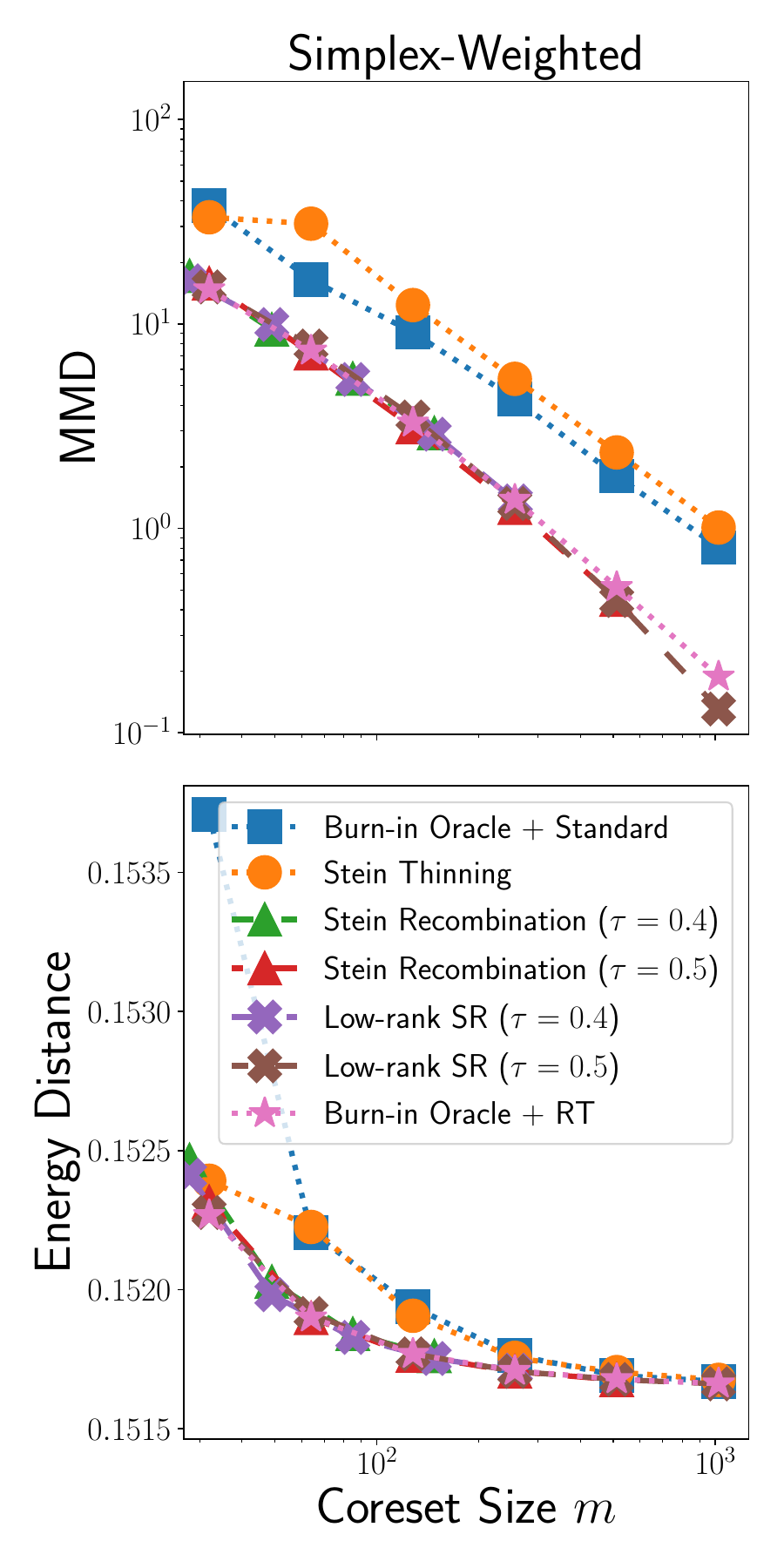}}
  \hspace*{-0.1in}
  \raisebox{-0.5\height}{\includegraphics[height=3.0in]{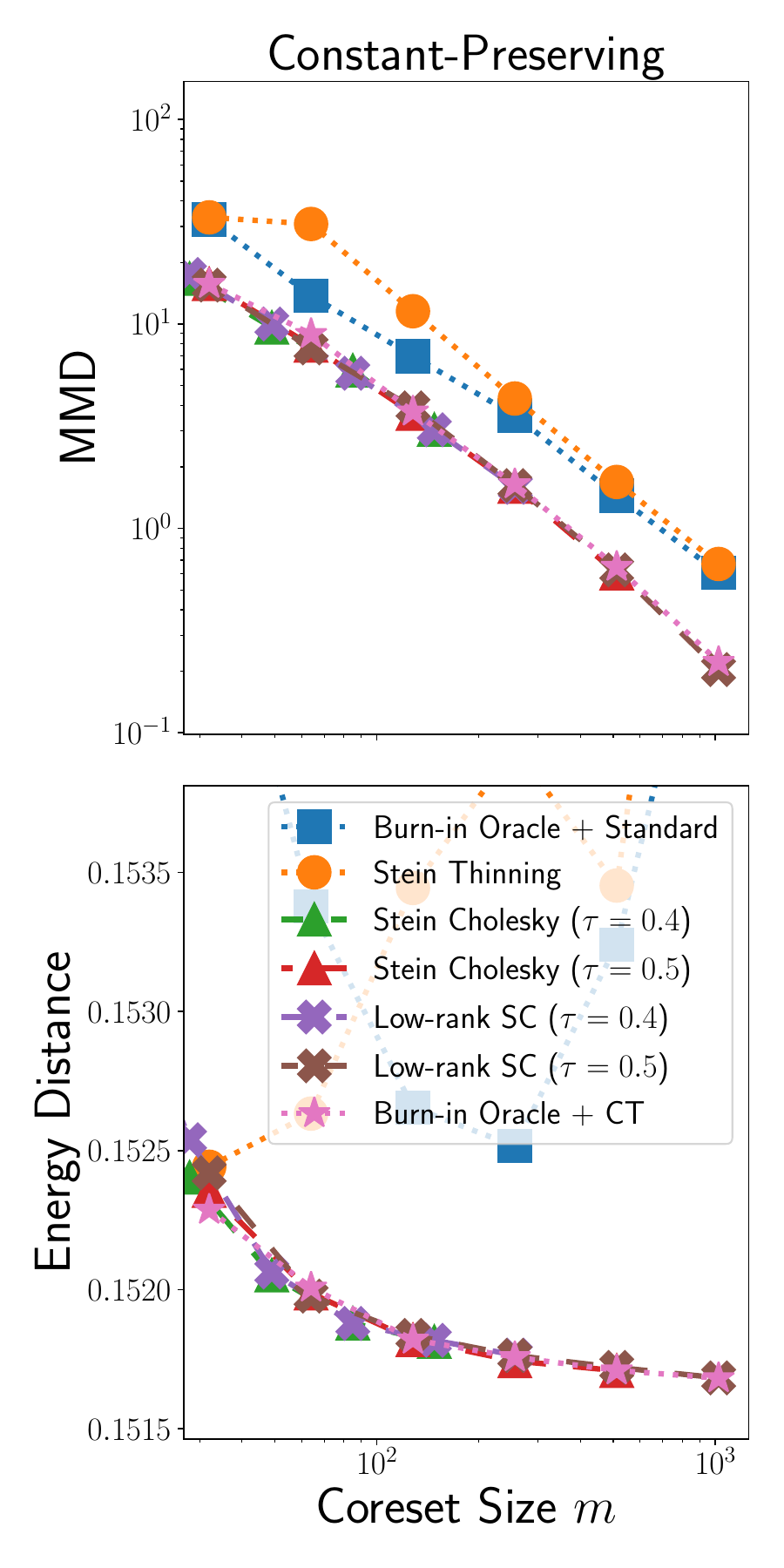}}
  \vspace{-2mm}
  \caption{
    \textbf{Correcting for burn-in.} 
    \emph{Left:} Before selecting coresets (orange), the burn-in oracle uses 6 independent Markov chains to discard burn-in (red) while \LSKT identifies the same high-density region (blue) with 1 chain.
    \emph{Right:} 
    Using only one chain, our methods consistently outperform the Stein and standard thinning baselines and match the 6-chain oracle.}
  \label{fig:goodwin_result_main}
\end{figure*}

\begin{theorem}[MMD guarantee for \GSC\ / \LSC]\label{thm:SC_guarantee}
Under \cref{assum:lr}, Stein Cholesky (\cref{alg:SC}) takes 
$O(\kev n^2 \!+\! (\kev\!+\!m)nm \!+\! m^3)$ time to output $w_{\GSCnotag}$, and Low-rank SC takes $O((\kev \!+\!r\!+\! T)nr\!+\!(\kev\!+\!m)nm\!+\!m^3)$ time to output $w_{\LSCnotag}$.
Moreover, for any $\delta\in(0, 1)$, with probability at least $1-\delta$, each of the following bounds hold:
\vspace{-2mm}
\begin{talign}
     &\Delta\!\mmd_{\kp}(w_{\GSCnotag}) = 2 \mmd_{\mrm{OPT}}\\ & +  O\big(\!\sqrt{\frac{\norm{\kp}_n\log n}{\delta n} + \frac{H_{n,m'}}{\delta}}\big) \stext{and}\\ 
     &\Delta\!\mmd_{\kp}(w_{\LSCnotag}) = 2 \mmd_{\mrm{OPT}} \\ &+ O\big(\!\sqrt{\frac{\norm{\kp}_n(\log n\vee 1/\delta)}{\delta n} + \frac{H_{n,m'}}{\delta}+ \frac{nH_{n,r}}{\delta^2}}\big)
  \end{talign}
  for $H_{n,r}$ as in \cref{thm:LD_guarantee} and $m' \defeq m + \log 2 - 2\sqrt{m\log 2 + 1}$.
\end{theorem}
\begin{example}
Instantiate the assumptions of \cref{thm:SC_guarantee} with $\slg$-slow-growing input points (\cref{exam:well_controlled}), \loggrowth $\kp$, 
and a heavily compressed coreset size $m = (\log n)^{2(\cvrw+\cvrd\slg)+\eps}$ for any $\eps> 0$.
Then \GSC delivers $\Otilde(n^{-1/2})$
excess MMD with high probability in $O(n^2)$ time, and \LSC with $r\!=\!m$ and $T=\Theta(\sqrt{n})$  achieves the same in $\Otilde(n^{1.5})$ time.
\end{example}
\begin{remark}
\label{rem:standard_thin_gsr}
    While we present our results for a target precision of $1/\sqrt{n}$, a coarser target precision of $1/\sqrt{n_0}$ for $n_0 < n$ can be achieved more quickly by random sub-sampling/standard thinning the input sequence down to size $n_0$ before running \GSR, \LSR, \GSC, or \LSC. %
\end{remark}

\begin{figure*}[tbh]
  \centering
  \includegraphics[width=0.3\textwidth]{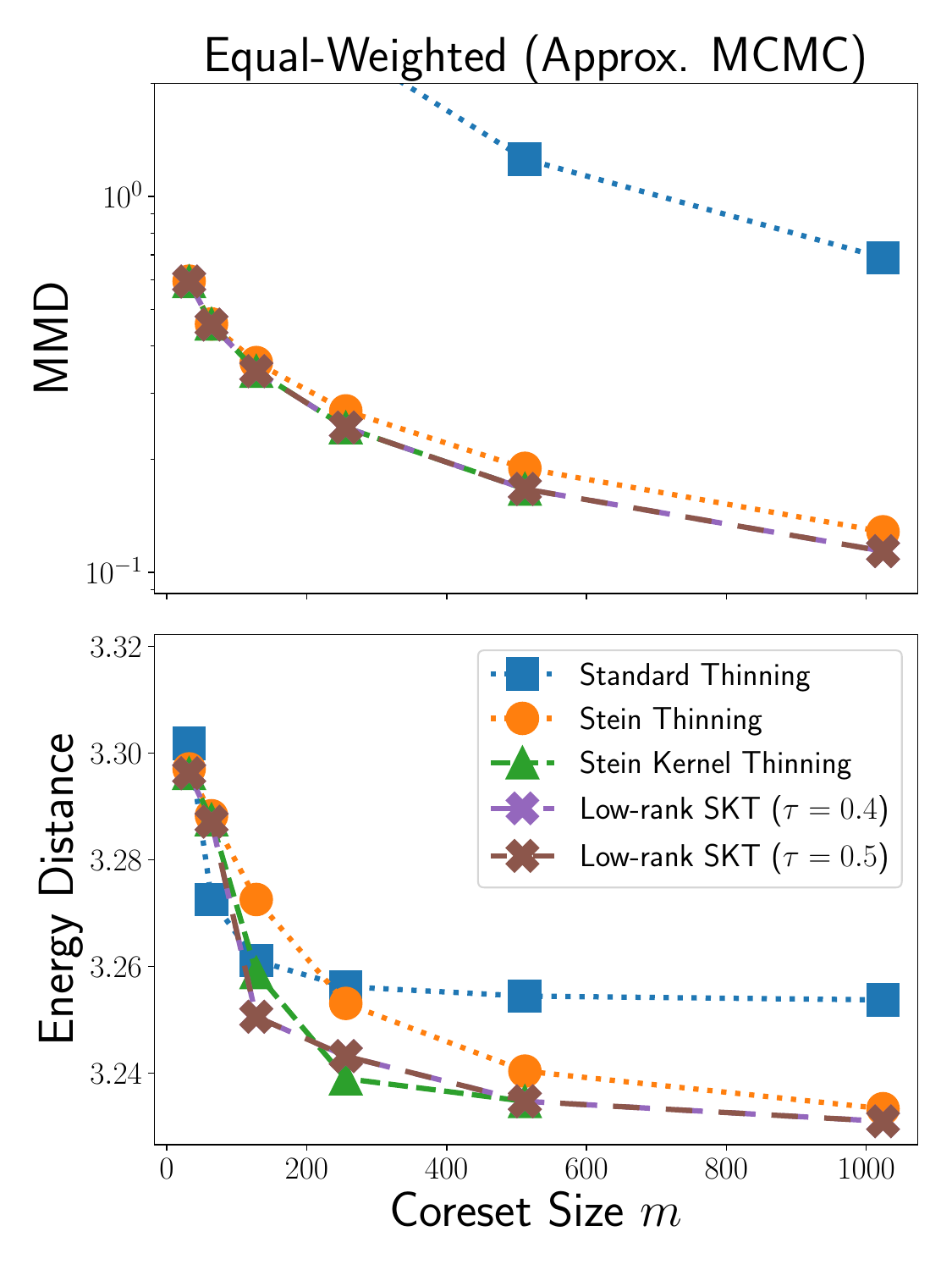}
  \includegraphics[width=0.3\textwidth]{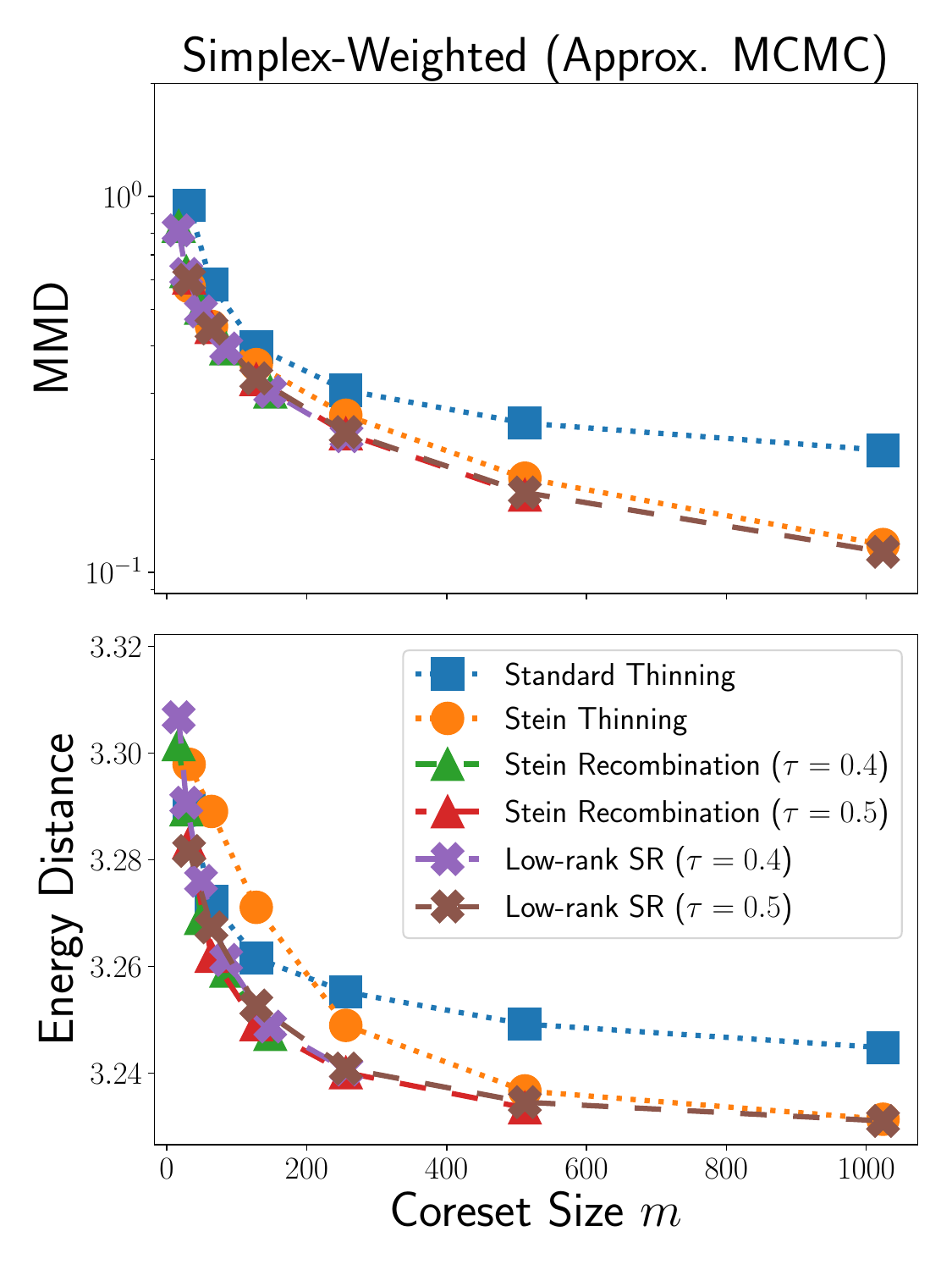}
  \includegraphics[width=0.3\textwidth]{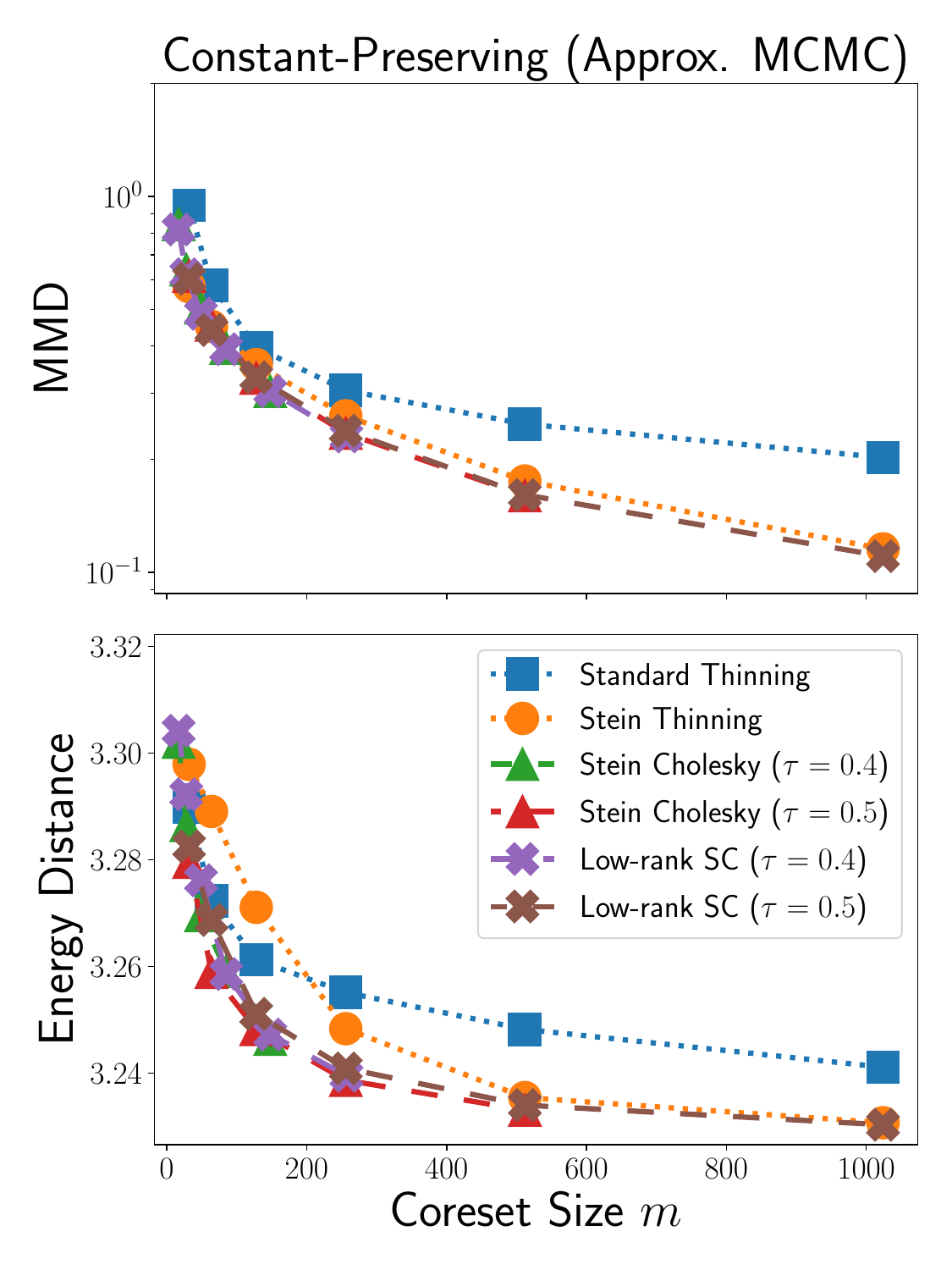}
  \\[-\baselineskip]
    \hrulefill\\
  \includegraphics[width=0.3\textwidth]{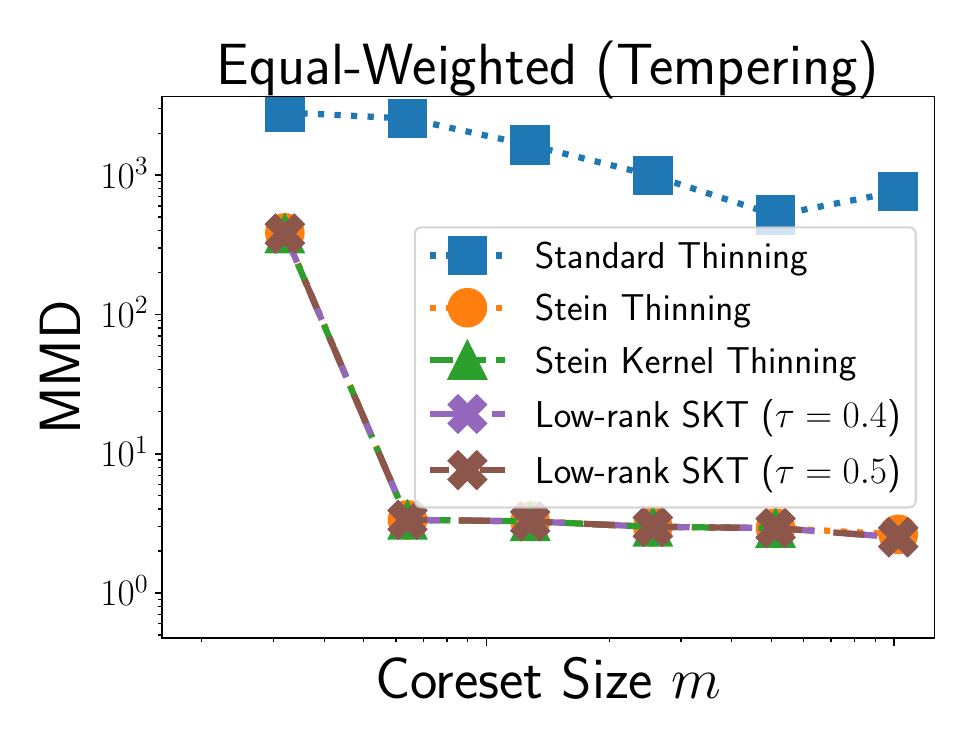}
  \includegraphics[width=0.3\textwidth]{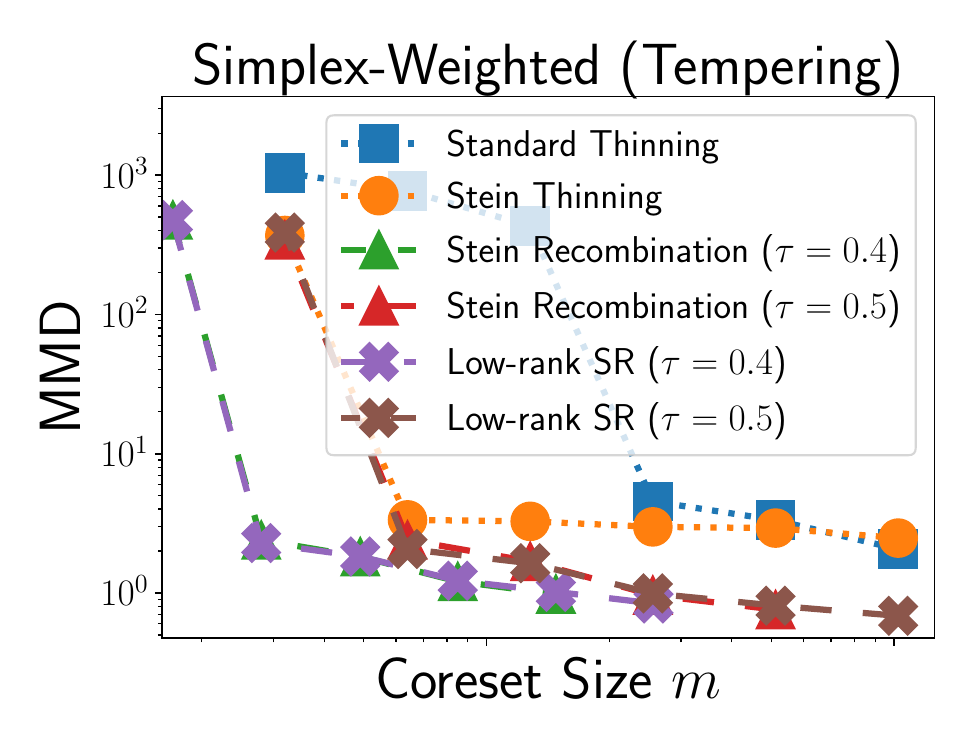}
  \includegraphics[width=0.3\textwidth]{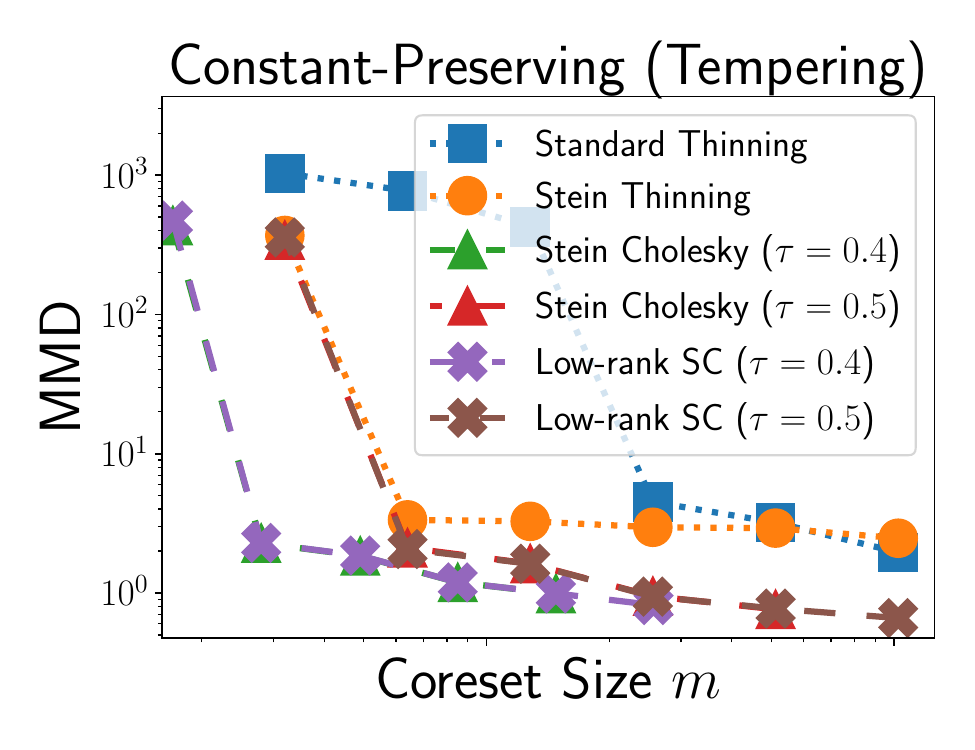}
\caption{\label{fig:covtype_and_cardiac_result_main}\textbf{Correcting for approximate MCMC (top) and tempering (bottom).} For posterior inference over the parameters of Bayesian logistic regression ($d\!=\!54$, top) and a cardiac calcium signaling model ($d\!=\!38$, bottom), our concise coreset constructions correct for approximate MCMC and tempering biases without need for explicit importance sampling.} %
\end{figure*}

\section{Experiments}\label{sec:experiments}
We next evaluate the practical utility of our procedures when faced with three common sources of bias: (1) burn-in, (2) approximate MCMC, and (3) tempering.
In all experiments, we use a Stein kernel $\ksm$ with an inverse multiquadric (IMQ) base kernel $\k(x,y) \!=\! (1\!+\!\staticnorm{x\!-\!y}^2_M/\sigma^2)^{-1/2}$ for $\sigma$ equal to the median pairwise $\staticnorm{\cdot}_M$ distance amongst $1000$ points standard thinned from the input. To vary output MMD precision, we first standard thin the input to size $n_0 \in \{2^{10}, 2^{12}, 2^{14}, 2^{16}, 2^{18}, 2^{20}\}$ before applying any method, as discussed in \cref{rem:standard_thin_gsr,rmk:standard_thin_n_n0}. 
For low-rank or weighted coreset methods, we show results for $m\!=\!r\!=\!n^{\tau}$. 
When comparing weighted coresets, we optimally reweight every coreset. We report the median over 5 independent runs for all error metrics. We implement our algorithms in JAX \citep{jax2018github} 
and refer the reader to \cref{sec:app_experiment} for additional experiment details (including runtime comparison in \cref{tab:goodwin_runtime}). 
Our open-source code is available as part of the GoodPoints Python library at \url{https://github.com/microsoft/goodpoints}.

\para{Correcting for burn-in}
\label{subsec:exp_goodwin}
The initial iterates of a Markov chain are biased by its starting point and need not accurately reflect the target distribution $\bP$. %
Classical burn-in corrections use convergence diagnostics to detect and discard these iterates but typically require running multiple independent Markov chains \citep{cowles1999possible}. %
Alternatively, our proposed debiased compression methods can be used to correct for burn-in given just a single chain. 

We test this claim using an experimental setup from \citet[Sec.~4.1]{riabiz2022optimal} and the  $6$-chain ``burn-in oracle'' diagnostic of~\citet{vats2021revisiting}.\footnote{The \citeauthor{vats2021revisiting} diagnostic is an improvement on the popular \citet{gelman1992inference} diagnostic (GRD). The GRD was designed to diagnose MCMC convergence rather than burn-in but is commonly used to identify and discard burn-in points automatically through packages like \texttt{coda} \citep{plummer2006coda}.} We aim to compress a posterior $\P$ over the parameters in the Goodwin model of oscillatory enzymatic control ($d\!=\!4$)
using $n\!=\!2\!\times\! 10^6$ points from a preconditioned Metropolis-adjusted Langevin algorithm (P-MALA) chain. 
We repeat this experiment with three alternative MCMC algorithms in \cref{subsec:exp_goodwin_app}. 
Our primary metric is $\mmd_{\kp}$ to $\P$ with  $M=I$, but, for external validation, we also measure the energy distance~\citep[Eq.~11]{riabiz2022optimal} to an auxiliary MCMC chain of length $n$.  
Trajectory plots of the first two coordinates (\cref{fig:goodwin_result_main}, left) highlight the substantial burn-in period for the Goodwin chain and the ability of \LSKT to mimic the 6-chain burn-in oracle using only a single chain. In \cref{fig:goodwin_result_main}  (right), for both the MMD metric and the auxiliary energy distance, our proposed methods consistently outperform Stein thinning and match the quality of 6-chain burn-in removal paired with unbiased compression. %
The spike in baseline energy distance for the constant-preserving task can be attributed to the selection of overly large weight values due to poor matrix conditioning; the simplex-weighted task does not suffer from this issue due to its regularizing nonnegativity constraint. 

\para{Correcting for approximate MCMC}
\label{subsec:exp_covtype}
In posterior inference, MCMC algorithms typically require iterating over every datapoint to draw each new sample point.
When datasets are large, approximating MCMC using datapoint mini-batches can reduce sampling time at the cost of persistent bias and an unknown stationary distribution that prohibits debiasing via importance sampling. 
Our proposed methods can correct for these biases during compression by computing full-dataset scores on a small subset of $n_0$ standard thinned points.
To evaluate this protocol, we compress a Bayesian logistic regression posterior conditioned on the Forest Covtype dataset ($d\!=\!54$) using $n\!=\!2^{24}$ approximate MCMC points from the stochastic gradient Fisher scoring sampler \citep{ahn2012bayesian} with batch size $32$. 
Following \citet{wang2024stein}, we set $M\!=\!-\Hess\!\log p(x_{\textup{mode}})$ at the sample mode $x_{\textup{mode}}$ and use $2^{20}$ surrogate ground truth points from the No U-turn Sampler~\citep{hoffman2014no} to evaluate energy distance. We find that our proposals improve upon standard thinning and Stein thinning for each compression task, not just in the optimized MMD metric (\cref{fig:covtype_and_cardiac_result_main}, top) but also in the auxiliary energy distance (\cref{fig:covtype_and_cardiac_result_main}, middle) and when measuring integration error for the mean (\cref{fig:covtype_mse}).

\para{Correcting for tempering}
\label{subsec:exp_cardiac}
Tempering, targeting a less-peaked and more dispersed distribution $\Q$, is a popular technique to improve the speed of MCMC convergence. One can correct for the sample bias using importance sampling, but this requires knowledge of the tempered density and can introduce substantial variance \citep{gramacy2010importance}. 
Alternatively, one can use constructions of this work to correct for tempering during compression; this requires no importance weighting and no knowledge of $\Q$.
To test this proposal, we compress the cardiac calcium signaling model posterior ($d=38$) of \citet[Sec. 4.3]{riabiz2022optimal} with $M=I$ and $n=3\times10^6$ tempered points from a Gaussian random walk Metropolis-Hastings chain. As discussed by \citeauthor{riabiz2022optimal}, compression is essential in this setting as the ultimate aim is to propagate posterior uncertainty through a human heart simulator, a feat which requires over $1000$ CPU hours for each summary point retained.  
Our methods perform on par with Stein thinning for equal-weighted compression and yield substantial gains over Stein (and standard) thinning for the two weighted compression tasks.

\section{Conclusions and Future Work}

We have introduced and analyzed a suite of new procedures for compressing a biased input sequence into an accurate summary of a target distribution.  For equal-weighted compression, Stein kernel thinning delivers $\sqrt{n}$ points with $\Otilde(n^{-1/2})$ MMD in $O(n^2)$ time, and low-rank SKT can improve this running time to $\Otilde(n^{3/2})$.  For simplex-weighted and constant-preserving compression, Stein recombination and Stein Cholesky provide enhanced parsimony, matching these guarantees with as few as $\polylog(n)$ points.
Recent work has identified some limitations of score-based discrepancies, like Stein kernel MMDs, and developed modified objectives that are more sensitive to the relative density of isolated modes \citep{liu2023using,benard2024kernel}.  A valuable next step would be to extend our constructions to provide compression guarantees for these modified discrepancy measures.
Other opportunities for future work include marrying the better-than-\iid guarantees of this work with the non-myopic compression of \citet{teymur2021optimal}, the control-variate compression of \citet{chopin2021fast}, and the online compression of \citet{hawkins2022online}.

\section*{Acknowledgments}
We thank Marina Riabiz for making the Markov chain data used in  \cref{sec:experiments} available and Jeffrey Rosenthal for helpful discussions concerning the geometric ergodicity of Markov chains.
Lingxiao Li acknowledges the generous support of Army Research Office grants W911NF2010168 and W911NF2110293, of Air Force Office of Scientific Research award FA9550-19-1-031, of National Science Foundation grant CHS-1955697, from the CSAIL Systems that Learn program, from the MIT–IBM Watson AI Laboratory, from the Toyota–CSAIL Joint Research Center, from a gift from Adobe Systems, and from a Google Research Scholar award.

\section*{Impact Statement}
This paper presents work with the aim of advancing the field of Machine Learning. There are many potential societal consequences of our work, none which we feel must be specifically highlighted here.

\bibliography{main}
\bibliographystyle{abbrvnat}
\newpage
\appendix
\onecolumn
     \etoctocstyle{1}{Appendix Contents}
    \etocdepthtag.toc{mtappendix}
    \etocsettagdepth{mtchapter}{none}
    \etocsettagdepth{mtappendix}{section}
    \etocsettagdepth{mtappendix}{subsection}
    \etocsettagdepth{mtappendix}{subsubsection}
    {\small\tableofcontents}
\numberwithin{lemma}{section} 
\numberwithin{proposition}{section} 
\numberwithin{definition}{section} 
\numberwithin{theorem}{section} 
\numberwithin{algorithm}{section} 
\numberwithin{figure}{section} 
\numberwithin{table}{section} 
\numberwithin{corollary}{section} 
\numberwithin{example}{section} 

\section{Appendix Notation}
For the point sequence $\cS_n = (x_i)_{i\in[n]}$, we use $\bS_n \defeq \frac{1}{n}\sum_{i\in[n]}\delta_{x_i}$ to denote the empirical distribution.
For a weight vector $w \in \bR^n$, we define the support $\supp(w) \defeq \{i \in [n]: w_i \neq 0\}$ and the signed measure $\bS_n^w  \defeq \sum_{i\in [n]}w_i \delta_{x_i}$.
For a matrix $K \in \bR^{n\times n}$ and $w \in \simplex$, we define the weighted matrix $K^w \defeq \diag(\sqrt{w})K\diag(\sqrt{w})$.
For positive semidefinite (PSD) matrices $(A, B)$, we use $A \succeq B$ (resp. $A \preceq B$) to mean $A - B$ (resp. $B - A$) is PSD.
For a symmetric PSD (SPSD) matrix $M$, we let $M^{1/2}$ denote a symmetric matrix square root satisfying $M=M^{1/2}M^{1/2}$.
For $A \in \bR^{n\times m}$, we denote $\norm{A}_p \defeq \sup_{x\neq 0} \frac{\norm{Ax}_p}{\norm{x}_p}$. 
We will use $\ind{E}$ to denote the indicator function for an event $E$.

Notation used only in a specific section will be introduced therein.

\section{Spectral Analysis of Kernel Matrices}
The goal of this section is to develop spectral bounds for kernel matrices.

\para{From covering numbers of kernels to eigenvalues of kernel \emph{matrices}} In \cref{subsec:T_k_spectral}, we transfer the bounds on covering numbers from the definition of \polygrowth or \loggrowth kernels to bounds on the eigenvalues of the kernel matrices.
This sets the theoretical foundation for the algorithms in later sections as their error guarantees rely on the fast decay of eigenvalues of kernel matrices.

\para{Covering numbers and eigenvalues of Stein kernels}
In \cref{subsec:stein_spectral}, we show that Stein kernels are \polygrowth (resp. \loggrowth) provided that their base kernels are differentiable (resp. radially analytic).
Putting the results together from \cref{subsec:T_k_spectral}, we obtain spectral bounds for a wide range of Stein kernels.

\para{Notation}
For a normed space $E$, we use $\norm{\cdot}_E$ to denote its norm, $\ball{E}(p, r) \defeq \{x \in E: \norm{x- p}_E \le r\}$ to denote the closed ball of radius $r$ centered at $p$ in $E$ with the shorthand  $\ball{E}(r) \defeq \ball{E}(0, r)$ and $\ball{E} \defeq \ball{E}(1)$.
When $E$ is an RKHS with kernel $\k$, for brevity we use $\k$ in place of $E$ in the subscript.
Let $\funcspace(\cX, \cY)$ denote the space of functions from $\cX$ to $\cY$, and $\boundedop(E, F)$ denote the space of bounded linear functions between normed spaces $E, F$.
For a set $A$, we use $\ellinf(A)$ to denote the space of bounded $\bR$-valued functions on $A$ equipped with the sup-norm $\norm{f}_{\infty,A} \defeq \sup_{x\in A} |f(x)|$.
We use $\idmap{E}{F}$ to denote the inclusion map.
e use $\lambda_\ell(T)$ to denote the $\ell$-th largest eigenvalue of an operator $T$.

\subsection{From covering numbers of kernels to eigenvalues of kernel \emph{matrices}}
\label{subsec:T_k_spectral}
We first introduce the general Mercer representation theorem from \citet{steinwart2012mercer}, which shows the existence of a discrete spectrum of the integral operator associated with a continuous square-integrable kernel.
The theorem also provides a series expansion of the kernel, i.e., the Mercer representation, in terms of the eigenvalues and eigenfunctions.
\begin{lemma}[General Mercer representation \citep{steinwart2012mercer}]\label{lem:general_mercer}
  Consider a kernel $\k: \bR^d \times \bR^d \to \bR$ that is jointly continuous in both inputs and a probability measure $\mu$ such that $\int \k(x,x)\dd\mu(x) < \infty$.
  Then the following holds.
  \begin{enumerate}[label=(\alph*)]
    \item\label{itm:general_mercer:cpt_embed}
      The inclusion $\idmap{\rkhs[\k]}{\cL^2(\mu)}$ is a compact operator, i.e., $\ball{\k}$ is a compact subset of $\cL^2(\mu)$. In particular, this inclusion is continuous.
    \item\label{itm:general_mercer:int_op}
      The Hilbert-space adjoint of the inclusion $\idmap{\rkhs[\k]}{\cL^2(\mu)}$ is the compact operator $S_{\k,\mu}: \cL^2(\mu) \to \rkhs[\k]$ defined as
      \begin{talign}
        S_{\k,\mu} f \defeq \int \k(\cdot, x) f(x)\dd\mu(x).\label{eqn:int_op_S_k}
      \end{talign}
      We also have $S_{\k,\mu}^* \defeq \idmap{\rkhs[\k]}{\cL^2(\mu)}$.
      Hence the operator 
      \begin{talign}
        T_{\k,\mu} \defeq S^*_{\k,\mu} S_{\k,\mu}: \cL^2(\mu) \to \cL^2(\mu)\label{eqn:int_op_T_k}
      \end{talign}
      is also compact.
    \item\label{itm:general_mercer:ons}
      There exist $\{\lambda_\ell\}_{\ell=1}^\infty$ with $\lambda_1 \ge \lambda_2 \ge \cdots \ge 0$ and $\{\phi_\ell\}_{\ell=1}^\infty \subset \rkhs[\k]$ such that $\{{\phi_\ell}\}_{\ell=1}^\infty$ is an orthonormal system in $\cL^2(\mu)$
      and $\{\lambda_\ell\}_{\ell=1}^\infty$ (resp. $\{{\phi_\ell}\}_{\ell=1}^\infty$) consists of the eigenvalues (resp. eigenfunctions) of $T_{\k,\mu}$ with eigendecomposition, for ${f} \in \cL^2(\mu)$,
      \begin{talign}
        T_{\k,\mu} {f} = \sum_{\ell = 1}^\infty \lambda_\ell \langle {f}, {\phi_\ell} \rangle_{\cL^2(\mu)}{\phi_\ell}
      \end{talign}
      with convergence in $\cL^2(\mu)$.
    \item\label{itm:general_mercer:series_rep}
      We have the following series expansion
      \begin{talign}
        \k(x, x') &= \sum_{\ell=1}^\infty \lambda_i \phi_\ell(x) \phi_\ell(x'), \label{eqn:mercer_series_rep}
      \end{talign}
      where the series convergence is absolute and uniform in $x,x'$ on all $A\times A \subset \supp\mu\times\supp\mu$.
  \end{enumerate}
\end{lemma}

\begin{proof}[Proof of \cref{lem:general_mercer}]
  Part~\cref{itm:general_mercer:cpt_embed,itm:general_mercer:int_op} follow respectively from \citet[Lem.~2.3 and 2.2]{steinwart2012mercer}.  Part~\cref{itm:general_mercer:ons} follows from part~\cref{itm:general_mercer:cpt_embed} and \citet[Lem.~2.12]{steinwart2012mercer}. Finally, part~\cref{itm:general_mercer:series_rep} follows from \citet[Cor.~3.5]{steinwart2012mercer}.
\end{proof}

We will use the following lemma regarding the restriction of covering numbers.
\begin{lemma}[Covering number is preserved in restriction]\label{lem:cvrnum_preserve_restrict}
  For a kernel $\k: \bR^d \times \bR^d \to \bR$ and a set $A \subset \bR^d$, we have $\cvrnum{\k}(A, \eps) = \cvrnum{\k|_A}(A, \eps)$, for $\k|_A$, the restricted kernel of $\k$ to $A$ \citep[Sec.~5.4]{paulsen2016introduction}.
\end{lemma}
\begin{proof}[Proof of \cref{lem:cvrnum_preserve_restrict}]
  It suffices to show that a $(\k, A, \eps)$ cover can be converted to a cover of $(\k|_A, A, \eps)$ of the same cardinality and vice versa.

  Let $\cC \subset \ball{\k|_A}$ be a $(\k|_A, A, \eps)$ cover.
  For any $f \in \cC$, we have $\snorm{f}_{\k|_A} = \inf\left\{\snorm{\tilde f}_{\k}: \tilde f \in \rkhs[\k], \tilde f|_A = f\right\} \le 1$ \citep[Corollary 5.8]{paulsen2016introduction}.
  Moreover, the infimum is attained by some $\tilde f \in \rkhs[\k]$ such that $\snorm{\tilde f}_{\k} = \snorm{f}_{\k|_A} \le 1$ and $\tilde f|_A = f$.
  Now form $\tilde \cC = \{\tilde f: f \in \cC\}$.
  For any $\tilde h \in \ball{\k}$, there exists $f \in \cC$ such that
  \begin{talign}
    \norm{\tilde{h}|_A - f}_{\infty,A}  \le \eps \implies = \norm{\tilde h- \tilde f}_{\infty,A} \le \eps,
  \end{talign}
  so $\tilde\cC$ is a $(\k|_A, A, \eps)$ cover.

  For the other direction, let $\tilde\cC \subset \ball{\k}$ be a $(\k, A, \eps)$ cover.
  Define $\cC = \{\tilde f|_A : \tilde f \in \tilde\cC\} \subset \rkhs[\k|_A]$.
  Since $\snorm{\tilde f|_A}_{\k|_A} \le \snorm{\tilde f}_{\k}$, we have $\cC \subset \ball{\k|_A}$.
  For any $h \in \ball{\k_A]}$, again by \citet[Corollary 5.8]{paulsen2016introduction}, there exists $\tilde h \in \rkhs[\k]$ such that $\snorm{\tilde h}_{\k} = \snorm{h}_{\k|_A} \le 1$, so there exists $\tilde f \in \tilde\cC$ such that 
  \begin{talign}
    \norm{\tilde{h} - \tilde f}_{\infty,A} \le \eps \implies \norm{h - \tilde f|_A}_{\infty,A} \le \eps,
  \end{talign}
  Hence $\cC$ is a $(\k, A, \eps)$ cover.
\end{proof}

The goal for the rest of this section is to transfer the bounds of the covering number in the definition of a \polygrowth or \loggrowth kernel from \cref{assum:kernel_growth} to bounds on entropy numbers \citep[Def.~6.20]{steinwart2008support} that are closely related to eigenvalues of the integral operator \eqref{eqn:int_op_T_k}.
\begin{definition}[Entropy number of a bounded linear map]\label{def:ent_num}
  For a bounded linear operator $S: E \to F$ between normed spaces $E, F$, for $\ell \in \bN$, the \emph{$\ell$-th entropy number} of $S$ is defined as
  \begin{talign}
      e_\ell(S) \defeq \inf\left\{\eps > 0: \exists s_1,\ldots,s_{2^{\ell-1}} \in S(\ball{E}) \text{ such that } 
      S(\ball{E}) \subset \bigcup_{i=1}^{2^{\ell-1}} \ball{F}(s_i, \eps)
      \right\}.
  \end{talign}
\end{definition}

The following lemma shows the relation between covering numbers and entropy numbers.
\begin{lemma}[Relation between covering number and entropy number]\label{lem:ent_num_cvr_num_relation}
  Suppose a kernel $\k$ is jointly continuous and $A \subset \bR^d$ is bounded. Then for any $\eps > 0$,
  \begin{talign}
    e_{\lceil \log_2 \cvrnum{\k}(A, \eps) \rceil+1}(\idmap{\rkhs[\k|_A]}{\ell_\infty(A)}) \le \eps.
  \end{talign}
\end{lemma}
\begin{proof}[Proof of \cref{lem:ent_num_cvr_num_relation}]
  First, the assumption implies $\k|_A$ is a bounded kernel, so by \citet[Lemma 4.23]{steinwart2008support}, the inclusion $\idmap{\rkhs[\k|_A]}{\ell_\infty(A)}$ is continuous.
  By the definition of $\cvrnum{\k|_A}(A, \eps)$, by adding arbitrary elements into the cover if necessary, there exists a $(\k|_A, A, \eps)$ cover of $\ball{\k|_A}$ of cardinality $2^{\lceil \log_2(\cvrnum{\k|_A}(A,\eps))\rceil} \ge \cvrnum{\k|_A}(A, \eps)$.
  Hence
  \begin{talign}
    e_{\lceil \log_2 \cvrnum{\k|_A}(A, \eps) \rceil+1}(\idmap{\rkhs[\k|_A]}{\ell_\infty(A)}) \le \eps.
  \end{talign}
  The claim follows since $\cvrnum{\k|_A}(A, \eps)= \cvrnum{\k}(A, \eps)$  by \cref{lem:cvrnum_preserve_restrict}.
\end{proof}

\begin{proposition}[$\ellinf$-entropy number bound for \polygrowth or \loggrowth $\k$]\label{prop:ent_num_bound_poly_or_log_growth}
  Suppose a kernel $\k$ satisfies \cref{assum:kernel_growth}.
  Let $\cvrC > 0$ denote the constant that appears in the \cref{assum:kernel_growth}.
  Define 
  \begin{talign}
    \krcd_\k(r) \defeq \frac{\cvrC}{\log 2}r^\cvrd.\label{eqn:krcd}
  \end{talign}
 Then for any $r > 0$ and $\ell \in \bN$ that satisfies $\ell > \krcd_\k(r+1) + 1$, we have
      \begin{talign}
        e_\ell(\idmap{\rkhs[\k|_{\balleuc(r)}]}{\ellinf(\balleuc(r))}) \le
        \begin{cases}
            \left(\frac{\krcd_\k(r+1)}{\ell-1}\right)^{\frac{1}{\cvrw}}
            &\text{if } \k \stext{is} \polygrowth(\alpha,\beta), \text{ and} \\ 
            \exp(1-\left(\frac{\ell-1}{\krcd_\k(r+1)}\right)^{\frac{1}{\cvrw}})&\text{if } \k \stext{is} \loggrowth(\alpha,\beta). 
        \end{cases}
        \label{eqn:ent_num_bound_poly_growth}
      \end{talign}
\end{proposition}

\begin{proof}[Proof of \cref{prop:ent_num_bound_poly_or_log_growth}]
  By \cref{lem:ent_num_cvr_num_relation} and the fact that $e_\ell$ is monotonically decreasing in $\ell$ by definition, if $\ell \ge  \log_2 \cvrnum{\k}(\balleuc(r), \eps)  + 1$ for some $\eps > 0$, then
  \begin{talign}
    e_\ell(\idmap{\rkhs[\k|_{\balleuc(r)}]}{\ellinf(\balleuc(r))}) \le e_{\lceil \log_2 \cvrnum{\k}(\balleuc(r), \eps) \rceil+1}(\idmap{\rkhs[\k|_{\balleuc(r)}]}{\ellinf(\balleuc(r))}) \le \eps. \label{eqn:loc:ent_num_self_bound}
  \end{talign}
  For the \polygrowth case, by its definition, the condition $\ell \ge  \log_2 \cvrnum{\k}(\balleuc(r), \eps)  + 1$ is met if $\eps \in (0, 1)$ and
  \begin{talign}
    \ell \ge \frac{\cvrC}{\log 2}(1/\eps)^{\cvrw}(r+1)^\cvrd  + 1
    \Longleftrightarrow \eps \le \left(\frac{\krcd_\k(r+1)}{\ell-1}\right)^{\frac{1}{\cvrw}}.
  \end{talign}
  Hence \eqref{eqn:loc:ent_num_self_bound} holds with $\eps = \left(\frac{\krcd_\k(r+1)}{\ell-1}\right)^{\frac{1}{\cvrw}}$, as long as $\eps \in (0, 1)$, so $\ell$ needs to satisfy
  \begin{talign}
    1  &> \left(\frac{\krcd_\k(r+1)}{\ell-1}\right)^{\frac{1}{\cvrw}} 
    \Longleftrightarrow \ell > \krcd_\k(r+1) + 1.
  \end{talign}
  Similarly, for the \loggrowth case, the condition $\ell \ge  \log_2 \cvrnum{\k}(\balleuc(r), \eps)  + 1$ is met if $\eps \in (0, 1)$ and
  \begin{talign}
    \ell \ge \frac{\cvrC}{\log 2}(\log(1/\eps)+1)^{\cvrw}(r+1)^\cvrd  + 1
    \Longleftrightarrow \eps \le \exp(1-\left(\frac{\ell-1}{\krcd_\k(r+1)}\right)^{\frac{1}{\cvrw}}).
  \end{talign}
  Hence \eqref{eqn:loc:ent_num_self_bound} holds with $\eps = \exp(1-\left(\frac{\ell-1}{\krcd_\k(r+1)}\right)^{\frac{1}{\cvrw}})$, as long as $\eps \in (0, 1)$, so $\ell$ needs to satisfy
  \begin{talign}
    1  &> \exp(1-\left(\frac{\ell-1}{\krcd_\k(r+1)}\right)^{\frac{1}{\cvrw}}) 
    \Longleftrightarrow \ell > \krcd_\k(r+1) + 1.
  \end{talign}
\end{proof}

Next, we show that we can transfer bounds on entropy numbers to obtain bounds for the eigenvalues of kernel matrices, which will become handy when we develop sub-quadratic-time algorithms in \cref{sec:sub_quad_time}.
We rely on the following lemma, which summarizes the relevant facts from \citet[Appendix A]{steinwart2008support}.
\begin{lemma}[Eigenvalue is bounded by entropy number]\label{lem:eigval_bounded_by_ent_num}
  Let $\k$ be a jointly continuous kernel and $\bP$ be a distribution such that $\bE_{x\sim\bP}[\k(x,x)] < \infty$, and recall that $\lambda_\ell(\cdot)$ denotes the $\ell$-th largest eigenvalue of a linear operator.
  Then, for all $\ell \in \bN$, %
  \begin{talign}
    \lambda_\ell(T_{\k,\bP}) \le 4e_\ell^2(\idmap{\rkhs[\k]}{\cL^2(\bP)}).
  \end{talign}
\end{lemma}
\begin{proof}[Proof of \cref{lem:eigval_bounded_by_ent_num}]
  For any bounded linear operator $S: \cH_1 \to \cH_2$ between Hilbert spaces $\cH_1$ and $\cH_2$, we have
  $a_\ell(S) \le 2e_\ell(S)$, where $a_\ell$ is the $\ell$-th approximation number defined in \citet[(A.29)]{steinwart2008support}.
  Recall the operator $S^*_{\k,\bP} = \idmap{\rkhs[\k]}{\cL^2(\bP)}$ from \eqref{eqn:int_op_S_k}, which is compact (in particular bounded) by \cref{lem:general_mercer}\ref{itm:general_mercer:cpt_embed}.
  Thus
  \begin{talign}
    s_\ell(S^*_{\k,\bP}) = a_\ell(S^*_{\k,\bP}) \le 2e_\ell(S^*_{\k,\bP}),
  \end{talign}
  where the first equality follows from the paragraph below \citet[(A.29)]{steinwart2008support}) and $s_\ell$ is $\ell$-th singular number of an operator \citep[(A.25)]{steinwart2008support}.
  Then using the identities mentioned under \citet[(A.25)]{steinwart2008support} and \citet[(A.27)]{steinwart2008support} and that all operators involved are compact by \cref{lem:general_mercer}\ref{itm:general_mercer:int_op}, we have
  \begin{talign}
    \lambda_\ell(T_{\k,\bP}) = \lambda_\ell(S^*_{\k,\bP}S_{\k,\bP}) = s_\ell(S^*_{\k,\bP}S_{\k,\bP}) = s^2_\ell(S^*_{\k,\bP}) \le 4e^2_\ell(S^*_{\k,\bP}).
  \end{talign}
\end{proof}

The previous lemma allows us to bound eigenvalues of kernel matrices by $\ellinf$-entropy numbers.
\begin{proposition}[Eigenvalue of kernel matrix is bounded by $\ellinf$-entropy number]
  \label{prop:eigval_K_bounded_by_ent_num}
  Let $\k$ be a jointly continuous kernel.
  Define $K \defeq \k(\cS_n,\cS_n)$ for the sequence of points $\cS_n = (x_1,\ldots,x_n)\subset \bR^d$.
  For any $w \in \simplex$, recall the notation $\bS_n^w = \sum_{i\in[n]} w_i \delta_{x_i}$, $K^w = \diag(\sqrt{w}) K \diag(\sqrt{w})$, and $R_n = 1 + \sup_{i\in[n]}\twonorm{x_i}$.
  Then for all $\ell\in\bN$,
  \begin{talign}
    \lambda_\ell(K^w) \seq{(i)} \lambda_\ell(T_{\k, \bS_n^w}) \sless{(ii)} 4e^2_\ell(\idmap{\rkhs[\k]_{|_{\balleuc(R_{n}-1)}}}{\ellinf(\balleuc(R_{n}-1))}).
    \label{eq:eigenvalue_bound}
  \end{talign}
\end{proposition}
\begin{proof}[Proof of \cref{prop:eigval_K_bounded_by_ent_num}]
  Without loss of generality, we assume $w_i > 0$ for all $i \in [n]$, since otherwise, we can consider a smaller set of points by removing the ones with zero weights.

  \para{Proof of equality~(i) from display~\cref{eq:eigenvalue_bound}}
  Note that $\cL^2(\bS_n^w)$ is isometric to $\bR^n$.
  Let $K \defeq \k(\cS_n, \cS_n)$ denote the kernel matrix.
  The action of $T_{\k, \bS_n^w}$ is given by, for $i \in [n]$,
  \begin{talign}
    T_{\k, \bS_n^w} f(x_i) = \sum_{j\in[n]} w_j \k(x_i, x_j)  f(x_j),
  \end{talign}
  so in matrix form, $T_{\k, \bS_n^w} f = K \diag(w) f$, and hence $T_{\k, \bS^w_n} = K \diag(w)$.
  If $\lambda_\ell$ is an eigenvalue of $K \diag(w)$ with eigenvector $v_\ell$, then
  \begin{talign}
    K\diag(w) v_\ell = \lambda_\ell v_\ell &\Longleftrightarrow \diag(\sqrt{w}) K \diag(w) v_\ell = \lambda_\ell \diag(\sqrt{w})v_\ell \\
                                           &\Longleftrightarrow \diag(\sqrt{w}) K \diag(\sqrt{w}) (\diag(\sqrt{w}) v_\ell) = \lambda_\ell \diag(\sqrt{w})v_\ell,
  \end{talign}
  where we used $w_i > 0$ for all $i\in[n]$.
  Hence the eigenspectrum of $T_{\k,\bS_n^w}$ agrees with that of $\diag(\sqrt{w}) K \diag(\sqrt{w})$.

\para{Proof of bound~(ii) from display~\cref{eq:eigenvalue_bound}}
  By \cref{lem:eigval_bounded_by_ent_num}, we have $\lambda_\ell(T_{\k,\bS_n^w}) \le 4e^2_\ell(\idmap{\rkhs[\k|_{\balleuc(R_n-1)}]}{\cL^2(\bS_n^w)})$.
  Finally, using \cref{def:ent_num}, we have $e_\ell(\idmap{\rkhs[\k|_{\balleuc(R_n-1)}]}{\cL^2(\bS_n^w)}) \le e_\ell(\idmap{\rkhs[\k|_{\balleuc(R_n-1)}]}{\ellinf(\balleuc(R_{n}-1))})$ because $\bS_n^w$ is supported in $\balleuc(R_{n}-1)$ and the fact that $\snorm{\cdot}_{\cL^2(\bP)} \le \snorm{\cdot}_\infty$ for any $\bP$.
\end{proof}

Combining the tools developed so far, we have the following corollary for bounding the eigenvalues of \polygrowth and \loggrowth kernel matrices.
\begin{corollary}[Eigenvalue bound for \polygrowth or \loggrowth kernel matrix]
  \label{cor:Kp_eigval_bound}
  Suppose a kernel $\k$ satisfies \cref{assum:kernel_growth}.
  Let $\cS_n=(x_1,\ldots,x_n)\subset\bR^d$ be a sequence of points.
  For any $w\in\simplex$, using the notation $\krcd_\k$ from \eqref{eqn:krcd}, for any $\ell > \krcd_\k(R_{n}) + 1$, we have
  \begin{talign}
       \lambda_\ell(K^w) \le 
           \begin{cases}
            4\left(\frac{\krcd_\k(R_{n})}{\ell-1}\right)^{\frac{2}{\cvrw}}
            &\polygrowth(\alpha,\beta) \qtext{and} \\ 
            4\exp(2-2\left(\frac{\ell-1}{\krcd_\k(R_{n})}\right)^{\frac{1}{\cvrw}})
            &\loggrowth(\alpha,\beta). 
       \end{cases}
  \end{talign}
\end{corollary}

\begin{proof}[Proof of \cref{cor:Kp_eigval_bound}]
  The claim follows by applying \cref{prop:eigval_K_bounded_by_ent_num} and \cref{prop:ent_num_bound_poly_or_log_growth}.
\end{proof}

\subsection{Covering numbers and eigenvalues of Stein kernels}
\label{subsec:stein_spectral}
The goal of this section is to show that a Stein kernel $\ks$ satisfies \cref{assum:kernel_growth} provided that the base kernel is sufficiently smooth and to derive the parameters $\cvrw$, $\cvrd$ for \polygrowth and \loggrowth cases. Such results when put together with results from \cref{subsec:T_k_spectral} imply spectral decay rates for Stein kernel matrices.

For a Stein kernel $\ksm$ with preconditioning matrix $M$, we define
\begin{talign}
    \scorebound(r)\defeq \max\left(1, \sup_{\twonorm{x} \le r}\twonorm{M^{1/2} \nabla \log p(x)}\right).\label{eqn:scorebound}
\end{talign}
We start by noting a useful alternative expression for a Stein kernel where we only need access to the density via the score $\nabla\log p$.
\begin{proposition}[Alternative expression for Stein kernel]
\label{prop:ks_alt_form}
The Stein kernel $\ksm$ has the following alternative form:
\begin{talign}
  \begin{split}
    \ksm(x,y) = \langle \nabla \log p(x), M\nabla\log p(y)\rangle \k(x,y) + 
    \langle\nabla\log p(x), M\nabla_y \k(x,y)\rangle +\\
    \langle\nabla\log p(y), M\nabla_x \k(x,y)\rangle + \tr(M \nabla_x \nabla_y\k(x,y)),\label{eqn:ks_alt_form}
  \end{split}
\end{talign}
where $\nabla_x \nabla_y\k(x,y)$ denotes the $d\times d$ matrix $(\partial_{x_i}\partial_{y_j}\k(x,y))_{i,j\in[d]}$.
\end{proposition}
\begin{proof}[Proof of \cref{prop:ks_alt_form}]
  We compute
  \begin{talign}
    (\nabla_x \cdot(p(x)M\k(x,y)p(y)))_j &= \sum_{i\in[d]} M_{ij}\left( \partial_{x_i}p(x) \k(x,y)p(y) + p(x)\partial_{x_i}\k(x,y)p(y) \right). \\
    \begin{split}
    \nabla_y \cdot \nabla_x \cdot(p(x)M\k(x,y)p(y)) &= \sum_{i,j\in[d]} M_{ij}\left( \partial_{x_i}p(x)\partial_{y_j}p(y)\k(x,y) + \partial_{x_i}p(x)\partial_{y_j}\k(x,y)p(y) \right) \\
                                                    &+\sum_{i,j\in[d]} M_{ij}\left(p(x)\partial_{y_j}p(y)\partial_{x_i}\k(x,y)  + p(x)\partial_{x_i}\partial_{y_j}\k(x,y)p(y) \right).
  \end{split}
  \end{talign}
  The four terms in the last equation correspond to the four terms in \eqref{eqn:ks_alt_form}.
\end{proof}

In what follows, we will make use of a matrix-valued kernel $\K:\bR^d\times\bR^d\to\bR^{d\times d}$ which generates an RKHS $\rkhs[\K]$ of vector-valued functions.
See \citet{carmeli2006vector} for an introduction to  vector-valued RKHS theory.

Our next goal is to build a Hilbert-space isometry between the direct sum Hilbert space $\rkhs[\k]^{\oplus d}$ and $\rkhs[\ksm]$ to represent functions in $\rkhs[\ksm]$ using functions from $\rkhs[\k]$.
\begin{lemma}[Preconditioned matrix-valued RKHS from a scalar kernel]\label{lem:rkhs_k_K_isometry}
  Let $\k:\bR^d \times \bR^d \to \bR$ be kernel and $\rkhs[\k]$ be the corresponding RKHS.
  Let $M\in\bR^{d\times d}$ be an SPSD matrix.
  Consider the map $\iota: \rkhs[\k]^{\oplus d} \to \funcspace(\bR^d, \bR^d)$ defined by $(f_1,\ldots,f_d) \mapsto [x \mapsto M^{1/2} (f_1(x),\ldots,f_d(x))]$, where $\rkhs[\k]^{\oplus d}$ is the direct-sum Hilbert space of $d$ copies of $\rkhs[\k]$
  Then $\iota$ is a Hilbert-space isometry onto a vector-valued RKHS $\rkhs[\K]$ with matrix-valued reproducing kernel  given by $\K(x,y) = \k(x,y) M$.
\end{lemma}
\begin{proof}[Proof of \cref{lem:rkhs_k_K_isometry}]
  Define $\gamma: \bR^d \to \funcspace(\bR^d, \rkhs[\k]^{\oplus d})$ via 
  \begin{talign}
    \gamma(x)(\alpha) \defeq \k(x,\cdot) M^{1/2} \alpha.
  \end{talign}
  We have
  \begin{talign}
    \norm{\gamma(x)(\alpha)}_{\rkhs[\k]^{\oplus d}} &\le \norm{\k(x,\cdot)}_{\k} \twonorm{M^{1/2}}\twonorm{\alpha},
  \end{talign}
  so $\gamma(x)$ is bounded.
  Since $\gamma(x)$ is also linear, we have $\gamma(x) \in \boundedop(\bR^d, \rkhs[\k]^{\oplus d})$. 
  Let $\gamma(x)^*: \rkhs[\k]^{\oplus d} \to \bR^d$ denote the Hilbert-space adjoint of $\gamma(x)$.
  Then for any $(f_1,\ldots,f_d) \in \rkhs[\k]^{\oplus d}$, $\alpha \in \bR^d$, we have
  \begin{talign}
    \langle \gamma(x)^*(f_1,\ldots,f_d), \alpha \rangle &= \langle (f_1,\ldots, f_d), \gamma(x)(\alpha)\rangle_{\rkhs[\k]^{\oplus d}} \\
                                                        &= \langle (f_1,\ldots, f_d), \k(x,\cdot)M^{1/2} \alpha\rangle_{\rkhs[\k]^{\oplus d}}  \\
                                                        &= \langle (f_1(x),\ldots, f_d(x)), M^{1/2} \alpha\rangle  \\
                                                        &=  \langle M^{1/2}(f_1(x),\ldots, f_d(x)), \alpha \rangle.
  \end{talign}
  Hence $\gamma(x)^*(f_1,\ldots,f_d) = M^{1/2}(f_1(x),\ldots,f_d(x))$, so $\iota(f_1,\ldots,f_d)(x) = \gamma(x)^*(f_1,\ldots,f_d)$.
  By \citet[Proposition 2.4]{carmeli2006vector}, we see that $\iota$ is a partial isometry from $\rkhs[\k]^{\oplus d}$ onto a vector-valued RKHS space withv reproducing kernel  $\K(x,y) = \gamma(x)^*\gamma(y): \bR^d \to \bR^d$.
  For $\alpha \in \bR^d$, previous calculation implies
  \begin{talign}
    \gamma(x)^*\gamma(y) (\alpha) &= \gamma(x)^*(\k(y,\cdot)M^{1/2}\alpha) = M^{1/2}\k(y,x)M^{1/2}\alpha = \k(x,y)M.
  \end{talign}
\end{proof}

\begin{lemma}[Stein operator is an isometry]
  Consider a Stein kernel $\ksm$ with base kernel $\k$ and preconditioning matrix $M$.
  Then, the Stein operator $\steinop$ defined by $\steinop(v) \defeq \frac{1}{p}\div(pv)$
  is an isometry from $\rkhs[\K]$ with $\K \defeq \k M$ to $\rkhs[\ksm]$.
\end{lemma}
\begin{proof}
  This follows from \citet[Theorem 2.6]{barp2022targeted} applied to $\K$.
\end{proof}
The previous two lemmas show that $\steinop \circ \iota$ is a Hilbert space isometry from $\rkhs[\k]^{\oplus d}$ to $\rkhs[\ksm]$.
Note that $\steinop(v) = \langle\nabla\log p, h\rangle + \div h$. Hence, we immediately have
\begin{talign}
  \rkhs[\ksm] = \left\{ \langle\nabla\log p, M^{1/2}f\rangle + \div (M^{1/2}f): f=(f_1,\ldots,f_d) \in \rkhs[\k]^{\oplus d}\right\}.\label{eqn:Hksm_rep}
\end{talign}

We next build a divergence RKHS which is one of the summands used to form $\rkhs[\ksm]$.
\begin{lemma}[Divergence RKHS]\label{lem:divergence_rkhs}
  Let $\k: \bR^d\times\bR^d \to \bR$ be a continuously differentiable kernel.
  Let $M$ be an SPSD matrix.
  Define $\divmk:\bR^d \times \bR^d \to \bR$ via
  \begin{talign}
    \divmk(x,y) \defeq \nabla_y\cdot \nabla_x\cdot(M\k(x,y)) = \tr(M \nabla_x\nabla_y \k(x,y)). \label{eqn:divmk}
  \end{talign}
  Then $\divmk$ is a kernel, and its  RKHS $\rkhs[\divmk]$ has the following explicit form
  \begin{talign}
    \rkhs[\divmk] = \div \rkhs[\K] = \left\{\div (M^{1/2}f): f=(f_1,\ldots,f_d) \in \rkhs[\k]^{\oplus d}\right\},\label{eqn:rkhs_div_k_form}
  \end{talign}
  where $\K = M\k$.
  Moreover, $\div: \rkhs[\K] \to \rkhs[\divmk]$ is an isometry.
\end{lemma}

\begin{proof}[Proof of \cref{lem:divergence_rkhs}]
  First of all, by \citet[Corollary 4.36]{steinwart2008support}, every $f \in \rkhs[\k]$ is continuously differentiable, so $\partial_{x_i} f$ exists.
  By \cref{lem:rkhs_k_K_isometry}, $\div \rkhs[\K]$ is well-defined and the right equality in \eqref{eqn:rkhs_div_k_form} holds.

  Define $\gamma: \bR^d \to \funcspace(\bR, \rkhs[\K])$ via
  \begin{talign}
      \gamma(x)(t) &\defeq t\sum_{i=1}^d \partial_{x_i}\K(x, \cdot) e_i,
  \end{talign} 
  where $e_i \in \bR^d$ is the $i$th standard basis vector in $\bR^d$; by \citet[Lemma C.8]{barp2022targeted} we have $\partial_{x_i}\K(x,\cdot)e_i \in \rkhs[\K]$.
  Note that
  \begin{talign}
      \norm{\gamma(x)(t)}_{\K} &= \abs{t}\norm{\sum_{i=1}^d \partial_{x_i}\K(x, \cdot) e_i}_{\K},
  \end{talign}
  so $\gamma(x) \in \boundedop(\bR, \rkhs[\K])$.
  The adjoint $\gamma(x)^* \in \boundedop(\rkhs[\K], \bR)$ must satisfy, for any $h \in \rkhs[\K]$,
  \begin{talign}
      t\gamma(x)^* h &= \langle h, \gamma(x)(t) \rangle_{\K} 
                                   = \left\langle h, t\sum_{i=1}^d \partial_{x_i}\K(x, \cdot) e_i \right\rangle_{\K} 
                                   = t\div h,
  \end{talign}
  where we used the fact \citep[Lemma C.8]{barp2022targeted} that, for $c \in \bR^d$, $h \in \rkhs[\K]$, $\langle \partial_{x_i}\K(x,\cdot)c, h\rangle =  c^\top \partial_{x_i}h(x)$.
  So we find $\gamma(x)^*(h) = \div h(x)$.
  By \citet[Proposition 2.4]{carmeli2006vector}, the map $A: \rkhs[\K] \to \funcspace(\bR^d, \bR)$ defined by $A(h)(x) = \gamma^*(x)(h) = \div h(x)$, i.e., $A = \div$, is a partial isometry from $\rkhs[\K]$ to an RKHS $\rkhs[\div\K]$ with
  reproducing kernel 
  \begin{talign}
      \gamma(x)^*\gamma(y) = \div \left( \sum_{i=1}^d \partial_{x_i}\K(x,\cdot) e_i\right)(y) =\nabla_y\cdot\nabla_x\cdot \K(x,y) = \divmk(x,y).
  \end{talign}
\end{proof}

The following lemma shows that we can project a covering of $\ball{\k}$ consisting of arbitrary functions to a covering using functions only in $\ball{\k}$ while inflating the covering radius by at most 2.
\begin{lemma}[Projection of coverings into RKHS balls]\label{lem:proj_covering}
  Let $\k$ be a kernel, $A\subset\bR^d$ be a set, and $\eps > 0$. 
  Let $\cC$ be a set of functions such that for any $f \in \ball{\k}$, there exists $g \in \cC$ such that $\norm{f - g}_{\infty,A} \le \eps$.
  Then
  \begin{talign}
    \cvrnum{\k}(A,2\eps) \le \abs{\cC}.
  \end{talign}
\end{lemma}
\begin{proof}
  We will build a $(\k,A,2\eps)$ covering $\cC'$ as follows.
  For any $h \in \cC$, if there exists $h' \in \ball{\k}$ with $\norm{h'-h}_{\infty,A}\le \eps$, then we include $h'$ in $\cC'$.
  By construction, $\abs{\cC'} \le \abs{\cC}$.
  Then, for any $f \in \ball{\k}$, by assumption, there exists $g \in \cC$ such that $\norm{f-g}_{\infty,A} \le \eps$.
  By construction, there exists $g' \in \cC'$ such that $\norm{g'-g}_{\infty,A} \le \eps$.
  Thus
  \begin{talign}
    \norm{f-g'}_{\infty, A} \le \norm{f-g}_{\infty,A} + \norm{g-g'}_{\infty,A} \le 2\eps.
  \end{talign}
  Hence $C'$ is a $(\k,A,2\eps)$ covering.
\end{proof}

We are now ready to bound the covering numbers of $\ks$ by those of $\k$ and $\divmk$.
Our key insight towards this end is that any element in $\rkhs[\ks]$ can be decomposed as a sum of functions originated from $\rkhs[\k]$ and a function from the divergence RKHS $\rkhs[\divmk]$.
\begin{lemma}[Upper bounding covering number of Stein kernel with that of its base kernel]\label{lem:covering_num_ks}
  Let $\ksm$ be a Stein kernel with density $p$ and preconditioning matrix $M$.
  For any $A \subset \bR^d$, $\eps_1,\eps_2 > 0$, 
  \begin{talign}
    \cvrnum{\ksm}(A, \eps) \le \cvrnum{\k}(A, \eps_1)^d \cvrnum{\divmk}(A, \eps_2),
  \end{talign}
  for $\eps = 2(\sqrt{d}\eps_1\sup_{x\in A}\norm{M^{1/2} \nabla\log p(x)} + \eps_2)$.
\end{lemma}
\begin{proof}[Proof of \cref{lem:covering_num_ks}]
    Let $\cC_\k$ be a $(\k, A, \eps_1)$ covering and $\cC_{\divmk}$ be a $(\divmk, A, \eps_2)$ covering.
  Define $b \defeq M^{1/2} \nabla \log p$.
  Form 
  \begin{talign}
    \cC \defeq \left\{
    \langle b, \tilde f \rangle +  \tilde g : \tilde f = (\tilde f_1,\ldots,\tilde f_d) \in (\cC_\k)^d, \tilde g \in \cC_{\divmk}\right\} \subset \funcspace(\bR^d, \bR).
  \end{talign}
  Then $\abs{\cC} \le \abs{\cC_{\k}}^d \abs{\cC_{\divmk}}$.
  Let $\K \defeq \k M$.
  For any $h \in \ball{\ksm}$, by \eqref{eqn:Hksm_rep}, we can find $f = (f_1,\ldots,f_d)\in\rkhs[\k]^{\oplus d}$ with $f_i \in \rkhs[\k]$ such that 
  \begin{talign}
  h = \steinop \circ \iota(f) = \langle\nabla\log p, M^{1/2}f\rangle + \div(M^{1/2}f) = \langle b, f \rangle + \div(M^{1/2}f).
  \end{talign}
  Since $\iota$ and $\steinop$ are isometries, we have $f \in \ball{\rkhs[\k]^{\oplus d}}$.
  Since, for each $i$,
  \begin{talign}
      \norm{f_i}_{\k} \le \sqrt{\sum_{j=1}^d \norm{f_j}_{\k}^2} = \norm{f}_{\rkhs[\k]^{\oplus d}} \le 1,
  \end{talign}
  we have $f_i \in \ball{\k}$.
  By \cref{lem:divergence_rkhs}, $\div: \rkhs[\K] \to \rkhs[\divmk]$ is also an isometry, so $\div (M^{1/2}f) \in \ball{\divmk}$.
  Thus there exist $\tilde{f}_i \in \cC_\k$ for each $i$ and $\tilde{g} \in \cC_{\divmk}$ such that
  \begin{talign}
    \norm{f_i - \tilde f_i}_{\infty,A} \le \eps_1,\quad \norm{\div (M^{1/2}f) - \tilde g}_{\infty,A} \le \eps_2.
  \end{talign}
  Let
  \begin{talign}
      \tilde h(x) \defeq \langle b, \tilde f\rangle + \tilde g \in \cC.
  \end{talign}
  Then for $x \in A$,
  \begin{talign}
      \abs{h(x) - \tilde h(x)} &= \abs{\langle b(x), f(x)-\tilde f(x)\rangle + \div(M^{1/2}f(x)) - \tilde g(x)} \\
                               &\le \norm{b(x)}\sqrt{\sum_{i=1}^d (f_i(x) - \tilde f_i(x))^2} + \abs{\div (M^{1/2}f(x)) - \tilde g(x)} \\
                               &\le \sqrt{d}\eps_1 \norm{b(x)} + \eps_2.
  \end{talign}
  Hence
  \begin{talign}
    \norm{h - \tilde h}_{\infty,A} &\le \sqrt{d}\eps_1 \sup_{x \in A} \norm{b(x)} + \eps_2 \defeq \eps_3.
  \end{talign}
  Note that $\cC$ that we constructed is not necessarily contained in $\ball{\ksm}$.
  By \cref{lem:proj_covering}, we can get a $(\ksm, A, 2\eps_3)$ covering and we are done.
\end{proof}

\begin{corollary}[Log-covering number bound for Stein kernel]\label{cor:covering_num_ks}
  Let $\ksm$ be a Stein kernel and $A \subset \bR^d$.
  For any $r > 0$, $\eps > 0$, 
  \begin{talign}
    \log \cvrnum{\ksm}(A, \eps) \le d\log \cvrnum{\k}\left(A, \frac{\eps}{4\sqrt{d}\scorebound(r)}\right) +\log \cvrnum{\divmk}\left(A, \frac{\eps}{4}\right),
  \end{talign}
  where $\scorebound$ is defined in \eqref{eqn:scorebound}.
\end{corollary}
\begin{proof}
  This is direct from \cref{lem:covering_num_ks} with $\eps_1=\frac{\eps}{4\sqrt{d}\scorebound(r)}$, $\eps_2=\frac{\eps}{4}$.
\end{proof}

\subsubsection{Case of differentiable base kernel}
\begin{definition}[$s$-times continuously differentiable kernel]\label{def:diff_k}
  A kernel $\k$ is \emph{$s$-times continuously differentiable} for $s \in \bN$ if all partial derivatives $\partial^{\alpha,\alpha}\k$ exist and are continuous for all multi-indices $\alpha \in \bN_0^d$ with $\abs{\alpha} \le s$.
\end{definition}
\begin{proposition}[Covering number bound for $\ksm$ with differentiable base kernel]\label{prop:covering_num_ks_diff}
  Suppose $\ksm$ is a Stein kernel with an $s$-times continuously differentiable base kernel $\k$ for $s \ge 2$. 
  Then there exist a constant $\cvrC > 0$ depending only on $(s,d,\k,M)$ such that for any $r > 0,\eps \in (0,1)$,
  \begin{talign}
    \log \cvrnum{\ksm}(\balleuc(r), \eps) \le \cvrC r^d \scorebound^{d/s}(r) (1/\eps)^{d/(s-1)}.
  \end{talign}
\end{proposition}
\begin{proof}[Proof of \cref{prop:covering_num_ks_diff}]
  Since $\k$ is $s$-times continuously differentiable, the divergence kernel $\divmk$ is $(s-1)$-times continuously differentiable.
  By \citet[Proposition 2(b)]{dwivedi2022generalized}, there exists constants $c_1, c_2$ depending only on $(s, d, \k, M)$ such that, for any $r > 0$, $\eps_1,\eps_2 > 0$,
  \begin{talign}
    \log \cvrnum{\k}\left(\balleuc(r), \eps_1\right) &\le c_1 r^d (1/\eps)^{d/s}, \\
    \log \cvrnum{\divmk}\left(\balleuc(r), \eps_2\right) &\le c_2 r^d (1/\eps)^{d/(s-1)}.
  \end{talign}
  By \cref{cor:covering_num_ks} with $A = \balleuc(r)$, we have, for any $r > 0$ and $\eps \in (0,1)$,
  \begin{talign}
    \log \cvrnum{\ksm}(\balleuc(r), \eps) &\le c_1 d r^d (4\sqrt{d}\scorebound(r))^{d/s} (1/\eps)^{d/\eps}  + c_2 r^d (4/\eps)^{d/(s-1)} \\
                                         &\le \cvrC r^d \scorebound^{d/s}(r) (1/\eps)^{d/(s-1)}
  \end{talign}
  for some $\cvrC > 0$ depending only on $(s, d, \k, M)$.
\end{proof}

\subsubsection{Case of radially analytic base kernel}
For a symmetric positive definite $M\in\bR^{d\times d}$, we define, for $x\in\bR^d$,
\begin{talign}
\norm{x}_M \defeq \sqrt{x^\top M^{-1} x}.
\end{talign}
\begin{definition}[Radially analytic kernel]\label{def:analytic_k}
  A kernel $\k$ is \emph{radially analytic} if $\k$ satisfies $\k(x,y) = \kappa(\norm{x-y}_M^2)$ for a symmetric positive definite matrix $M\in\bR^{d\times d}$ and a function $\kappa: \bR_{\ge 0} \to \bR$ real-analytic everywhere with convergence radius $R_\kappa > 0$ such that there exists a constant $C_\kappa > 0$ for which
  \begin{talign}
    \abs{\frac{1}{j!}\kappa_+^{(j)}(0)} \le C_\kappa (2/R_\kappa)^j, \text{ for all } j \in \bN_0,
    \label{eqn:analytic_k}
  \end{talign}
  where $\kappa_+^{(j)}$ indicates the $j$-th right-sided derivative of $\kappa$.
\end{definition}
\begin{example}[Gaussian kernel]\label{exam:cvrnum_gauss}
  Consider the Gaussian kernel $\k(x,y) = \kappa(\norm{x-y}_M^2)$ with $\kappa(t) = e^{-\frac{t}{2\sigma^2}}$ where $\sigma > 0$.
  Note the exponential function is real-analytic everywhere, and so is $\kappa$.
  Since $\kappa(t) = \sum_{j=0}^\infty \frac{(-t/2\sigma^2)^j}{j}$, we find $\frac{1}{j!}\kappa^{(j)}(0) = \frac{(-1)^j}{j(2\sigma^2)^{j}}$. Hence \eqref{eqn:analytic_k} holds with $C_\kappa=1$ and $R_\kappa = 2\inf_{j \ge 0} (j(2\sigma^2)^{j})^{1/j} = 4\sigma^2$.
  
 \end{example}
\begin{example}[IMQ kernel]\label{exam:cvrnum_imq}
  Consider the inverse multiquadric kernel $\k(x,y) = \kappa(\norm{x-y}_M^2)$ with $\kappa(t) = (c^2+t)^{-\beta}$ where $c,\beta > 0$.
  By \citet[Example 3]{sun2008reproducing}, $\kappa$ is real-analytic everywhere with $C_\kappa = c^{-2\beta}(2\beta+1)^{\beta+1}$ and $R_\kappa=c^2$.
 \end{example}

A simple calculation yields the following lemma.
\begin{proposition}[Expression for $\ksm$ with a radially analytic base kernel]
  \label{lem:ks_expr}
  Suppose a Stein kernel $\ksm$ has a symmetric positive definite preconditioning matrix and a base kernel $\k(x,y) = \kappa(\norm{x-y}_M^2)$ where $\kappa$ is twice-differentiable.
  Then
  \begin{talign}
  \begin{split}
    \ksm(x,y) = \langle \nabla \log p(x), M\nabla\log p(y)\rangle \kappa(\norm{x-y}_M^2) - \\
    2\kappa'_+(\norm{x-y}_M^2)\langle x-y, \nabla\log p(x) - \nabla\log p(y)\rangle -\\
    4\kappa''_+(\norm{x-y}_M^2)\norm{x-y}_M^2 - 2d\kappa'_+(\norm{x-y}_M^2).\label{eqn:ks_alt_form_analytic}
    \end{split}
  \end{talign}
  In particular,
  \begin{talign}
    \ksm(x,x) &= \twonorm{M^{1/2} \nabla\log p(x)}^2\kappa(0) - 2d\kappa'_+(0).\label{eqn:ks_diag_analytic}
  \end{talign}
\end{proposition}
\begin{proof}[Proof of \cref{lem:ks_expr}]
  From $\k(x,y) = \kappa(\norm{x-y}_M^2) = \kappa((x-y)^\top M^{-1} (x-y))$, we compute, using \eqref{eqn:divmk},
  \begin{talign}
    \nabla_y \k(x,y) &= -2\kappa'_+(\norm{x-y}_M^2)M^{-1}(x-y) \\
    \nabla_x \nabla_y \k(x,y) &= -2\kappa'_+(\norm{x-y}_M^2)M^{-1}  - 4\kappa''_+(\norm{x-y}_M^2)M^{-1}(x-y)((x-y) M^{-1})^\top\\
    \divmk (x, y) &=\tr(M \nabla_x \nabla_y \k(x,y)) = -4\kappa''_+(\norm{x-y}_M^2)\norm{x-y}_M^2 - 2d\kappa'_+(\norm{x-y}_M^2). \label{eqn:div_k_expr}
  \end{talign}
  The form \eqref{eqn:ks_alt_form_analytic} then follows from applying \cref{prop:ks_alt_form}.
\end{proof}

We next show that the divergence kernel $\divmk$ is radially analytic given that $\k$ is.
\begin{lemma}[Convergence radius of divergence kernel]\label{lem:div_k_conv_rad}
  Let $\k$ be a radially analytic kernel with the corresponding real-analytic function $\kappa$, convergence radius $R_\kappa$ with constant $C_\kappa$, and a symmetric positive definite matrix $M$.
  Then
  \begin{talign}
    \divmk(x,y) = \kappa_{\divmk}(\norm{x-y}_M^2),
  \end{talign}
  where $\kappa_{\divmk}:\bR_{\ge 0} \to \bR$ is the real-analytic function defined as
  \begin{talign}
    \kappa_{\divmk}(t) \defeq -4\kappa''_+(t)t - 2d\kappa'_+(t).
  \end{talign}
  Moreover, $\kappa_{\divmk}$ has convergence radius with constant
  \begin{talign}
    R_{\kappa_{\divmk}} = \frac{R_\kappa}{4d+8}, \quad C_{\kappa_{\divmk}} = 4dC_\kappa/R_\kappa.
  \end{talign}
\end{lemma}
\begin{proof}[Proof of \cref{lem:div_k_conv_rad}]
  The first statement regarding the form of $\kappa_{\divmk}$ directly follows from \eqref{eqn:div_k_expr}.
  Next, iterative differentiation yields, for $j \in \bN_0$,
  \begin{talign}
    \kappa_{\divmk}^{(j)}(t) = -(2d+4j)\kappa^{(j+1)}_+(t) - 4\kappa^{(j+2)}_+(t)t.
  \end{talign}
  Thus
  \begin{talign}
    \abs{\frac{1}{j!}\kappa_{\divmk}^{(j)}(0)} &= \frac{2d+4j}{j!}\kappa^{(j+1)}_+(0) \\
                                             &= \frac{(2d+4j)(j+1)}{(j+1)!}\kappa^{(j+1)}_+(0) \\
                                             &\le (2d+4j)(j+1) C_\kappa(2/R_\kappa)^{j+1}.\label{eqn:loc:kappa_div_bound_case}
  \end{talign}
  For $j \ge 1$,
  \begin{talign}
     \abs{\frac{1}{j!}\kappa_{\divmk}^{(j)}(0)} &\le (2C_\kappa/R_\kappa)\left(((2d+4j)(j+1))^{1/j}2/R_\kappa\right)^j \\
                                              &\le (2C_\kappa/R_\kappa)\left((2(2d+4)\cdot 2/R_\kappa\right)^j.
  \end{talign}
  where we used the fact that $((2d+4j)(j+1))^{1/j}$ is decreasing in $j$.
  For $j = 0$, \eqref{eqn:loc:kappa_div_bound_case} is just $2dC_\kappa \cdot 2/R_\kappa$.
  Hence $\kappa_{\divmk}$ is analytic with $C_{\kappa_{\divmk}} = 4dC_\kappa/R_\kappa$ and $R_{\kappa_{\divmk}} = \frac{R_\kappa}{4d+8}$.
\end{proof}

We will use the following lemma repeatedly.
\begin{lemma}[Covering number of radially analytic kernel with $M$-metric]\label{lem:cvrnum_M_metric}
  Let $\k_0$ be a radially analytic kernel with $\k_0(x,y)=\kappa(\twonorm{x-y}^2)$. 
  For any symmetric positive definite $M \in \bR^{d\times d}$, consider the radially analytic kernel $\k(x,y) \defeq \kappa(\norm{x-y}^2_M)$.
  Then for any $A \subset \bR^d$ and $\eps > 0$, we have
  \begin{talign}
    \cvrnum{\k}(M^{-1/2}(A),\eps) = \cvrnum{\k_0}(A,\eps).
  \end{talign}
  In particular, for any $r > 0$,
  \begin{talign}
    \cvrnum{\k}(\balleuc(r),\eps) \le \cvrnum{\k_0}(\balleuc(r\stwonorm{M^{1/2}}),\eps).
  \end{talign}
\end{lemma}
\begin{proof}
  Note that $\k(x,y) = \k_0(M^{-1/2}x, M^{-1/2}y)$.
  By \citet[Theorem 5.7]{paulsen2016introduction}, $\rkhs[\k] = \{f\circ M^{-1/2}: f \in \rkhs[\k_0]\}$, and moreover $\ball{\k} = \{f\circ M^{-1}: f\in\ball{\k_0}\}$.
  Let $\cC_0$ be a $(\k_0, A, \eps)$ covering.
  Form $\cC = \{h\circ M^{-1/2}: h \in \cC_0\} \subset \ball{\k}$.
  For any element $f\circ M^{-1/2} \in \ball{\k}$ where $f \in \ball{\k_0}$, there exists $h \in \cC_0$ such that $\norm{f-h}_{\infty, A} \le \eps$.
  Thus 
  \begin{talign}
    \norm{f\circ M^{-1/2}-h\circ M^{-1/2}}_{\infty, M^{-1/2}(A)} = \norm{f-h}_{\infty, A} \le \eps.
  \end{talign}
  Thus $\cvrnum{\k}(M^{-1/2}(A),\eps) \le \cvrnum{\k_0}(A,\eps)$. By considering $M^{-1}$ in place of $M$, we get our desired equality.

  For the second statement, by letting $A = M^{1/2}\balleuc(r)$, we have 
  \begin{talign}
    \cvrnum{\k}(\balleuc(r),\eps) = \cvrnum{\k_0}(M^{1/2}\balleuc(r),\eps) \le \cvrnum{\k_0}(\balleuc(r\stwonorm{M^{1/2}}),\eps),
  \end{talign}
  where we use the fact that $M^{1/2}\balleuc(r) \subset \balleuc(r\stwonorm{M^{1/2}})$.
\end{proof}

In the next lemma, we rephrase the result from \citet[Theorem 2]{sun2008reproducing} for bounding the covering number of a radially analytic kernel.
\begin{lemma}[Covering number bound for radially analytic kernel]\label{lem:covering_num_rad_analytic}
  Let $\k$ be a radially analytic kernel with $\k(x,y) = \kappa(\twonorm{x-y}^2)$. Then, there exist a polynomial $P(r)$ of degree $2d$ and a constant $C$ depending only on $(\kappa, d)$ such that for any $r > 0$, $\eps \in (0, 1/2)$,
  \begin{talign}
    \log\cvrnum{\k}(\balleuc(r),\eps) \le P(r)(\log(1/\eps)+C)^{d+1}.
  \end{talign}
\end{lemma}
\begin{proof}[Proof of \cref{lem:covering_num_rad_analytic}]
  Let $R_\kappa, C_\kappa$ denote the constants for $\kappa$ as in \eqref{eqn:analytic_k}.
  By and \citet[Theorem 2]{sun2008reproducing} with $R=1$, $D=2r$, and \cref{lem:cvrnum_preserve_restrict}, for $\eps \in (0, 1/2)$, we have
  \begin{talign}
    \log \cvrnum{\k}(\balleuc(r), \eps) \le N_2(\balleuc(r), r^\dagger/2)\left(4\log(1/\eps) + 2 + 4\log(16\sqrt{C_\kappa}+1)\right)^{d+1},
  \end{talign}
  where $r^\dagger = \min(\frac{\sqrt{R_\kappa}}{2d}, \sqrt{R_\kappa + (2r)^2} - 2r)$, and $N_2(\balleuc(r), r^\dagger/2)$ is the covering number of $\balleuc(r)$ as a subset of $\bR^d$, which can be further bounded by \citep[(5.8)]{wainwright2019high}
  \begin{talign}
    N_2(\balleuc(r), r^\dagger/2) &\le \left(1+\frac{4r}{r^\dagger}\right)^d.
  \end{talign}
If $r^\dagger = \sqrt{R_\kappa + (2r)^2} - 2r$, then $\frac{r}{r^\dagger} = \frac{r}{\sqrt{R_\kappa+(2r)^2}-2r} = \frac{r(\sqrt{R_\kappa + (2r)^2}+2r)}{R_\kappa} \le\frac{r(\sqrt{R_\kappa} + 4r)}{R_\kappa}$ which is a quadratic polynomial in $r$. 
Hence for a constant $C > 0$ and a polynomial $P(r)$ of degree $2d$ that depend only on $(\kappa, d)$, we have the claim.
\end{proof}

\begin{proposition}[Covering number bound for $\ksm$ with radially analytic base kernel]\label{prop:covering_num_ks_analytic}
  Suppose $\ksm$ is a Stein kernel with a preconditioning matrix $M$ and a radially analytic base kernel $\k$ based on a real-analytic function $\kappa$. 
  Then there exist a constant $C > 0$ and a polynomial $P(r)$ of degree $2d$ depending only on $(\kappa,d,M)$ such that for any $r > 0,\eps \in (0,1)$,
  \begin{talign}
    \log \cvrnum{\ksm}(\balleuc(r), \eps) \le \left(\log \frac{\scorebound(r)}{\eps} + C\right)^{d+1}P(r). \label{eqn:covering_num_ks_analytic}
  \end{talign}
\end{proposition}
\begin{proof}[Proof of \cref{prop:covering_num_ks_analytic}]
  Recall $\k(x,y)= \kappa(\norm{x-y}^2_M)$. Consider $\k_0(x,y) \defeq \kappa(\twonorm{x-y}^2)$.
  For $\eps_1 \in (0, 1/2)$, by \cref{lem:cvrnum_M_metric}, we have
  \begin{talign}
\log \cvrnum{\k}(\balleuc(r/\stwonorm{M^{1/2}}), \eps_1) \le \log \cvrnum{\k_0}(\balleuc(r), \eps_1/2).
\end{talign}
Thus by \cref{lem:covering_num_rad_analytic}, there exists a polynomial $P_{\k}(r)$ of degree $2d$ and a constant $C_{\k}$ depending only on $(\kappa, d, M)$ such that
\begin{talign}
  \log \cvrnum{\k}(\balleuc(r), \eps_1) \le P_{\k}(r)(\log(1/\eps_1)+C_{\k})^{d+1}
\end{talign}
Similarly, for $\eps_2 \in (0,1/2)$, by \cref{lem:div_k_conv_rad} and \cref{lem:covering_num_rad_analytic}, we have, for a constant $C_{\divmk} > 0$ and a polynomial $P_{\divmk}(r)$ of degree $2d$ that depend only on $(\kappa, d, M)$, 
  \begin{talign}
    \log \cvrnum{\divmk}(\balleuc(r), \eps_2) \le P_{\divmk}(r)(\log(1/\eps_2)+C_{\divmk})^{d+1}.
  \end{talign}
  For a given $\eps \in (0, 1)$, let $\eps_1 = \frac{\eps}{4\sqrt{d}\scorebound(r)}$ and $\eps_2 = \frac{\eps}{4}$. Then since $\scorebound \ge 1$, we have $\eps_1,\eps_2 \in (0,1/2)$.
  By \cref{cor:covering_num_ks} with $A = \balleuc(r)$, we obtain, for a constants $C > 0$ and a polynomial $P(r)$ of degree $2d$ that depend only on $(\kappa, d, M)$,
  \begin{talign}
    \log \cvrnum{\ksm}(\balleuc(r), \eps) \le P(r)(\log(1/\eps) + \log \scorebound(r) + C)^{d+1}.
  \end{talign}
  Hence \eqref{eqn:covering_num_ks_analytic} is shown.
\end{proof}

When $\log \scorebound(r)$ grows polynomially in $r$, we apply \cref{prop:covering_num_ks_analytic} to immediately obtain the following.
\begin{corollary}\label{cor:cvrnum_ks_analytic_score}
  Under the assumption of \cref{prop:covering_num_ks_analytic}, suppose $\scorebound(r) = O(\poly(r))$.
  Then for any $\delta > 0$, there exists $\cvrC > 0$ such that
  \begin{talign}
      \log \cvrnum{\ksm}(\balleuc(r),\eps)\le\cvrC\log(e/\eps)^{d+1}(r+1)^{2d+\delta}.
  \end{talign}
\end{corollary}
\begin{proof}[Proof of \cref{cor:cvrnum_ks_analytic_score}]
  This immediately follows from \cref{prop:covering_num_ks_analytic} by using $\delta > 0$ to absorb the $\log\scorebound(r)=O(r^\delta)$ term.
\end{proof}

\subsubsection{\pcref{prop:ks_growth}}
This follows from \cref{prop:covering_num_ks_diff,cor:cvrnum_ks_analytic_score}, and by noticing that if $\sup_{\twonorm{x}\le r}\twonorm{\nabla\log p(x)}$ is bounded by a degree $d_\ell$ polynomial, then so is \begin{talign}
\scorebound(r) = \sup_{\twonorm{x}\le r}\twonorm{M^{1/2}\nabla\log p(x)} \le \norm{M^{1/2}}_2\sup_{\twonorm{x}\le r}\twonorm{\nabla\log p(x)}.
\end{talign}
\qed

\section{Analysis of Debiasing Benchmarks}
We start with a result on the MMD quality of \iid sample points in \cref{subsec:mmd_iid}, followed by the proofs of our \iid-like and better-than-\iid guarantees for debiasing from \cref{thm:mcmc_convex,thm:iid_convex} in \cref{sec:thm:mcmc_convex_proof,sec:thm:iid_convex_proof} respectively.
\subsection{MMD of unbiased \iid sample points}
\label{subsec:mmd_iid}
We start by showing that sequence of $n$ points sampled \iid  from $\bP$ achieves $\Theta(n^{-1})$ squared $\mmd_{\kp}$ to $\bP$ in expectation.
\begin{proposition}[MMD of unbiased \iid sample points]\label{prop:iid_guarantee}
Let $\kp$ be a kernel satisfying \cref{assum:mean_zero_p} with $\fp \ge 1$.
Let $\cS_n = (x_i)_{i\in[n]}$ be $n$ \iid samples from $\bP$. Then
\begin{talign}
    \bE[\mmd_{\kp}(\bS_n,\bP)^2] = \frac{\bE_{x\sim\bP}[\kp(x,x)]}{n}.
\end{talign}
\end{proposition}
\begin{proof}[Proof of \cref{prop:iid_guarantee}]
    We compute
    \begin{talign}
        \bE[\mmd_{\kp}(\bS_n,\bP)^2] &= 
        \bE[\sum_{i,j\in[n]}\frac{1}{n^2}\kp(x_i,x_j)] =
        \frac{1}{n^2}\sum_{i,j\in[n]} \bE[\kp(x_i,x_j)] = \frac{1}{n}\bE[\kp(x_1,x_1)],
    \end{talign}
    where we used the fact that $\kp$ is mean-zero with respect to $\bP$ and the independence of $x_i, x_j$ for $i \neq j$.
\end{proof}

\subsection{\pcref{thm:mcmc_convex}}\label{sec:thm:mcmc_convex_proof}
We make use of the self-normalized importance sampling weights $w_j^\snis = \impw(x_j)/\sum_{i\in[n]} \impw(x_i)$ for $j\in[n]$ in our proofs. Notice that $(w_1^{\snis}, \ldots, w_{n}^{\snis})\tp \in \simplex$ and hence
\begin{talign}
    \mmdopt 
    \leq 
\mmd_{\kp}(w_i^{\snis}\delta_{x_i}, \P)
= \frac{\snorm{\sum_{i=1}^n \impw(x_i)\kp(x_i, \cdot)}_{\kp}}{\sum_{i=1}^n \impw(x_i)} = 
\frac{\snorm{\frac1n\sum_{i=1}^n \impw(x_i)\kp(x_i, \cdot)}_{\kp}}{\frac1n\sum_{i=1}^n \impw(x_i)}.
\end{talign}
Introduce the bounded in probability notation $X_n = O_p(g_n)$ to mean $\Pr(\abs{X_n/g_n}> C_\eps) \le \eps$ for all $n \ge N_\eps$ for any $\eps > 0$. Then we claim that under the conditions assumed in \cref{thm:mcmc_convex},
\begin{talign}
\label{eq:int_claims}
    \snorm{\frac1n\sum_{i=1}^n \impw(x_i)\kp(x_i, \cdot)}_{\kp}=O_p(n^{-\half})
    \qtext{and}
    \frac1n\sum_{i=1}^n \impw(x_i) \to 1 \stext{almost surely,}
\end{talign}
so that by Slutsky's theorem \citep[Ex.~1.4.7]{wellner2013weak}, we have $\mmdopt = O_p(n^{-\half})$  as desired. We prove the claims in \cref{eq:int_claims} in two main steps: (a) first, we construct a weighted RKHS and then (b) establish a central limit theorem (CLT) that allows us to conclude both claims from \cref{eq:int_claims}

\paragraph{Constructing a weighted and separable RKHS}
Define the kernel $\kq(x,y) \defeq \impw(x)\kp(x,y)\impw(y)$ with Hilbert space $\rkhs[\kq] = \impw\rkhs[\kp]$ and the elements $\xi_i \defeq \kq(x_i,\cdot) = \impw(x_i)\kp(x_i,\cdot)\impw(\cdot) \in \rkhs[\kq]$ for each $i\in\naturals$.  
By \citet[Prop.~5.20]{paulsen2016introduction}, any element in $\rkhs[\kq]$ is represented by $\impw f$ for some $f \in \rkhs[\kp]$ and moreover, $f \mapsto \impw f$ preserves inner product between the two RKHSs, i.e., $\angles{f, g}_{\kp} = \langle{\impw f, \impw g}\rangle_{\kq}$ for $f, g \in \rkhs[\kp]$, which in turn implies $\snorm{f}_{\kp} = \snorm{\impw f}_{\kq}$. As a result, we also have that 
\begin{talign}
    \snorm{\frac1n\sum_{i=1}^n \impw(x_i)\kp(x_i, \cdot)}_{\kp} =\snorm{\frac1n\sum_{i=1}^n \impw(x_i)\kp(x_i, \cdot) \impw(\cdot)}_{\kq} 
    = \snorm{\frac1n\sum_{i=1}^n \xi_i}_{\kq}. 
    \label{eq:relating_the_norms}
\end{talign}
Since $\rkhs[\kp]$ is separable, there exists a dense countable subset $(f_n)_{n\in\bN} \subset \rkhs[\kp]$. For any $\impw f\in\rkhs[\kq]$, there exists $\{n_k\}_{k\in\bN}$ such that $\lim_{k\to\infty}\snorm{f_{n_k} - f}_{\kp} = 0$. Since $\snorm{\impw f_{n_k} - \impw f}_{\kq} = \snorm{\impw (f_{n_k} - f)}_{\kq} = \norm{f_{n_k}-f}_{\kp}$ due to inner-product preservation, we thus have $\lim_{k\to\infty}\snorm{\impw f_{n_k} - \impw f}_{\kq} = \lim_{k\to\infty}\snorm{f_{n_k} - f}_{\kp}0$, so $(\impw f_n)_{n\in\bN}$ is dense in $\rkhs[\kq]$, showing that $\rkhs[\kq]$ is separable.

\paragraph{Harris recurrence of the chain $(x_i)_{i\in\N}$}
Let $\mu_1$ denote the distribution of $x_1$.
Since $\cS_\infty = (x_i)_{i=1}^\infty$ is a homogeneous $\phi$-irreducible geometrically ergodic Markov chain with stationary distribution $\bQ$, it is also positive \citep[Ch.~10]{meyn2012markov} by definition and aperiodic by \citet[Lem.~9.3.9]{douc2018markov}. 
Moreover, since $\cS_\infty$ is $\phi$-irreducible, aperiodic, and geometrically ergodic in the sense of \citet[Thm.~1]{gallegosherrada2023equivalences} and $\mu_1$ is absolutely continuous with respect to $\bP$, we will assume, without loss of generality, that $\cS_\infty$ is Harris recurrent \citep[Ch.~9]{meyn2012markov}, since, by \citet[Lem.~9]{qin2023geometric}, $\cS_\infty$ is equal to a geometrically ergodic Harris chain with probability $1$.

\paragraph{CLT for $\frac1{\sqrt n}\sum_{i=1}^n \xi_i$}
We now show that $\frac1{\sqrt n}\sum_{i=1}^n \xi_i$ converges to a Gaussian random element taking values in $\rkhs[\kq]$. We separate the proof in two parts: first when the initial distribution $\mu_1 = \bQ$ and next when $\mu_1 \neq \bQ$.

\para{Case 1: $\mu_1 = \bQ$}  When $\mu_1 = \bQ$, $\cS_\infty$ is a strictly stationary chain.
By \citet[Thm.~3.7 and (1.11)]{bradley2005basic}, since $\cS_\infty$ is geometrically ergodic, its strong mixing coefficients $(\tilde{\alpha}_i)_{i\in\naturals}$ satisfy $\tilde{\alpha}_i \leq C \rho^i$ for some $C > 0$ and $\rho \in [0,1)$ and all $i\in \naturals$. Since each $\xi_i$ is a measurable function of $x_i$, the strong mixing coefficients $(\alpha_i)_{i\in\naturals}$ of $(\xi_i)_{i\in\naturals}$ satisfy 
$\alpha_i \leq \tilde{\alpha}_i \leq C \rho^i$ for each $i \in \naturals$.
Consequently, $\sum_{i\in\naturals} i^{2/\delta} \alpha_i < \infty$ for $\delta = 2\fq - 2 > 0$.
Note that we also have 
\begin{talign}
\E_{z\sim\bQ}[\kpnorm[\kq]{\kq(z,\cdot)}^{2+\delta}] 
    = 
\E_{z\sim\bQ}[\kq(z,z)^{\fq}]
    =
\E_{z\sim\bQ}[\impw(z)^{2\fq}\kp(z,z)^{\fq}]
    =
\E_{x\sim\bP}[\impw(x)^{2\fq-1}\kp(x,x)^{\fq}]
    < \infty,
\end{talign}
  $\E_{x_i\sim\Q}[\xi_i] = \E_{x_i\sim\P}[\kp(x_i, \cdot)] = 0$ and that $\rkhs[\kq]$ is separable. Since $\cS_\infty$ is a strictly stationary chain, we conclude that $(\xi_i)_{i\in\N}$ is a strictly stationary centered sequence of $\rkhs[\kq]$-valued random variables satisfying the conditions needed to invoke \citet[Cor.~1]{merlevede1997sharp}, and hence $\sum_{i=1}^n \xi_i/\sqrt{n}$ converges in distribution to a Gaussian random element taking values in $\rkhs[\kq]$.

\para{Case 2: $\mu_1 \neq \bQ$} Since $\cS_\infty$ is positive Harris and $\sum_{i=1}^n \xi_i/\sqrt{n}$ satisfies a CLT for the initial distribution $\mu_1 = \bQ$, \citet[Prop.~17.1.6]{meyn2012markov} implies that $\sum_{i=1}^n \xi_i/\sqrt{n}$ also satisfies the same CLT for any initial distribution $\mu_1$.

\para{Putting the pieces together for \cref{eq:int_claims}}
Since, for any initial distribution for $x_1$, the sequence $(\sum_{i=1}^n \xi_i/\sqrt{n})_{n\in\naturals}$ converges in distribution and that $\rkhs[\kq]$ is separable and (by virtue of being a Hilbert space) complete, Prokhorov's theorem \citep[Thm.~5.2]{billingsley2013convergence} implies that the sequence is also tight, i.e., $\kpnorm[\kq]{\sum_{i=1}^n \xi_i}/\sqrt{n} = O_p(1)$.
Consequently,
\begin{talign}
\snorm{\frac1n\sum_{i=1}^n \impw(x_i)\kp(x_i, \cdot)}_{\kp}
\seq{\cref{eq:relating_the_norms}}
\snorm{\frac1n\sum_{i=1}^n \xi_i}_{\kq}
= \frac{1}{\sqrt{n}} \cdot \kpnorm[\kq]{\frac{\sum_{i=1}^n \xi_i}{\sqrt n}}
    = O_p(n^{-\half}),
\end{talign}
as desired for the first claim in \cref{eq:int_claims}.
Moreover, the strong law of large numbers for positive Harris chains \citep[Thm.~17.0.1(i)]{meyn2012markov} implies that $\frac{1}{n}\sum_{i\in[n]} \impw(x_i)$ converges almost surely to $\E_{z\sim\bQ}[\impw(z)] = 1$ as desired for the second claim in \cref{eq:int_claims}.
\qed

\subsection{\pcref{thm:iid_convex}}\label{sec:thm:iid_convex_proof}
We start with \cref{thm:iid_convex_spectral}, proved in \cref{sec:thm:iid_convex_spectral}, that bounds $\mmdopt$ in terms of the eigenvalues of the integral operator of the kernel $\kp$.
Our proof makes use of a weight construction from \citet[Theorem 3.2]{liu2017black}, but is a non-trivial generalization of their proof as we no longer assume uniform bounds on the eigenfunctions, and instead leverage truncated variations of Bernstein's inequality (\cref{lem:trunc_bernstein_ineq,lem:u_stats_bernstein}) to establish suitable concentration bounds.

\begin{theorem}[Debiasing via \iid simplex reweighting]\label{thm:iid_convex_spectral}
  Consider a kernel $\kp$ satisfying \cref{assum:mean_zero_p} with $\fp=2$.
  Let $(\lambda_\ell)_{\ell=1}^\infty$ be the decreasing sequence of eigenvalues of $T_{\kp, \bP}$ defined in \eqref{eqn:int_op_T_k}.
  Let $\cS_n$ be a sequence of $n \in 2\bN$ i.i.d. random variables with law $\bQ$ such that $\bP$ is absolutely continuous with respect to $\bQ$ and $\sinfnorm{\frac{\dd\bP}{\dd\bQ}}\leq \Mis$ for some $\Mis>0$.
  Futhermore, assume there exist constants $\delta_n, B_n > 0$ such that $\Pr\left(\kpnormSn > B_n\right) < \delta_n$.
  Then for all $L \in \bN$ such that $\lambda_L > 0$, we have
  \begin{talign}
    \bE[\mmd^2_{\kp}(\bS_n^{\wopt}, \bP)] \le \frac{8\Mis}{n}\left( \frac{2\Mis}{n} \frac{\bE_{x\sim\bP}[\kpsq(x,x)]}{\lambda_L} + \sum_{\ell > L} \lambda_\ell \right) + \eps_n\bE[\kpsq(x_1,x_1)],\label{eqn:sub_o(n)_mmd_bound}
  \end{talign}
  where 
  \begin{talign}
    \eps_n^2 \defeq
    n\exp (\frac{-3n}{16\Mis B_{n}/\lambda_L} ) + 2\exp (\frac{-n}{16\Mis^2}) 
    + 2\exp(-\frac{n}{64M^2(\bE_{x\sim\bP}[\kp(x,x)]+B_n/12)/\lambda_L})
    + \delta_{n}.
    \label{eqn:sub_o(n)_mmd_bound_eps_n}
  \end{talign}

\end{theorem}

  Given \cref{thm:iid_convex_spectral},  \cref{thm:iid_convex} follows, i.e., we have $\bE[\mmd^2_{\kp}(\bS_n^{\wopt}, \bP)] = o(n^{-1})$, as long as we can show (i) $\bE[\kpsq(x_1,x_1)] < \infty$, which in turn holds when $\fq > 3$ as assumed in \cref{thm:iid_convex}, and (ii) find sequences $(B_n)_{n=1}^\infty$, $(\delta_n)_{n=1}^\infty$, and $(L_n)_{n=1}^\infty$ such that $\Pr(\kpnormSn > B_n) < \delta_n$ for all $n$ and the following conditions are met:
  \begin{enumerate}[label=(\alph*)]
  \itemsep0em
    \item\label{itm:cvx_reweight_cond1} $\frac{\bE_{x\sim\bP}[\kpsq(x,x)]}{\lambda_{L_n}} = o(n)$;
    \item\label{itm:cvx_reweight_cond2} $\frac{B_n}{\lambda_{L_n}} = O(n^\beta)$, for some $\beta < 1$; 
    \item\label{itm:cvx_reweight_cond3} $\sum_{\ell > L_n} \lambda_\ell = o(1)$;
    \item\label{itm:cvx_reweight_cond4} $\delta_n = o(n^{-2})$.
  \end{enumerate}

  We now proceed to establish these conditions under the assumptions of \cref{thm:iid_convex}.
  
  \para{Condition \cref{itm:cvx_reweight_cond4}}
  By the de La Vall\'{e}e Poussin Theorem \citep[Thm.~1.3]{chandra2015vallee} applied to the $\bQ$-integrable function $x \mapsto \kp(x,x)^\fq$ (which is a uniformly integrable family with one function), there exists a convex increasing function $G$ such that $\lim_{t\to\infty} \frac{G(t)}{t} = \infty$ and $\bE[G(\kp(x_1,x_1)^\fq)] < \infty$.
  Thus,
\begin{talign}
  \Pr(\kp(x_1,x_1) > n^{3/\fq}) = \Pr(\kp(x_1,x_1)^{\fq} > n^3) &= \Pr(G(\kp(x_1,x_1))^\fq > G(n^3)) \\
                                                                &\le \frac{\bE[G(\kp(x_1,x_1))^\fq]}{G(n^3)} = o(n^{-3}),
\end{talign}
where the last step uses $\lim_{t\to\infty} \frac{G(t)}{t} = \infty$.
Hence by the union bound, 
\begin{talign}
  \Pr(\kpnormSn > n^{3/\fq}) = \Pr(\max_{i\in[n]}\kp(x_i,x_i) > n^{3/\fq}) \le n\Pr(\kp(x_1,x_1) > n^{3/\fq}) = o(n^{-2}).
\end{talign}
Hence if we set $B_n = n^{\tau}$ for $\tau\defeq 3/\fq < 1$, there exists $(\delta_n)_{n=1}^\infty$ such that $\delta_n = o(n^{-2})$.
This fulfills \ref{itm:cvx_reweight_cond4} and that $\Pr(\kpnormSn > B_n) < \delta_n$.

To prove remaining conditions, without loss of generality, we assume that $\lambda_\ell > 0$ for all $\ell\in \bN$, since otherwise we can choose $L_n$ to be, for all $n$, the largest $\ell$ such that $\lambda_\ell > 0$. Then $\sum_{\ell > L_n} \lambda_{L_n} = 0$ and all other conditions are met.

\para{Condition~\cref{itm:cvx_reweight_cond3}}
If $L_n \to \infty$, then \ref{itm:cvx_reweight_cond3} is fulfilled since $\sum_\ell {\lambda_\ell} < \infty$, which follows from \cref{lem:general_mercer}\ref{itm:general_mercer:series_rep} and that
\begin{talign}
  \sum_{\ell}\lambda_\ell = \sum_{\ell=1}^\infty \lambda_i \bE_{x\sim\bP}[\phi_\ell(x)^2]
  = \sum_{\ell=1}^\infty \lambda_i \bE_{x\sim\bP}[\phi_\ell(x)^2] 
  =\bE_{x\sim\bP}[\sum_{\ell=1}^\infty \lambda_i \phi_\ell(x)^2] =\bE_{x\sim\bP}[\kp(x,x)] < \infty.
\end{talign}

\para{Conditions~\cref{itm:cvx_reweight_cond1,itm:cvx_reweight_cond2}} Note that the condition \ref{itm:cvx_reweight_cond1} is subsumed by \ref{itm:cvx_reweight_cond2} since $\bE_{x\sim\bP}[\kpsq(x,x)] < \infty$.
  It remains to choose $(L_n)_{n=1}^\infty$ to satisfy \ref{itm:cvx_reweight_cond2} such that $\lim_{n\to\infty}L_n = \infty$.
  Define $L_n \defeq \max\{\ell\in\bN: \lambda_\ell \ge n^{\frac{\tau - 1}{2}}\}$.
  Then $L_n$ is well-defined for $n \ge (\frac{1}{\lambda_1})^{\frac{2}{1-\tau}}$, since for such $n$ we have $\lambda_1 \ge  n^{\frac{\tau-1}{2}}$.
  Hence for $n \ge (\frac{1}{\lambda_1})^{\frac{2}{1-\tau}}$, we have
  \begin{talign}
    \frac{B_n}{\lambda_{L_n}} \le\frac{n^\tau}{n^{\frac{\tau-1}{2}}} = n^{\frac{\tau+1}{2}},
  \end{talign}
  so \ref{itm:cvx_reweight_cond2} is satisfied with $\beta = \frac{\tau+1}{2} < 1$.
  Since $\tau < 1$, $L_n$ is non-decreasing in $n$ and  $n^{\frac{\tau-1}{2}}$ decreases to $0$.
  Since each $\lambda_\ell > 0$, we therefore have $\lim_{n\to\infty}L_n = \infty$.
  \qed

\subsection{\pcref{thm:iid_convex_spectral}}\label{sec:thm:iid_convex_spectral}
We will slowly build up towards proving \cref{thm:iid_convex_spectral}.
First notice $\bE_{x\sim\bP}[\kpsq(x,x)] < \infty$ implies $\bE_{x\sim\bP}[\kp(x,x)] < \infty$, so \cref{lem:general_mercer} holds.
Fix any $L\in\bN$ satisfying $\lambda_L > 0$.
Since $n$ is even, we can define $\cD_0 \defeq [n/2]$ and $\cD_1 \defeq [n] \setminus \cD_0$.
We will use $\cS_{\cD_0}$ and $\cS_{\cD_1}$ to denote the subsets of $\cS_n$ with indices in $\cD_0$ and $\cD_1$ respectively.
Let $(\phi_\ell)_{\ell=1}^\infty \subset \rkhs[\kp]$ be eigenfunctions corresponding to the eigenvalues $(\lambda_\ell)_{\ell=1}^\infty$ by \cref{lem:general_mercer}\ref{itm:general_mercer:ons}, so that $({\phi_\ell})_{\ell=1}^\infty$ is an orthonormal system of $\cL^2(\bP)$. %

We start with a useful lemma.
\begin{lemma}[{$\rkhs[{\kp}]$} consists of mean-zero functions]\label{lem:zero_mean_rkhs_kp}
    Let $\kp$ be a kernel satisfying \cref{assum:mean_zero_p}.
    Then for any $f \in \rkhs[\kp]$, we have $\bP f = 0$.
\end{lemma}
\begin{proof}
    Fix $f \in \rkhs[\kp]$.
    By \citet[Thm~4.26]{steinwart2008support}, $f$ is $\bP$ integrable.
    Consider the linear operator $I$ that maps $f \mapsto \bP f$.
    Since
    \begin{talign}
    \abs{I(f)} = \abs{\bP f} \le \bP \abs{f} = \int \abs{\langle f, \kp(x,\cdot)\rangle_{\kp}} \dd\bP \le \int \norm{f}_{\kp}  \sqrt{\kp(x,x)}\dd\bP = \norm{f}_{\kp} \bE_{x\sim\bP}[\kp(x,x)^{\frac{1}{2}}].
    \end{talign}
    Hence $I$ is a continuous linear operator, so by the Riez representation theorem \citep[Thm.~A.5.12]{steinwart2008support}, there exists $g \in \rkhs[\kp]$ such that $I(h) = \langle h, g\rangle_{\kp}$ for any $h \in \rkhs[\kp]$.
    
    By \citet[Thm.~4.21]{steinwart2008support}, the set 
    \begin{talign}
    H_{\text{pre}} \defeq \left\{\sum_{i=1}^n \alpha_i \kp(\cdot, x_i): n\in\bN, (\alpha_i)_{i\in[n]} \subset \bR, (x_i)_{i\in[n]}\subset \bR^d\right\}
    \end{talign}
    is dense in $\rkhs[\kp]$.
    Note that $H_{\text{pre}}$ consists of mean zero functions under $\bP$ by linearity.
    So there exists $f_n$ converging to $f$ in $\rkhs[\kp]$ where each $f_n$ has $\bP f_n = I(f_n) = \langle f_n, g \rangle_{\kp} = 0$.
    Since
    \begin{talign}
       \lim_{n\to\infty} \abs{\langle f, g \rangle_{\kp} - \langle f_n, g \rangle_{\kp}} = \lim_{n\to\infty}\abs{\langle f - f_n, g \rangle_{\kp}} \le \lim_{n\to\infty}\norm{f-f_n}_{\kp} \norm{g}_{\kp} = 0,
    \end{talign}
    we have $\bP f= \langle f, g\rangle_{\kp} = 0$.
\end{proof}
In particular, the assumption $\bE_{x\sim\bP}[\kpsq(x,x)] < \infty$ of \cref{thm:iid_convex_spectral} implies $\bE_{x\sim\bP}[\kp(x,x)^{\frac{1}{2}}] < \infty$, so \cref{lem:zero_mean_rkhs_kp} holds. 

\para{Step 1. Build control variate weights}

Fix any $L \geq 1$ and $h \in \rkhs[\kp]$, and let $\hat{h}_{\cD_0}$ denote the eigen-expansion truncated approximation of $h$ based on $\cD_0$, 
\begin{talign}
  \hat{h}_{\cD_0}(x) \defeq \sum_{\ell=1}^L \hat{\beta}_{\ell, 0} \phi_\ell(x)\qtext{for}
  \hat{\beta}_{\ell, 0} \defeq \frac{2}{n}\sum_{i\in\cD_0} h(x_i) \phi_\ell(x_i) \wis(x_i).
\end{talign}
Then
\begin{talign}
  \bE[\hat{\beta}_{\ell, 0}] &= \bE\left[\frac{2}{n}\sum_{i\in\cD_0} h(x_i) \phi_\ell(x_i) \wis(x_i)\right] = \langle {h}, {\phi_\ell} \rangle_{\cL^2(\bP)}.\label{eqn:E_beta_hat}
\end{talign}
Next, define the control variate 
\begin{talign}\label{eqn:Z0_cv}
  \hat{Z}_0[h] = \frac{2}{n} \sum_{i\in D_1} \left( \wis(x_i) (h(x_i) - \hat{h}_{\cD_0}(x_i))\right).
\end{talign}
which satisfies
\begin{talign}\label{eq:mean-zero-cv}
  \bE[\hat{Z_0}[h]] = \bE_{x \sim \bP} \left[ h(x) - \sum_{\ell=1}^L \bE[\hat{\beta}_{\ell,0}]\phi_\ell(x) \right] = 0,
\end{talign}
since functions in $\rkhs[\kp]$ have mean $0$ with respect to $\bP$ (\cref{lem:zero_mean_rkhs_kp}).
Similarly, we define $\hat{Z_1}[h]$ by swapping $\cD_0$ and $\cD_1$.
Then we form $\hat{Z}[h] \defeq \frac{\hat{Z}_0[h] + \hat{Z}_1[h]}{2}$.
We can rewrite $\hat{Z}[h]$ as a quadrature rule over $\cS_n$ \citep[Lemma B.6]{liu2017black}
\begin{talign}
  \hat{Z}[h] = \sum_{i\in[n]} w_i h(x_i),\label{eqn:cv_weight_exp}
\end{talign}
where $w_i$ is defined as (whose randomness depends on the randomness in $\cS_n$)
\begin{talign}
  w_i \defeq \left\{\begin{array}{cc}
      \frac{1}{n}\wis(x_i) - \frac{2}{n^2}\sum_{j\in\cD_1} \wis(x_i) \wis(x_j)\langle \Phi_L(x_i), \Phi_L(x_j) \rangle, \forall i \in \cD_0,\\
      \frac{1}{n}\wis(x_i) - \frac{2}{n^2}\sum_{j\in\cD_0} \wis(x_i) \wis(x_j)\langle \Phi_L(x_i), \Phi_L(x_j) \rangle, \forall i \in \cD_1,
  \end{array}\right. \label{eqn:cv_weights}
\end{talign}
and $\Phi_L(x) \defeq (\phi_1(x),\ldots,\phi_L(x))$. 

\para{Step 2. Show $\bE[\mmd^2_{\kp}(\bS_n^w, \bP)] = o(n^{-1})$}

We first bound the variance of the control variate $\hat{Z}_0[h]$ for $h = \phi_{\ell'}$ for $\ell' \in \bN$.
Let us fix $\ell' \in \bN$.
From \eqref{eqn:Z0_cv}, we compute
\begin{talign}
  \bE[\hat{Z}_0[h]^2] = \frac{4}{n^2} \bE\left[ \left(\sum_{i \in \cD_1}\wis(x_i)(h(x_i) - \hat{h}_{\cD_0}(x_i))\right)^2 \right] 
                    &= \frac{4}{n^2} \bE \left[\sum_{i \in \cD_1}\wis(x_i)^2(h(x_i)-\hat{h}_{\cD_0}(x_i))^2\right] \\
                      &= \frac{2}{n} \bE[\bE_{x\sim \bQ} [\wis(x)^2(h(x)-\hat{h}_{\cD_0}(x))^2 \vert \cS_{\cD_0}]] \\
                      &= \frac{2}{n} \bE[\bE_{x\sim \bP} [\wis(x)(h(x)-\hat{h}_{\cD_0}(x))^2 \vert \cS_{\cD_0}]] \\ 
                      &\le \frac{2\Mis}{n} \bE[\bE_{x\sim \bP} [(h(x)-\hat{h}_{\cD_0}(x))^2 \vert \cS_{\cD_0}]],
\end{talign}
where in the second equality, the cross terms are zero due to the independence of points $x_i$ and the equality \cref{eq:mean-zero-cv}.
By the definition of $\hat{h}_{\cD_0}$, we compute
\begin{talign}
  \bE_{x\sim \bP} [(h(x)-\hat{h}_{\cD_0}(x))^2 \vert \cS_{\cD_0}] &= \bE_{x \sim \bP}\left[ \left( \phi_{\ell'}(x) - \sum_{\ell \le L} \hat{\beta}_{\ell,0}\phi_\ell(x)  \right)^2\middle\vert \cS_{\cD_0} \right] \\
																																	&= \bE_{x \sim \bP}\left[  \phi_{\ell'}^2(x) + \sum_{\ell \le L} \hat{\beta}^2_{\ell,0}\phi^2_\ell(x) - 2\phi_{\ell'}(x)\sum_{\ell \le L} \hat{\beta}_{\ell, 0}\phi_\ell(x)  \middle\vert \cS_{\cD_0} \right] \\
																																	&= 1 +\sum_{\ell \le L} \hat\beta^2_{\ell, 0} - 2\sum_{\ell \le L} \hat\beta_{\ell,0}\ind{\ell'=\ell} \\
																																	&= 1 +\sum_{\ell \le L} \hat\beta^2_{\ell, 0} - 2\hat\beta_{\ell', 0}\ind{\ell' \le L},
\end{talign}
where we use the fact that $({\phi_\ell})_{\ell=1}^\infty$ is an orthonormal system in $\cL^2(\bP)$.
By \eqref{eqn:E_beta_hat} with $h = \phi_{\ell'}$, we have $\bE[\hat{\beta}_{\ell', 0}] = 1$. 
On the other hand, we can bound, again using the orthonormality of $({\phi_\ell})_{\ell=1}^\infty$,
\begin{talign}
  \bE[\hat{\beta}^2_{\ell, 0}] 
  = \bE\left[\left(\frac{2}{n}\sum_{i\in \cD_0} \phi_{\ell}(x_i)\phi_{\ell'}(x_i)\wis(x_i)\right)^2\right]
  = \frac{4}{n^2}\bE\left[\sum_{i \in \cD_0}(\phi_{\ell}(x_i)\phi_{\ell'}(x_i)\wis(x_i))^2\right] 
  \le \frac{2\Mis}{n}\bE_{x\sim \bP} [(\phi_\ell(x)\phi_{\ell'}(x))^2].
\end{talign}
Thus for all $\ell' \in \bN$,
\begin{talign}
  \bE[\hat{Z}_0[\phi_{\ell'}]^2] \le \frac{2\Mis}{n}\left(1 + \frac{2\Mis}{n}\sum_{\ell \le L}\bE_{x\sim\bP}[(\phi_\ell(x)\phi_{\ell'}(x))^2] - 2 \ind{\ell'\le L}\right)
                                 \le\frac{2\Mis}{n}\left(\frac{2\Mis}{n}\sum_{\ell \le L}\bE_{x\sim\bP}[(\phi_\ell(x)\phi_{\ell'}(x))^2] + \ind{\ell'> L}\right).
\end{talign}
Since $\hat{Z}[h] = \frac{\hat{Z}_0[h] + \hat{Z}_1[h]}{2}$ and $(\frac{a+b}{2})^2 \le \frac{a^2 + b^2}{2}$ for $a,b\in\bR$, and, by symmetry, $\bE[\hat{Z}_0[h]^2] = \bE[\hat{Z}_1[h]^2]$, we have
\begin{talign}
  \bE[\hat{Z}[\phi_{\ell'}]^2] &\le \frac{2\Mis}{n}\left(\frac{2\Mis}{n}\sum_{\ell \le L}\bE_{x\sim\bP}[(\phi_\ell(x)\phi_{\ell'}(x))^2] + \ind{\ell'> L}\right). \label{eqn:Z_phi_var_bound}
\end{talign}

Now we have
\begin{talign}
  \bE[\mmd^2_{\kp}(\bS_n^w, \bP)] = \bE\left[\sum_{i,j\in[n]} w_i w_j \kp(x_i,x_j)\right]  
                                &= \bE\left[\sum_{i,j\in[n]} w_i w_j \sum_{\ell'=1}^\infty \lambda_{\ell'}\phi_{\ell'}(x_i)\phi_{\ell'}(x_j)\right]\\
                                &= \bE\left[\sum_{\ell'=1}^\infty\sum_{i,j\in[n]} w_i w_j  \lambda_{\ell'}\phi_{\ell'}(x_i)\phi_{\ell'}(x_j)\right]\\
                                &= \bE\left[\sum_{\ell'=1}^\infty \lambda_{\ell'}\left(\sum_{i\in[n]} w_i \phi_{\ell'}(x_i)\right)^2\right]\\
                                &= \sum_{\ell'=1}^\infty\lambda_{\ell'} \bE\left[\left(\sum_{i\in[n]} w_i \phi_{\ell'}(x_i)\right)^2\right]
                         = \sum_{\ell'=1}^\infty \lambda_{\ell'}\bE[\hat{Z}[\phi_{\ell'}]^2],
\end{talign}
where the second and third equalities are due to the absolute convergence of the Mercer series (\cref{lem:general_mercer}\ref{itm:general_mercer:series_rep}), the fourth equality follows from Tonelli's theorem \citep[Thm.~A.3.10]{steinwart2008support}, and the last step is due to \eqref{eqn:cv_weight_exp}.
Plugging in \eqref{eqn:Z_phi_var_bound}, we have
\begin{talign}
  \bE[\mmd^2_{\kp}(\bS_n^w, \bP)] &\le \frac{2\Mis}{n}\left( \frac{2\Mis}{n} \sum_{\ell'= 1}^\infty\sum_{\ell \le L}\lambda_{\ell'}\bE_{x\sim\bP}[(\phi_{\ell}(x)\phi_{\ell'}(x))^2] + \sum_{\ell > L} \lambda_\ell \right).
\end{talign}
Since the eigenvalues are nonnegative and non-increasing, we can write, by \eqref{eqn:mercer_series_rep}, 
\begin{talign}
  \kpsq(x,x) &= \left(\sum_{\ell=1}^\infty \lambda_\ell \phi_{\ell}(x)^2 \right)^2 \ge \sum_{\ell'=1}^\infty\sum_{\ell\le L} \lambda_{\ell'}\lambda_{\ell}(\phi_\ell(x)\phi_{\ell'}(x))^2 
             \ge \lambda_L \sum_{\ell'=1}^\infty\sum_{\ell\le L} \lambda_{\ell'}(\phi_\ell(x)\phi_{\ell'}(x))^2.
\end{talign}
Thus by Tonelli's theorem \citep[Thm.~A.3.10]{steinwart2008support},
\begin{talign}
  \sum_{\ell'= 1}^\infty\sum_{\ell \le L}\lambda_{\ell'}\bE_{x\sim\bP}[(\phi_{\ell}(x)\phi_{\ell'}(x))^2] &= \bE_{x\sim\bP}\left[\sum_{\ell'= 1}^\infty\sum_{\ell \le L}\lambda_{\ell'}(\phi_{\ell}(x)\phi_{\ell'}(x))^2\right]
                                                                                                          \le \frac{\bE_{x\sim\bP}[\kpsq(x,x)]}{\lambda_L}.
\end{talign}
Finally, we have
\begin{talign}
  \bE[\mmd^2_{\kp}(\bS^w_n, \bP)] &\le \frac{2\Mis}{n}\left( \frac{2\Mis}{n} \frac{\bE_{x\sim\bP}[\kpsq(x,x)]}{\lambda_L}  + \sum_{\ell > L} \lambda_\ell \right).\label{eqn:mmd_kp_random_w_bound}
\end{talign}

\para{Step 3. Meet the non-negative constraint}

We now show that the weights \eqref{eqn:cv_weights} are nonnegative and sum close to one with high probability.
For $i \in \cD_0$, we have
\begin{talign}
  w_i = \frac{1}{n}\wis(x_i)\left(1 - T_i\right)
  \qtext{for}
  T_i \defeq \frac{2}{n} \sum_{j\in\cD_1}\wis(x_j)\langle \Phi_L(x_i), \Phi_L(x_j) \rangle.
  \label{eqn:cv_weights_D0}
\end{talign}
Our first goal is to derive an upper bound for $\Pr\left(\min_{i\in\cD_0} w_i < 0\right)$.
Define the event 
\begin{talign}
  E_n \defeq \left\{\kpnormSn \le B_{n}\right\},\label{eqn:kp_bound_event}
\end{talign}
so $\Pr(E_n^c) < \delta_{n}$ by the assumption on $\kpnormSn$.
Then
\begin{talign}
    \Pr\left(\min_{i\in [n]} w_i < 0, E_n\right) &= \Pr\left(\max_{i\in[n]} T _i > 1, E_n\right) 
    \le n\Pr(T_1\ind{E_n} > 1),\label{eqn:loc:union_delta_bound}
\end{talign}
where we applied the union bound and used the fact that $T_i\ind{E_n}$ has the same law for different $i$.
To further bound $\Pr(T_1\ind{E_n} > 1)$, we will use the following lemma.
\begin{lemma}[Truncated Bernstein inequality]\label{lem:trunc_bernstein_ineq}
  Let $X_1,\ldots,X_n$ be i.i.d. random variables with $\bE[X_1] = 0$ and $\bE[X_1^2] < \infty$.
  For any $B > 0$, $t > 0$,
  \begin{talign}
    \Pr\left( \frac{1}{n}\sum_{i\in[n]} X_i\ind{X_i\le B} > t \right) \le \exp\left(\frac{-nt^2}{2(\bE[X_1^2]+\frac{Bt}{3})}\right).
  \end{talign}
\end{lemma}
\begin{proof}[Proof of \cref{lem:trunc_bernstein_ineq}]
Fix any $B> 0$ and $t>0$ and define, for each $i\in[n]$,
$
  Y_i \defeq X_i \ind{X_i \le B}.
$
Then $Y_i \le B$, 
\begin{talign}
  \bE[Y_i] &= \bE[X_i\ind{X_i \le B}] \le \bE[X_i\ind{X_i \le B}] + \bE[X_i\ind{X_i > B}] = \bE[X_i] = 0, \qtext{and}\\
  \bE[Y_i^2] &= \bE[X_i^2\ind{X_i \le B}] \le \bE[X_i^2] = \bE[X_1^2].
\end{talign}
Now we can invoke the non-positivity of $\bE[Y_i]$, the one-sided Bernstein inequality \citep[Prop.~2.14]{wainwright2019high}, and the relation $\bE[Y_i^2]\leq \bE[X_1^2]$ to conclude that 
\begin{talign}
  \Pr\left(\frac{1}{n}\sum_{i\in[n]} Y_i > t\right) &\le \Pr\left(\frac{1}{n}\sum_{i\in[n]} \left(Y_i - \bE[Y_i] \right) > t\right) 
  \le \exp\left(\frac{-nt^2}{2(\frac{1}{n}\sum_{i\in[n]}\bE[Y_i^2]+\frac{Bt}{3})}\right) 
  \le \exp\left(\frac{-nt^2}{2(\bE[X_1^2]+\frac{Bt}{3})}\right).
\end{talign}
\end{proof}
  
For $j \in \cD_1$, define 
$
  X_j \defeq \wis(x_j)\langle \Phi_L(x_1), \Phi_L(x_j)\rangle
$
and note that 
\begin{talign}
  \bE[X_j \vert x_1] &= \bE_{x\sim\bQ}[\wis(x)\langle \Phi_L(x_1), \Phi_L(x) \rangle\vert x_1] = \bE_{x\sim\bP}[\langle \Phi_L(x_1), \Phi_L(x) \rangle \vert x_1] = 0 \\
  \bE[X_j^2 \vert x_1] &= \bE[\wis(x_j)^2 \langle \Phi_L(x_1), \Phi_L(x_j)\rangle^2 \vert x_1]
             \le \Mis \bE_{x\sim\bP}[\langle \Phi_L(x_1), \Phi_L(x)\rangle^2\vert x_1] \\
             &=\Mis \bE_{x\sim\bP}\left[\sum_{\ell,\ell'\le L}\phi_\ell(x_1) \phi_{\ell'}(x_1) \phi_\ell(x)\phi_{\ell'}(x)\middle\vert x_1\right] \\
             &=\Mis \sum_{\ell,\ell'\le L}\phi_\ell(x_1) \phi_{\ell'}(x_1) \bE_{x\sim\bP}\left[\phi_\ell(x)\phi_{\ell'}(x)\right] \\
             &= \Mis \twonorm{\Phi_L(x_1)}^2.
\end{talign}

Since $\lambda_1\ge \lambda_2\ge \cdots \ge 0$, for any $x \in \bR^d$, we can bound $\twonorm{\Phi_L(x)}^2$ via
\begin{talign}
  \twonorm{\Phi_L(x)}^2 &= \sum_{\ell\le L} \phi_\ell(x)^2 \le \frac{\sum_{\ell\le L} \lambda_\ell\phi_\ell(x)^2}{\lambda_L} 
                     \le \frac{\sum_{\ell=1}^\infty\lambda_\ell\phi_\ell(x)^2}{\lambda_L} = \frac{\kp(x,x)}{\lambda_L},\label{eqn:Phi_L_norm_bound}
\end{talign}
where we applied \cref{lem:general_mercer}\ref{itm:general_mercer:series_rep} for the last equality.
Thus
\begin{talign}
  \abs{X_j} \le \Mis \twonorm{\Phi_L(x_1)}\twonorm{\Phi_L(x_j)} \le \Mis \twonorm{\Phi_L(x_1)} \sqrt{\frac{\kp(x_j,x_j)}{\lambda_L}},
\end{talign}
so if we let 
$
  B \defeq \sqrt{\frac{B_{n}}{\lambda_L}}\Mis\twonorm{\Phi_L(x_1)},
$ 
then 
\begin{talign}
  E_n = \left\{\sup_{i\in[n]}\kp(x_i,x_i) \le B_{n}\right\} \subset \bigcap_{j\in\cD_1}\{\abs{X_j} \le B\}.
\end{talign}
Since $T_1 = \frac{2}{n} \sum_{j\in\cD_1} X_j$,
we have inclusions of events
\begin{talign}
  \{T_1\ind{E_n} > 1\} &= \{T_1 > 1\} \cap E_n 
                      \subset \left\{\frac{2}{n}\sum_{j\in\cD_1} X_j \ind{X_j \le B} > 1\right\}.
\end{talign}
Thus \cref{lem:trunc_bernstein_ineq} with $t=1$ and conditioned on $x_1$ implies
\begin{talign}
  \Pr\left(T_1\ind{E_n} > 1\middle\vert x_1\right) &\le \Pr\left(\frac{2}{n}\sum_{j\in\cD_1} X_j \ind{X_j \le B}> 1 \middle\vert x_1\right)\\
  &\le \exp\left(\frac{-n}{4(\Mis\twonorm{\Phi_L(x_1)}^2+\sqrt{\frac{B_{n}}{\lambda_L}}\Mis\twonorm{\Phi_L(x_1)}/3)}\right).
\end{talign}
On event $\{\kp(x_1,x_1) \le B_{n}\}$, by \eqref{eqn:Phi_L_norm_bound}, we have
\begin{talign}
  \twonorm{\Phi_L(x_1)} \le \sqrt{\frac{B_{n}}{\lambda_L}}.
\end{talign}
Hence 
\begin{talign}
  \Pr\left(T_1\ind{E_n} > 1 \middle\vert x_1\right)\ind{\kp(x_1,x_1) \le B_{n}} \le \exp\left(\frac{-n}{\frac{16}{3}\Mis\frac{B_{n}}{\lambda_L}}\right).
\end{talign}
On the other hand, $\{\kp(x_1,x_1)>B_{n}\} \notin E_n$, so
\begin{talign}
  \Pr\left(T_1\ind{E_n} > 1 \middle\vert x_1\right)\ind{\kp(x_1,x_1) > B_{n}} = 0
\end{talign}
Thus 
\begin{talign}
  \Pr(T_1\ind{E_n} > 1) = \bE[\Pr(T_1\ind{E_n} > 1\vert x_1)] \le \exp\left(\frac{-n}{\frac{16}{3}\Mis\frac{B_{n}}{\lambda_L}}\right).
\end{talign}
Combining the last inequality with \eqref{eqn:loc:union_delta_bound}, we have:
\begin{talign}
  \Pr\left(\min_{i\in [n]} w_i < 0, E_n\right) &\le n\exp\left(\frac{-n}{\frac{16}{3}\Mis\frac{B_{n}}{\lambda_L}}\right).\label{eqn:w_nonneg_bound}
\end{talign}

\para{Step 4. Meet the sum-to-one constraint}

Let
\begin{talign}
  S \defeq \sum_{i\in\cD_0} w_i = \sum_{i\in\cD_0} \frac{1}{n}\wis(x_i)\left(1-\frac{2}{n}\sum_{j\in\cD_1}\wis(x_j)\langle \Phi_L(x_i), \Phi_L(x_j)\rangle\right).
\end{talign}
We now derive a bound for $\Pr(S < 1/2 - t/2)$ for $t \in (0, 1)$.
Let
\begin{talign}
  S_1 \defeq \frac{1}{n}\sum_{i\in\cD_0} \wis(x_i), \quad S_2 \defeq -\frac{2}{n^2} \sum_{i\in\cD_0}\sum_{j\in\cD_1} \wis(x_i)\wis(x_j)\langle\Phi_L(x_i), \Phi_L(x_j)\rangle,
\end{talign}
so $S = S_1 + S_2$.
Note that $\bE[S_1] = 1/2$ and $\bE[S_2] = 0$ since $\cD_0$ and $\cD_1$ are disjoint.
Let $E_n$ be the same event defined as in \eqref{eqn:kp_bound_event}.
For $t_1 \in (0, t/2)$ to be determined later and $t_2 \defeq t/2 - t_1$, we have, by the union bound
\begin{talign}
  \Pr(S < 1/2 - t/2, E_n) \le \Pr(S_1 < 1/2 - t_1, E_n) + \Pr(S_2 < -t_2, E_n).
\end{talign}
By Hoeffding's inequality and the assumption $\wis(x) \le \Mis$, we have
\begin{talign}
  \Pr(S_1 < 1/2 - t_1, E_n) \le \Pr\left(\frac{2}{n}\sum_{i\in\cD_0}\frac{\wis(x_i)}{2} - 1/2 < -t_1\right) &\le \exp\left(\frac{-2(n/2)t_1^2}{(\Mis/2)^2}\right) 
                                                                                                   = \exp\left(\frac{-4nt_1^2}{\Mis^2}\right).\label{eqn:loc:hoeffding_S1_bound}
\end{talign}
To give a concentration bound for $\Pr(S_2 < -t_2, E_n)$, we will use the following lemma.
\begin{lemma}[U-statistic Bernstein's inequality]\label{lem:u_stats_bernstein}
  Let $h: \cX \times \cX \to \bR$ be a function bounded above by $b > 0$. Assume $n\in 2\bN$ and let $x_1,\ldots,x_n$ be i.i.d. random variables taking values in $\cX$.
  Denote $m_h \defeq \bE[h(x_1,x_2)]$ and $\sigma_h^2 \defeq \Var[h(x_1,x_2)]$.
  Let $\cD_0 = [n/2]$ and $\cD_1 = [n]\setminus [n/2]$. Define
  \begin{talign}
    U \defeq \frac{1}{(n/2)^2}\sum_{i\in\cD_0}\sum_{j\in\cD_1} h(x_i, x_j).
  \end{talign}
  Then
  \begin{talign}
    \Pr(U - m_h > t)  \le \exp\left(\frac{-nt^2}{4(\sigma_h^2 + \frac{bt}{3})}\right).
  \end{talign}
\end{lemma}
\begin{proof}[Proof of \cref{lem:u_stats_bernstein}]
  We adapt the proof from \citet[Section 3]{pitcan2017note} as follows.
  Let $k \defeq n/2$.
  Define $V: \cX^n \to \bR$ as
  \begin{talign}
    V(x_1,\ldots,x_n) \defeq \frac{1}{k} \sum_{i\in [k]}h(x_i, x_{i+k}).
  \end{talign}
  Then note that
  \begin{talign}
    U &= \frac{1}{k!}\sum_{\sigma \in \perm(k)} V_\sigma, \\
     V_\sigma &\defeq V(x_{\sigma_1}, \ldots, x_{\sigma_{k}}),
  \end{talign}
  where $\perm(k)$ is the set of all permutations of $[k]$; this is because every $h(x_i, x_j)$ term for $i\in\cD_0,j\in\cD_1$ will appear in the summation an equal number of times.
  For a fixed $\sigma \in \perm(k)$, the random variable $V(x_{\sigma_1}, \ldots, x_{\sigma_{k}}, x_{k+1},\ldots,x_{n})$ is a sum of $k$ \iid terms $h(x_{\sigma_i}, x_{i+k})$.
  Denote $V = V(x_1,\ldots,x_n)$.
  For any $s > 0$, we have, by independence,
  \begin{talign}
      \bE[e^{s(V-m_h)}] &= \bE\left[\exp(\frac{s}{k}\sum_{i\in[k]} (h(x_i,x_{i+[k]})-m_h))\right] \\
      & = \left(\bE\left[\exp(\frac{s}{k} (h(x_1,x_{2})-m_h))\right]\right)^{k}
  \end{talign}
    By the one-sided Bernstein's lemma \citet[Prop.~2.14]{wainwright2019high} applied to $\frac{h(x_1,x_2)}{k}$ which is upper bounded by $\frac{b}{k}$ with variance $\frac{\sigma^2_h}{k^2}$, we have
  \begin{talign}
      \bE\left[\exp(s \frac{h(x_1,x_2) - m_h}{k})\right] &\le \exp(\frac{s^2\sigma_h^2/2}{k(k - \frac{bs}{3})}),
  \end{talign}
  for $s \in [0, 3k/b)$.
    Next, by Markov's inequality and Jensen's inequality,
    \begin{talign}
        \Pr(U-m_h > t) &= \Pr(e^{s(U-m_h)} > e^{st}) \le \bE[e^{s(U-m_h)}] e^{-st} \\
        &= \bE\left[\exp(\frac{1}{(n/2)!} \sum_{\sigma\in\perm(n/2)} s(V_\sigma-m_h))\right] e^{-st} \\
        &\le \bE\left[\frac{1}{(n/2)!} \sum_{\sigma\in\perm(n/2)} \exp(s(V_\sigma-m_h))\right] e^{-st} \\
        &= \bE[e^{s(V-m_h)}] e^{-st}.
    \end{talign}
  Therefore,
  \begin{talign}
      \Pr(U-m_h > t)  &\le \exp(\frac{s^2\sigma_h^2}{2(k-\frac{bs}{3})}-st).
  \end{talign}
  Now, we get the desired bound if we pick $s = \frac{k^2 t}{k\sigma_h^2+\frac{ktb}{3}} \in [0,3k/b)$ and simplify.
\end{proof}
Let
\begin{talign}
  h(x,x') &\defeq \wis(x)\wis(x')\langle\Phi_L(x), \Phi_L(x')\rangle \\
  \bar{h}(x, x') &\defeq h(x,x')\ind{h(x,x') \le \Mis^2\frac{B_{n}}{\lambda_L}}.
\end{talign}
Then
\begin{talign}
  \Pr(S_2 < -t_2, E_n) &= \Pr\left(\frac{1}{(n/2)^2} \sum_{i\in\cD_0}\sum_{j\in\cD_1}h(x_i,x_j) > 2t_2, E_n\right)  \\
                 &\le \Pr\left(\frac{1}{(n/2)^2} \sum_{i\in\cD_0}\sum_{j\in\cD_1}\bar{h}(x_i,x_j) > 2t_2\right),\label{eqn:loc:S2_bound}
\end{talign}
where the last inequality used the fact that, for $i \in \cD_0, j \in \cD_1$,
\begin{talign}
  E_n \subset \left\{\max(\kp(x_i, x_i), \kp(x_j, x_j)) \le B_{n}\right\} \subset \left\{h(x_i, x_j) \le \Mis^2\frac{B_{n}}{\lambda_L}\right\},
\end{talign}
using \eqref{eqn:Phi_L_norm_bound}.
We further compute
\begin{talign}
  m_{\bar{h}} &= \bE[\bar{h}(x_1,x_2)] \le \bE[h(x_1,x_2)] = \bE[\wis(x_1)\wis(x_2)\langle\Phi_L(x_1),\Phi_L(x_2)\rangle] \\
              &= \sum_{\ell\le L} \bE[\wis(x_1)\wis(x_2)\phi_\ell(x_1)\phi_\ell(x_2)] \\
              &=\sum_{\ell\le L}(\bE_{x\sim\bP}[\phi_\ell(x)])^2 = 0,
\end{talign}
and
\begin{talign}
  \sigma^2_{\bar{h}} &= \Var[\bar{h}(x_1,x_2)] \le \bE[\bar{h}(x_1,x_2)^2] \le \bE[h(x_1,x_2)^2] \\
                     &= \bE\left[\left(\wis(x_1)\wis(x_2)\langle \Phi_L(x_1),\Phi_L(x_2)\rangle\right)^2\right] \\
                     &\le \Mis^2 \bE_{(x,x')\sim\bP\times\bP}[\langle \Phi_L(x),\Phi_L(x')\rangle^2] \\
                     &=\Mis^2 \bE_{(x,x')\sim\bP\times\bP}\left[\sum_{\ell,\ell'\le L}\phi_\ell(x)\phi_{\ell'}(x)\phi_\ell(x')\phi_{\ell'}(x')\right] \\
                     &= \Mis^2 \sum_{\ell,\ell'\le L} (\bE[\phi_\ell(x)\phi_{\ell'}(x)])^2
                     = L\Mis^2.
\end{talign}
Since $\bE_{x\sim\bP}[\kp(x,x)] = \sum_{\ell} \lambda_\ell \ge L \lambda_L$, we have $L \le \frac{\norm{\kp}^2_{\cL^2(\bP)}}{\lambda_L}$, so that $\sigma_{\bar h}^2 \le \frac{M^2 \norm{\kp}^2_{\cL^2(\bP)}}{\lambda_L}$.
Applying \cref{lem:u_stats_bernstein} to $\bar{h}$, which is bounded by $\Mis^2\frac{B_n}{\lambda_L}$ and using the fact that $m_{\bar h} \le 0$, we have
\begin{talign}
  \Pr\left(\frac{1}{(n/2)^2} \sum_{i\in\cD_0}\sum_{j\in\cD_1}\bar{h}(x_i,x_j) > 2t_2\right) &\le  \Pr\left(\frac{1}{(n/2)^2} \sum_{i\in\cD_0}\sum_{j\in\cD_1}\bar{h}(x_i,x_j) - m_{\bar h} > 2t_2\right) \\
  &\le \exp\left(\frac{-n(2t_2)^2}{4\left(\frac{\Mis^2\norm{\kp}^2_{\cL^2(\bP)}}{\lambda_L} + 2\Mis^2\frac{B_{n}}{\lambda_L}t_2/3\right)}\right).\label{eqn:loc:S2_bound_ex}
\end{talign}
Thus combining \eqref{eqn:loc:hoeffding_S1_bound}, \eqref{eqn:loc:S2_bound}, \eqref{eqn:loc:S2_bound_ex}, we get
\begin{talign}
  \Pr(S < 1/2 - t/2, E_n) \le \exp\left(\frac{-4nt_1^2}{\Mis^2}\right) + \exp\left(\frac{-nt_2^2}{\left(\frac{\Mis^2 \norm{\kp}^2_{\cL^2(\bP)}}{\lambda_L} + 2\Mis^2\frac{B_{n}}{\lambda_L}t_2/3\right)}\right).
\end{talign}

Finally, by symmetry and the union bound, for $t \in (0, 1)$, $t \in (0, t/2)$ and $t_2 = t/2 - t_1$, we have
\begin{talign}
  \Pr\left(\sum_{i\in[n]} w_i < 1-t, E_n\right) &\le \Pr\left(\sum_{i\in \cD_0} w_i < 1/2-t/2, E_n\right) + \Pr\left(\sum_{i\in \cD_1} w_i < 1/2-t/2, E_n\right) \\
  &= 2\Pr(S < 1/2 - t/2, E_n) \\
  &\le 2\left(\exp\left(\frac{-4nt_1^2}{\Mis^2}\right) + \exp\left(\frac{-nt_2^2}{\left(\frac{\Mis^2\norm{\kp}^2_{\cL^2(\bP)}}{\lambda_L} + 2\Mis^2\frac{B_{n}}{\lambda_L}t_2/3\right)}\right)\right). \label{eqn:w_sum_to_one_bound}
\end{talign}

\para{Step 5. Putting it all together}

Define the event
\begin{talign}
  F_n &= \left\{ \min_{i\in [n]} w_i \ge 0, \sum_{i\in[n]}w_i \ge \frac{1}{2} \right\}.
\end{talign}
Then, by the union bound,
\begin{talign}
    \Pr(F_n^c) &\le \Pr(\min_{i\in[n]} w_i < 0, E_n) + \Pr(\sum_{i\in[n]} w_i < \frac{1}{2}, E_n) + \Pr(E_n^c).
\end{talign}
Applying \eqref{eqn:w_nonneg_bound} and \eqref{eqn:w_sum_to_one_bound} to bound the last expression with $t = 1/2$, $t_1=t_2=1/8$, we have $\Pr(F_n^c)\leq \eps_n^2$ for $\eps_n$ defined in \cref{eqn:sub_o(n)_mmd_bound_eps_n}.
On the event $F_n$, if we define $w^+ \in \simplex$ via
\begin{talign}
  w^+_i \defeq \frac{w_i}{\sum_{i\in [n]}w_i},
\end{talign}
then $w^+_i = \alpha w_i$ for $i\in [n]$ and $\alpha \defeq \frac{1}{\sum_{i\in[n]} w_i} \le 2$.
Let $\tilde w \in \simplex$ be the weight defined by $\tilde w_1 = 1$ and $\tilde w_i = 0$ for $i > 1$.

Since $\wopt$ is the best simplex weight, we have $\mmd_{\kp}^2(\bS^{\wopt}_n, \bP) \le \min(\mmd_{\kp}^2(\bS^{w^+}_n, \bP),\mmd_{\kp}^2(\bS^{\tilde w}_n, \bP))$.
Hence 
\begin{talign}
  \bE\left[\mmd^2_{\kp}(\bS^{\wopt}_n, \bP)\right] 
  &= \bE\left[\mmd^2_{\kp}(\bS^{\wopt}_n, \bP)\ind{F_n}\right] + \bE\left[\mmd^2_{\kp}(\bS^{\wopt}_n, \bP)\ind{F_n^c}\right]\\
  &\le \bE\left[\mmd^2_{\kp}(\bS^{w^+}_n, \bP)\ind{F_n}\right] + \bE\left[\mmd^2_{\kp}(\bS^{\tilde w}_n, \bP)\ind{F_n^c}\right]. 
\end{talign}
For the first term, we have the bound
\begin{talign}
  \bE\left[\mmd^2_{\kp}(\bS^{w^+}_n, \bP)\ind{F_n}\right]
  &= \bE \left[\sum_{i,j\in[n]} w_i^+ w_j^+ \kp(x_i, x_j) \ind{F_n}\right] 
  =  \bE \left[\alpha^2 \sum_{i,j\in[n]} w_i w_j \kp(x_i, x_j) \ind{F_n}\right] \\
  &\le 4 \bE \left[\sum_{i,j\in[n]} w_i w_j \kp(x_i, x_j)\right]
  \le \frac{8\Mis}{n}\left( \frac{2\Mis}{n} \frac{\bE_{x\sim\bP}[\kpsq(x,x)]}{\lambda_L}  + \sum_{\ell > L} \lambda_\ell \right),
\end{talign}
where we applied \eqref{eqn:mmd_kp_random_w_bound} for the last inequality.
For the second term, by the Cauchy-Schwartz inequality,
\begin{talign}
  \bE\left[\mmd^2_{\kp}(\bS^{\tilde w}_n, \bP)\ind{F_n^c}\right] &\le \sqrt{\Pr(F_n^c)} \sqrt{\bE\left[\left(\sum_{i,j\in[n]} \kp(x_i,x_j) \tilde w_i \tilde w_j\right)^2\right]} \\
  &\le \sqrt{\Pr(F_n^c)} \sqrt{\bE[\kp(x_1,x_1)^2]}.
\end{talign}
Putting everything together we obtain \eqref{eqn:sub_o(n)_mmd_bound}.
\qed

\section{Stein Kernel Thinning}\label{sec:proof_quad_time}
In this section, we detail our Stein thinning implementation in \cref{sec:st}, our kernel thinning implementation and analysis in \cref{sec:kt}, and our proof of  \cref{thm:GSKT_guarantee} in \cref{sec:thm:GSKT_guarantee_proof}.
\subsection{Stein Thinning with sufficient statistics}\label{sec:st}
For an input point set of size $n$, the original implementation of Stein Thinning of \citet{riabiz2022optimal} takes $O(nm^2)$ time to output a coreset of size $m$.
In \cref{alg:GBC}, we show that this runtime can be improved to $O(nm)$ using sufficient statistics.
The idea is to maintain a vector $g \in \bR^n$ such that $g = 2\kp(\bS_n, \bS_n)w$ where $w$ is the weight representing the current coreset. 
\begin{algorithm}[htb]
  \caption{\GBCfullnotag (\GBCnotag) with sufficient statistics}
  \label{alg:GBC}
  \begin{algorithmic}\itemindent=-.7pc
    \STATE {\bf Input:} kernel $\kp$ with zero-mean under $\bP$, input points $\cS_n=(x_i)_{i\in[n]}$, output size $m$
    \STATE $w \gets \bm{0} \in \bR^n$
    \STATE $j \gets \argmin_{i\in[n]}\kp(x_i, x_i)$
    \STATE $w_j \gets 1$
    \STATE $g \gets 2\kp(\cS_n, x_j)$ \COMMENT{maintain sufficient statistics $g = 2\kp(\cS_n, \cS_n)w$}
    \FOR{$t=1$ {\bf to} $m-1$}\itemindent=-.7pc
    \STATE $j \gets \argmin_{i\in[n]}\{t g_i + \kp(x_i, x_i)\}$
    \STATE $w \gets \frac{t}{t+1}w + \frac{1}{t+1} e_j$
    \STATE $g \gets \frac{t}{t+1}g + \frac{2}{t+1} \kp(\cS_n, x_j)$
    \ENDFOR
    \STATE {\bf Return:} $w$
  \end{algorithmic}
\end{algorithm}

\subsection{Kernel Thinning targeting $\bP$}\label{sec:kt}
\begin{algorithm}[htb]
  \caption{\kttargetfullnotag (\kttargetnotag) (adapted from  \citet[Alg.~1]{dwivedi2022generalized}) }
  \label{alg:kt_target}
  \begin{algorithmic}\itemindent=-.7pc
    \STATE {\bf Input:} kernel $\kp$ with zero-mean under $\bP$, input points $\cS_n=(x_i)_{i\in[n]}$, multiplicity $n'$ with $\log_2 \frac{n'}{m} \in \bN$, weight $w \in \simplex \cap (\frac{\bN_0}{n'})^n$, output size $m$ with $\frac{n'}{m} \in 2^\bN$, failure probability $\delta$
    \STATE $\tS \gets $ index sequence where $k \in [n]$ appears $n'w_k$ times
    \STATE $\ft \gets \log_2\frac{n'}{m} \in \bN$
    \STATE $(\tI^{(\ell)})_{\ell\in[2^\ft]} \gets \ktsplit(\kp, \cS_n[\tS], \ft, \delta/n')$ \COMMENT{\ktsplit is from \citet[Algorithm 1a]{dwivedi2022generalized} and we set $\delta_i=\delta$ for all $i$}
    \STATE $\tI^{(\ell)} \gets \tS[\tI^{(\ell)}]$ for each $\ell \in [2^\ft]$
    \STATE $\tI \gets \ktswaptarget(\kp, \cS_n, (\tI^{(\ell)})_{\ell\in[2^\ft]})$ 
    \STATE $w_{\kttargetnotag} \gets \text{simplex weight corresponding to } \tI$ \COMMENT{$w_i = \frac{\text{number of occurrences of } i \text{ in } \tI}{\abs{\tI}}$}
    \STATE {\bf Return:} $w_{\kttargetnotag} \in \simplex \cap (\frac{\bN_0}{m})^n$ \COMMENT{Hence $\norm{w_{\kttargetnotag}}_0 \le m$}
  \end{algorithmic}
\end{algorithm}

\begin{algorithm}[htb]
  \caption{\ktswaptargetnotag  (modified \citet[Alg.~1b]{dwivedi2022generalized} to minimize $\mmd$ to $\P$) 
	}
  \label{alg:kt_swap_target}
  \begin{algorithmic}\itemindent=-.7pc
    \STATE {\bf Input:} kernel $\kp$ with zero-mean under $\bP$, input points $\cS_n=(x_i)_{i\in[n]}$, candidate coreset indices $(\tI^{(\ell)})_{\ell\in[L]}$
    \STATE $m \gets \abs{\tI^{(0)}}$ \COMMENT{all candidate coresets are of the same size}
    \STATE $\tI \gets \tI^{(\ell^*)}$ for $\ell^* \in \argmin_{\ell\in [L]}\mmd_{\kp}(\cS_n[\tI^{(\ell)}], \bP)$ \COMMENT{select the best \ktsplit coreset}
    \STATE $\tI_{\GBCnotag} \gets $ index sequence of $\GBCfull(\kp, \cS_n, m)$\COMMENT{add Stein thinning baseline}
    \STATE $\tC = \{\tI, \tI_{\GBCnotag}\}$ \COMMENT{shortlisted candidates}
    \FOR{$\tI \in \tC$}\itemindent=-.7pc
    \STATE $g \gets \bm{0} \in \bR^n$
    \COMMENT{maintain sufficient statistics $g = \sum_{j\in[m]} \kp(x_{\tI_j}, \cS_n)$}
    \STATE $\mathtt{Kdiag} \gets (\kp(x_i,x_i))_{i\in[n]}$
    \FOR{$j=1$ {\bf to} $m$}\itemindent=-.7pc
    \STATE $g \gets g + \kp(x_{\tI_j}, \cS_n)$
    \ENDFOR
    \FOR{$j=1$ {\bf to} $m$}\itemindent=-.7pc
    \STATE $\Delta = 2(g -  \kp(x_{\tI_j}, \cS_n))+ \mathtt{Kdiag}$ \COMMENT{this is the change in $\mmd^2_{\kp}(\cS_n[\tI], \bP)$ if we were to replace $\tI_j$}
    \STATE $k \gets \argmin_{i\in[n]}\Delta_i$
    \STATE $g = g - \kp(x_{\tI_j}, \cS_n) + \kp(x_{k}, \cS_n)$
    \STATE $\tI_j \gets k$
    \ENDFOR
    \ENDFOR
    \STATE {\bf Return:} $\tI = \argmin_{\tI\in \tC}\mmd_{\kp}(\cS_n[\tI], \bP)$
  \end{algorithmic}
\end{algorithm}
Our \kttargetfull implementation is detailed in \cref{alg:kt_target}. 
Since we are able to directly compute $\mmd_{\kp}(\bS_n^w, \bP)$, we use \ktswaptarget (\cref{alg:kt_swap_target}) in place of the standard \ktswap subroutine \citep[Algorithm 1b]{dwivedi2022generalized} to choose candidate points to swap in so as to greedily minimize $\mmd_{\kp}(\bS_n^w, \bP)$.
To facilitate our subsequent \GSKT analysis, we restate the guarantees of \ktsplit \citep[Theorem 2]{dwivedi2022generalized} in the sub-Gaussian format of \citep[Definition 3]{shetty2022distribution}.
\begin{lemma}[Sub-Gaussian guarantee for \ktsplit]\label{lem:kt_mmd_guarantee_sg}
  Let $\cS_n$ be a sequence of $n$ points and $\k$ a kernel.
  For any $\delta \in (0, 1)$ and $m \in \bN$ such that $\log_2\frac{n}{m} \in \bN$, consider the \ktsplit algorithm  \citep[Algorithm 1a]{dwivedi2022generalized} with $\ksplit = \k$, thinning parameter $\ft = \log_2\frac{n}{m}$, and $\delta_i = \frac{\delta}{n}$ to compress $\cS_n$ to $2^{\ft}$ coresets $\{\cS^{(i)}_{\out}\}_{i\in[2^\ft]}$ where each $\cS^{(i)}_\out$ has $m$ points. 
  Denote the signed measure $\phi^{(i)} \defeq \frac{1}{n}\sum_{x \in \cS_n}\delta_x - \frac{2^\ft}{n} \sum_{x \in \cS^{(i)}_{\out}}\delta_x$.
  Then for each $i\in[2^\ft]$, on an event $\Eequi^{(i)}$  with $\bP(\Eequi^{(i)}) \ge 1-\frac{\delta}{2}$, $\phi^{(i)} = \tilde \phi^{(i)}$ for a random signed measure $\tilde \phi^{(i)}$\footnote{This is the signed measure returned by repeated applications of self-balancing Hilbert walk (SBHW) \citep[Algorithm 3]{dwivedi2021kernel}. Although SBHW returns an element of $\rkhs[\k]$, by tracing the algorithm, the returned output is equivalent to a signed measure via the correspondence $\sum_{i\in[n]}c_i \k(x_i,\cdot) \Leftrightarrow \sum_{i\in [n]}c_i \delta_{x_i}$. The usage of signed measures is consistent with \citet{shetty2022distribution}.} such that, for any $\delta'\in (0, 1)$, 
\begin{talign}
  \Pr\left(\norm{\tilde \phi^{(i)} \k}_{\rkhs[\k]} \ge a_{n, m} + v_{n, m} \sqrt{\log(\frac{1}{\delta'})}\right) \le \delta',\label{eqn:kt_guarantee_sg}
\end{talign}
  where
  \begin{talign}
    a_{n, m} &\defeq \frac{1}{m} \left(2 + \sqrt{\frac{8}{3}\knormSn\log(\frac{6(\log_2\frac{n}{m})m}{\delta})\log (4\cvrnum{\k}(\balleuc(R_{n}), m^{-1}))}\right), \label{eqn:kt_sg_a} \\
    v_{n, m} &\defeq \frac{1}{m}\sqrt{\frac{8}{3}\knormSn\log(\frac{6(\log_2\frac{n}{m})m}{\delta})}.
  \end{talign}
\end{lemma}
\begin{proof}[Proof of \cref{lem:kt_mmd_guarantee_sg}]
  Fix $i \in [2^\ft]$, $\delta \in (0, 1)$ and $n, m \in \bN$ such that $\ft = \log_2\frac{n}{m} \in \bN$.
  Define $\phi \defeq \phi^{(i)}$.
  By the proof of \citet[Thms.~1 and 2]{dwivedi2022generalized}, there exists an event $\Eequi$ with $\Pr(\Eequi^c) \le \frac{\delta}{2}$ such that, on this event, $\phi = \tilde \phi$ where $\tilde \phi$ is a signed measure such that, for any $\delta' \in (0, 1)$, with probability at least $1-\delta'$,
  \begin{talign}
    \norm{\tilde\phi \k}_{\rkhs[\k]} \le \inf_{\eps \in (0, 1), A: \cS_n \subset A} 2\eps + \frac{2^\ft}{n}\sqrt{\frac{8}{3}\knormSn\log(\frac{6\ft n}{2^\ft \delta})\left[\log\frac{4}{\delta'} + \log\cvrnum{\k}(A, \eps)\right]}.
  \end{talign}
  Note that on $\Eequi$, $\norm{\tilde\phi \k}_{\rkhs[\k]} = \norm{\phi \k}_{\rkhs[\k]}$.
  We choose $A = \balleuc(R_{n})$ and $\eps = \frac{2^\ft}{n} = m^{-1}$, so that, with probability at least $1 - \delta'$, using the fact that $\sqrt{a+b} \le \sqrt{a} + \sqrt{b}$ for $a,b\ge 0$,
  \begin{talign}
    \norm{\tilde\phi \k}_{\rkhs[\k]} &\le \frac{2^{\ft+1}}{n} + \frac{2^\ft}{n} \sqrt{\frac{8}{3}\knormSn\log(\frac{6\ft n}{2^\ft\delta})\left[\log\frac{4}{\delta'} + \log\cvrnum{\k}(\balleuc(R_{n}), m^{-1})\right]} \label{eqn:loc:kt_split_sub_gauss}\\
                                 &\le \frac{2^{\ft+1}}{n} + \frac{2^\ft}{n} \sqrt{\frac{8}{3}\knormSn\log(\frac{6\ft n}{2^\ft\delta})}\left[\sqrt{\log\frac{1}{\delta'}} + \sqrt{\log4\cvrnum{\k}(\balleuc(R_{n}), m^{-1})}\right] \\
                                 &\le a_{n,m} + v_{n,m} \sqrt{\log(\frac{1}{\delta'})},
  \end{talign}
  for $a_{n,m}$, $v_{n,m}$ in \cref{lem:kt_mmd_guarantee_sg}.
\end{proof}

\begin{corollary}[MMD guarantee for \ktsplit]\label{prop:kt_mmd_guarantee}
  Let $\cS_\infty$ be an infinite sequence of points in $\bR^d$ and $\k$ a kernel.
  For any $\delta \in (0,1)$ and $n,m \in \bN$ such that $\log_2 \frac{n}{m} \in \bN$, consider the \ktsplit algorithm \citep[Algorithm 1a]{dwivedi2022generalized} with parameters $\ksplit = \k$ and $\delta_i = \frac{\delta}{n}$ to compress $\cS_n$ to $2^\ft$ coresets $\{\cS_{\out}^{(i)}\}_{i\in[2^\ft]}$ where $\ft = {\log_2\frac{n}{m}}$, each with $m$ points. 
  Then for any $i \in [2^\ft]$, with probability at least $1 - \delta$,
\begin{talign}
  \mmd_{\k}(\bS_n, \bS^{(i)}_{\out}) &\le  \frac{1}{m} \left(2 + \sqrt{\frac{8}{3}\knormSn\log(\frac{6(\log_2\frac{n}{m})m}{\delta})\left(\log \cvrnum{\k}(\balleuc(R_{n}), m^{-1}) + \log \frac{8}{\delta}\right)}\right).\label{eqn:kt_mmd_guarantee}
\end{talign}
\end{corollary}

\begin{proof}[Proof of \cref{prop:kt_mmd_guarantee}]
  Fix $i \in [2^\ft]$.
  By taking $\delta' = \frac{\delta}{2}$ in \cref{eqn:loc:kt_split_sub_gauss}, we obtain \eqref{eqn:kt_mmd_guarantee}. %
  This occurs with probability
  \begin{talign}
    &\Pr(\mmd_{\k}(\bS_n, \bS^{(i)}_{\out}) < a_{n,m} + v_{n,m} \sqrt{\log(\frac{1}{\delta'})}) \\
    =& 1-\Pr\left(\mmd_{\k}(\bS_n,  \bS^{(i)}_{\out}) \ge a_{n,m} + v_{n,m} \sqrt{\log(\frac{1}{\delta'})}\right)  \\
    \ge& 1-\Pr\left(\Eequi^{(i)}, \mmd_{\k}(\bS_n,  \bS^{(i)}_{\out}) \ge a_{n,m} + v_{n,m} \sqrt{\log(\frac{1}{\delta'})}\right) - \Pr\left({\Eequi^{(i)}}^c\right) \\
    \ge& 1-\Pr\left(\norm{\tilde \phi^{(i)} \k}_{\rkhs[\k]} \ge a_{n,m} + v_{n,m} \sqrt{\log(\frac{1}{\delta'})}\right) -\Pr({\Eequi^{(i)}}^c) \\
                                                                                         \ge& 1-\frac{\delta}{2} - \frac{\delta}{2} = 1 - \delta.
  \end{talign}
\end{proof}

\subsection{\pcref{thm:GSKT_guarantee}}\label{sec:thm:GSKT_guarantee_proof}
\cref{thm:GSKT_guarantee} follows directly from  \cref{assum:kernel_growth} and the following result for a kernel with generic covering number.
\begin{theorem}\label{thm:GSKT_guarantee_cvrnum}
  Let $\kp$ be a kernel satisfying \cref{assum:mean_zero_p}.
  Let $\cS_\infty$ be an infinite sequence of points.
  Then for a prefix sequence $\cS_n$ of $n$ points, $m\in [n]$, and $n' \defeq m 2^{\ceil{\log_2\frac{n}{m}}}$,
  \GSKT outputs $w_{\GSKTnotag}$ in $O(n^2\kev)$ time that satisfies, with probability at least $1-\delta$,
  \begin{talign}
    \Delta\!\mmd_{\kp}(w_{\GSKTnotag}) \le \sqrt{\left(\frac{1+\log n'}{n'}\right)\kpnormSn} + \frac{1}{m} \left(2 + \sqrt{\frac{8}{3}\knormSn\log(\frac{6(\log_2\frac{n'}{m})m}{\delta})\left(\log \cvrnum{\k}(\balleuc(R_{n}), m^{-1}) + \log \frac{8}{\delta}\right)}\right).
  \end{talign}
\end{theorem}
\begin{proof}[Proof of \cref{thm:GSKT_guarantee_cvrnum}]
  The runtime of \GSKT comes from the fact that all of \GBCfull (with output size $n$), \ktsplit, and \ktswaptarget take $O(\kev n^2)$ time.

  By \citet[Theorem 1]{riabiz2022optimal}, \GBCfull (which is a deterministic algorithm) from $n$ points to $n'$ points has the following guarantee
  \begin{talign}
    \mmd^2_{\kp}(\bS_n^{w^\dagger},\bP) \le \mmd^2_{\kp}(\bS_n^{\wopt}, \bP) + \left(\frac{1+\log n'}{n'}\right)\kpnormSn,
  \end{talign}
  where we denote the output weight of \GBCfull as $w^\dagger$.
  Using $\sqrt{a+b}\le \sqrt{a}+\sqrt{b}$ for $a,b\ge 0$, we have
  \begin{talign}
    \mmd_{\kp}(\bS_n^{w^\dagger},\bP) \le \mmd_{\kp}(\bS_n^{\wopt}, \bP) + \sqrt{\left(\frac{1+\log n'}{n'}\right)\kpnormSn}.
  \end{talign}
  Fix $\delta \in (0, 1)$. By \cref{prop:kt_mmd_guarantee} with $\k=\kp$ and $\ft = \log_2\frac{n'}{m}$, with probability at least $1 - \delta$, we have, for any $i \in [2^\ft]$,
\begin{talign}
  \mmd_{\kp}(\bS^{w^\dagger}_n, \bS^{(i)}_{\out}) \le \frac{1}{m} \left(2 + \sqrt{\frac{8}{3}\knormSn\log(\frac{6(\log_2\frac{n'}{m})m}{\delta})\left(\log \cvrnum{\k}(\balleuc(R_{n}), m^{-1}) + \log \frac{8}{\delta}\right)}\right),
\end{talign}
where $\bS_{\out}^{(i)}$ is the $i$-th coreset output by \ktsplit.
Since \ktswaptarget can only decrease the MMD to $\P$, we have, by the triangle inequality of $\mmd_{\kp}$,
  \begin{talign}
    \mmd_{\kp}(\bS_n^{w_\GSKTnotag},\bP) &\le \mmd_{\kp}(\bS_\out^{(1)},\bP) \le \mmd_{\kp}(\bS_\out^{(1)}, \bS_n^{w^\dagger}) + \mmd_{\kp}(\bS_n^{w^\dagger}, \bP),
  \end{talign}
  which gives the desired bound.
\end{proof}

\cref{thm:GSKT_guarantee} now follows from \cref{thm:GSKT_guarantee_cvrnum}, the kernel growth definitions in \cref{assum:kernel_growth}, $n \le n' \le 2n$, and that $\log_2 (\frac{n'}{m})m  \le n'$.
\qed

\section{Resampling of Simplex Weights}\label{sec:resample}
Integral to many of our algorithms is a resampling procedure that turns a simplex-weighted point set of size $n$ into an equal-weighted point set of size $m$ while incurring at most $O(1/\sqrt{m})$ MMD error.
The motivation for wanting an equal-weighted point set is two-fold: First, in \LSKT, we need to provide an equal-weighted point set to \compresspptarget, but the output of \alrbc is a simplex weight.
Secondly, we can exploit the fact that non-zero weights are bounded away from zero in equal-weighted point sets to provide a tighter analysis of \rpc.
While \iid resampling also achieves the $O(1/\sqrt{m})$ goal, we choose \stratresamp (\cref{alg:stratified_resample}), a stratified residual resampling algorithm \citep[Sec. 3.2, 3.3]{douc2005comparison}.
In this section, we derive an MMD bound for \stratresamp and show that it is better in expectation than using \iid resampling or residual resampling alone.

Let $D_w^\inv$ be the inverse of the cumulative distribution function of the multinomial distribution with weight $w$, i.e., 
\begin{talign}
  D_w^\inv(u) \defeq \min\left\{i \in [n]: u \le \sum_{j=1}^i w_j\right\}.
\end{talign}
\begin{algorithm}[htb]
  \caption{\iid resampling}
  \label{alg:iid_resample}
  \begin{algorithmic}\itemindent=-.7pc
    \STATE {\bfseries Input:} Weights $w \in \simplex$, output size $m$
    \STATE $w' \gets \bm 0 \in \bR^n$
    \FOR{$j=1$ {\bfseries to} $m$}\itemindent=-.7pc
    \STATE Draw $U_j \sim \unif([0, 1))$
    \STATE $\tI_j \gets D_w^{\inv}(U_j)$
    \STATE $w'_{\tI_j} \gets w'_{\tI_j} + \frac{1}{m}$
    \ENDFOR

    \STATE {\bfseries Return:} $w'\in\Delta_{n-1}\cap(\frac{\N_0}{m})^{n}$
  \end{algorithmic}
\end{algorithm}
\begin{algorithm}[htb]
  \caption{Residual resampling}
  \label{alg:truncated_resample}
  \begin{algorithmic}\itemindent=-.7pc
    \STATE {\bfseries Input:} Weights $w \in \simplex$, output size $m$
    \STATE $w'_i \gets \frac{\lfloor mw_i \rfloor}{m}$, $\forall i\in[n]$
    \STATE $r \gets m - \sum_{i\in[n]} \lfloor m w_i \rfloor \in \bN$
    \STATE $\eta_i \gets \frac{mw_i - \lfloor mw_i \rfloor}{r}$, $\forall i\in[n]$
    
    \FOR{$j=1$ {\bfseries to} $r$}\itemindent=-.7pc
    \STATE Draw $U_j \sim \unif([0, 1))$
    \STATE $\tI_j \gets D_\eta^{\inv}(U_j)$
    \STATE $w'_{\tI_j} \gets w'_{\tI_j} + \frac{1}{m}$
    \ENDFOR

    \STATE {\bfseries Return:} $w'\in\Delta_{n-1}\cap(\frac{\N_0}{m})^{n}$
  \end{algorithmic}
\end{algorithm}
\begin{algorithm}[htb]
  \caption{Stratified residual resampling (\stratresampnotag)}
  \label{alg:stratified_resample}
  \begin{algorithmic}\itemindent=-.7pc
    \STATE {\bfseries Input:} Weights $w \in \simplex$, output size $m$
    \STATE $w'_i \gets \frac{\lfloor mw_i \rfloor}{m}$, $\forall i\in[n]$
    \STATE $r \gets m - \sum_{i\in[n]} \lfloor m w_i \rfloor \in \bN$
    \STATE $\eta_i \gets \frac{mw_i - \lfloor mw_i \rfloor}{r}$, $\forall i\in[n]$
    
    \FOR{$j=1$ {\bfseries to} $r$}\itemindent=-.7pc
    \STATE Draw $U_j \sim \unif([\frac{j}{r}, \frac{j+1}{r}))$
    \STATE $\tI_j \gets D_\eta^{\inv}(U_j)$
    \STATE $w'_{\tI_j} \gets w'_{\tI_j} + \frac{1}{m}$
    \ENDFOR

    \STATE {\bfseries Return:} $w'\in\Delta_{n-1}\cap(\frac{\N_0}{m})^{n}$
  \end{algorithmic}
\end{algorithm}

\begin{proposition}[MMD guarantee of resampling algorithms]
  \label{prop:mmd_resample}
  Consider any kernel $\k$, points $\cS_n=(x_1,\ldots,x_n) \subset \bR^d$, and a weight vector $w\in\simplex$.
  \begin{enumerate}[label=(\alph*)]
    \item\label{itm:mmd_iid_resample}
      Using the notation from \cref{alg:iid_resample}, let $X, X'$ be independent random variables with law $\bS_n^w$.
      Then, the output weight vector $w^\iidup \defeq w'$ of \cref{alg:iid_resample} satisfies
      \begin{talign}
        \bE[\mmd^2_{\k}(\bS^{w^{\iidup}}_n, \bS_n^w)] = \frac{\bE\k(X,X) - \bE\k(X,X')}{m}.\label{eqn:mmd_iid_resample}
      \end{talign}
    \item\label{itm:mmd_truncated_resample}
      Using the notation from \cref{alg:truncated_resample}, let $R, R'$ be independent random variables with law $\bS_n^\eta$.
      Then, the output weight vector $w^\trunc \defeq w'$ of \cref{alg:truncated_resample} satisfies
      \begin{talign}
        \bE[\mmd^2_{\k}(\bS_n^{w^\trunc},\bS_n^w)] = \frac{r(\bE \k(R,R) - \bE\k(R,R'))}{m^2}.\label{eqn:mmd_truncated_resample}
      \end{talign}
    \item\label{itm:mmd_stratified_resample}
      Using the notation from \cref{alg:stratified_resample}, let $R_j \defeq x_{\tI_j}$ and $R'_j$ be an independent copy of $R_j$.
      Let $R$ be an independent random variable with law $\bS_n^\eta$.
      Then, the output weight vector $w^\truncstrat \defeq w'$ of \cref{alg:stratified_resample} satisfies
      \begin{talign}
        \bE[\mmd^2_{\k}(\bS^{w^\truncstrat}_n, \bS_n^w)] = \frac{r\bE\k(R,R) - \sum_{j\in[r]}\bE\k(R_j,R_j')}{m^2}.\label{eqn:mmd_stratified_resample}
      \end{talign}
  \end{enumerate}
\end{proposition}
\begin{proof}[Proof of \cref{prop:mmd_resample}\ref{itm:mmd_iid_resample}]
  Let $X_i \defeq x_{\tI_i}$. As random signed measures, we have
  \begin{talign}
    \bS_n^{w'} - \bS_n^w &= \frac{1}{m}\sum_{i\in[m]} \delta_{X_i} - \sum_{i\in[n]} w_i \delta_{x_i}.
  \end{talign}
  Hence
  \begin{talign}
    \mmd^2_\k(\bS^{w'}_n, \bS_n^w) &= ((\bS_n^{w'} - \bS_n^w) \times (\bS_n^{w'} - \bS_n^w))\k \\
                                   &=  \frac{1}{m^2}\sum_{i,i'\in[m]} \k(X_i, X_{i'}) - \frac{2}{m}\sum_{i\in[m],i'\in[n]} w_{i'}\k(X_i, x_{i'}) + \sum_{i,i'\in[n]} w_i w_{i'}\k(x_i, x_{i'}).
  \end{talign}
  Since each $X_i$ is distributed to $\bS_n^w$ and $X_i$ and $X_{i'}$ are independent for $i \neq {i'}$, taking expectation, we have
  \begin{talign}
    \bE[\mmd^2_\k(\bS^{w'}_n, \bS_n^w)] &= \frac{1}{m} \bE\k(X,X) + \frac{m-1}{m} \bE\k(X,X') - 2 \bE\k(X,X') + \bE\k(X,X').
  \end{talign}
  This gives the bound \eqref{eqn:mmd_iid_resample}.
\end{proof}
\begin{proof}[Proof of \cref{prop:mmd_resample}\ref{itm:mmd_truncated_resample}]
  Let $R_j \defeq x_{\tI_j}$. As random signed measures, we have
  \begin{talign}
    \bS_n^{w'} - \bS_n^w &= \left(\sum_{i\in[n]} \frac{\lfloor mw_i \rfloor}{m}\delta_{x_i} + \frac{1}{m}\sum_{j\in[r]}\delta_{R_j}\right) - \sum_{i\in[n]}w_i\delta_{x_i} \\
                         &= \frac{1}{m}\sum_{j\in[r]} \delta_{R_j} - \sum_{i\in[n]} \left(w_i - \frac{\lfloor mw_i \rfloor}{m}\right)\delta_{x_i} \\
                         &= \frac{1}{m}\sum_{j\in[r]} \delta_{R_j} - \frac{r}{m}\sum_{i\in[n]} \eta_i\delta_{x_i}.
  \end{talign}
  Hence
  \begin{talign}
    \mmd^2_\k(\bS^{w'}_n, \bS_n^w) &= ((\bS_n^{w'} - \bS_n^w) \times (\bS_n^{w'} - \bS_n^w))\k\\
                                   &= \frac{1}{m^2}  \sum_{j,j'\in[r]}\k(R_j, R_{j'}) - \frac{2r}{m^2}\sum_{j\in[r], i\in[n]} \eta_i \k(R_j, x_i) + \frac{r^2}{m^2}\sum_{i,i'\in[n]}\eta_i\eta_j \k(x_i,x_j).\label{eqn:loc:trunc_mmd_sqr}
  \end{talign}
  Since each $R_j$ is distributed to $\bS_n^\eta$ and $R_j$ and $R_{j'}$ are independent for $j \neq j'$, taking expectation, we have
  \begin{talign}
    \bE[\mmd^2_\k(\bS^{w'}_n, \bS_n^w)] &= \frac{r}{m^2} \bE\k(R,R) + \frac{r(r-1)}{m^2} \bE\k(R,R') - \frac{2r^2}{m^2} \bE\k(R,R') + \frac{r^2}{m^2}\bE\k(R,R').
  \end{talign}
  This gives the bound \eqref{eqn:mmd_truncated_resample}.
\end{proof}
\begin{proof}[Proof of \cref{prop:mmd_resample}\ref{itm:mmd_stratified_resample}]
  We repeat the same steps from the previous part of the proof to get \eqref{eqn:loc:trunc_mmd_sqr}.
  In the case of \ref{itm:mmd_stratified_resample}, $R_j$'s are not identically distributed so the analysis is different.
  Let $R'$ be an independent copy of $R$.
  Taking expectation of \eqref{eqn:loc:trunc_mmd_sqr}, we have
  \begin{talign}
    m^2\bE[\mmd^2_\k(\bS^{w'}_n, \bS_n^w)] &= \sum_{j\in[r]} \bE \k(R_j, R_j) + \sum_{j\in[r]}\sum_{j'\in[r]\setminus\{j\}} \bE\k(R_j, R_{j'}) - 2r\sum_{j\in[r]}\bE \k(R_j, R) + r^2\bE\k(R,R').
  \end{talign}
  Note
  \begin{talign}
    \sum_{j\in[r]} \bE \k(R_j, R_j) = \sum_{j\in[r]} r \int_{[\frac{j}{r}, \frac{j+1}{r})} \k(x_{D_\eta^\inv(u)}, x_{D_\eta^\inv(u)})\dd u = r\int_0^1 \k(x_{D_\eta^\inv(u)}, x_{D_\eta^\inv(u)})\dd u = r\bE\k(R,R),
  \end{talign}
  where we used the fact that $x_{D_\eta^\inv(U)} \stackrel{D}{=} R$ for $U \sim \unif([0, 1])$.
  Similarly, we deduce
  \begin{talign}
    \sum_{j\in[r]}\sum_{j'\in[r]\setminus\{j\}} \bE\k(R_j, R_{j'}) &= \sum_{j\in[r]}\left(\sum_{j'\in[r]}\bE\k(R_j,R'_{j'}) - \bE\k(R_j,R'_j)\right)\\
                                                                   &= \sum_{j\in[r]}(r\bE\k(R_j, R') - \bE\k(R_j, R'_j))\\
                                                                   &= r^2\bE\k(R, R') - \sum_{j\in[r]}\bE\k(R_j,R_j'),
  \end{talign}
  and also
  \begin{talign}
    \sum_{j\in[r]}\bE\k(R_j, R) = r\bE \k(R,R'). 
  \end{talign}
  Combining terms, we get
  \begin{talign}
    m^2\bE \mmd^2_\k(\bS^{w'}_n, \bS_n^w) &= r\bE\k(R,R) + r^2\bE\k(R, R') - \sum_{j\in[r]}\bE\k(R_j,R_j') - 2r^2 \bE\k(R, R') + r^2\bE\k(R,R') \\
                                          &= r\bE\k(R,R) - \sum_{j\in[r]}\bE\k(R_j,R_j'),
  \end{talign}
  which yields the desired bound \eqref{eqn:mmd_stratified_resample}.
\end{proof}
The next proposition shows that stratifying the residuals always improves upon using \iid sampling or residual resampling alone.
We need the following convexity lemma.
\begin{lemma}[Convexity of squared MMD]\label{lem:Ik_cvx}
  Let $\k$ be a kernel. 
  Let $\cS_n=(x_1,\ldots,x_n)$ be an arbitrary set of points.
  The function $E_\k: \bR^n \to \bR$ defined by
  \begin{talign}
    E_\k(w) \defeq \norm{\bS_n^w \k}^2_{\rkhs[\k]} = \sum_{i,j\in[n]} w_iw_j\k(x_i,x_j)
  \end{talign}
  is convex on $\bR^n$.
\end{lemma}
\begin{proof}[Proof of \cref{lem:Ik_cvx}]
    Since $\k$ is a kernel, the Hessian $\hess E_\k = 2\k(\cS_n,\cS_n)$ is PSD, and hence $E_\k$ is convex.
\end{proof}
\begin{proposition}[Stratified residual resampling improves MMD]\label{prop:mmd_resample_tower}
  Under the assumptions of \cref{prop:mmd_resample}, we have
  \begin{talign}
    \bE[\mmd^2_{\k}(\bS^{w^{\iidup}}_n, \bS_n^w)] \ge \bE[\mmd^2_{\k}(\bS_n^{w^\trunc},\bS_n^w)] \ge \bE[\mmd^2_{\k}(\bS^{w^\truncstrat}_n, \bS_n^w)].
  \end{talign}
\end{proposition}
\begin{proof}[Proof of \cref{prop:mmd_resample_tower}]
  Let $K \defeq \k(\cS_n, \cS_n)$.
  To show the first inequality, note that since $\eta = \frac{mw-\lfloor mw\rfloor}{r}$, by \cref{prop:mmd_resample},
  \begin{talign}
    \bE[\mmd^2_{\k}(\bS_n^{w^\trunc},\bS_n^w)] &= \frac{r(\bE\k(R,R) - \bE\k(R,R'))}{m^2} \\
                                               &= \frac{r(\sum_{i\in[n]}K_{ii}\eta_i - \sum_{i,j\in[n]}K_{ij}\eta_i\eta_j)}{m^2}\\
                                               &= \frac{1}{m}\left(\sum_{i\in[n]} K_{ii}\left(w_i-\frac{\lfloor mw_i \rfloor}{m}\right) - \frac{m}{r}\left(w-\frac{\lfloor mw \rfloor}{m}\right)^\top K \left(w-\frac{\lfloor mw \rfloor}{m}\right)\right).
  \end{talign}
  Hence
  \begin{talign}
   &\bE[\mmd^2_{\k}(\bS_n^{w^\iidup},\bS_n^w)] - \bE[\mmd^2_{\k}(\bS_n^{w^\trunc},\bS_n^w)] \\
    =& \frac{1}{m}\left(\sum_{i\in[n]} K_{ii}\frac{\lfloor mw_i \rfloor}{m} + \frac{m}{r}\left(w-\frac{\lfloor mw \rfloor}{m}\right)^\top K \left(w-\frac{\lfloor mw \rfloor}{m}\right) - w^\top K w\right) \\
    =& \frac{1}{m}\left( (1-\theta)\sum_{i\in[n]}K_{ii}\xi_i + \theta \eta^\top K \eta - w^\top K w \right),
  \end{talign}
 where we let $\xi \defeq \frac{m}{m-r}\frac{\lfloor mw \rfloor}{m}$ and $\theta \defeq \frac{r}{m}$. 
 Note that $w = \theta \eta + (1-\theta)\xi$.
  By \cref{lem:Ik_cvx} and Jensen's inequality, we have
  \begin{talign}
    w^\top K w &= E_\k(w) \le \theta E_\k(\eta) + (1-\theta) E_\k(\xi) = \theta \eta^\top K \eta + (1-\theta) \xi^\top K \xi \le\theta \eta^\top K \eta + (1-\theta) \sum_{i\in[n]}K_{ii}\xi_i,
  \end{talign}
  where the last inequality follows from \cref{prop:mmd_resample}\ref{itm:mmd_iid_resample} with $w=\xi$ and the fact that MMD is nonnegative.
  Hence we have shown
  \begin{talign}
    \bE[\mmd^2_{\k}(\bS_n^{w^\iidup},\bS_n^w)] - \bE[\mmd^2_{\k}(\bS_n^{w^\trunc},\bS_n^w)] \ge 0,
  \end{talign}
  as desired.

  For the second inequality, by \cref{prop:mmd_resample}, we compute
  \begin{talign}
    &\bE[\mmd^2_{\k}(\bS_n^{w^\trunc},\bS_n^w)] - \bE[\mmd^2_{\k}(\bS_n^{w^\truncstrat},\bS_n^w)] 
    = \frac{r}{m^2}\left(\frac{1}{r}\sum_{j\in[r]}\bE \k(R_j,R'_j) - \bE\k(R,R')\right).
  \end{talign}
  Note that
  \begin{talign}
    \bE\k(R,R') = \int_{[0,1)}\int_{[0,1)} k(x_{D_\eta^\inv(u)}, x_{D_\eta^\inv(v)}) \dd u \dd v = E_\k\left((D_\eta^\inv)_\# \unif[0,1)\right),
  \end{talign}
  where we used $T_\# \mu$ to denote the pushforward measure of $\mu$ by $T$.
  Similarly,
  \begin{talign}
    \frac{1}{r}\sum_{j\in[r]}\bE\k(R_j,R'_j) &= \frac{1}{r}\sum_{j\in[r]}\int_{[\frac{j}{r},\frac{j+1}{r})}\int_{[\frac{j}{r},\frac{j+1}{r})} k(x_{D_\eta^\inv(u)}, x_{D_\eta^\inv(v)}) \dd u \dd v \\
                                             &= \frac{1}{r} \sum_{j\in[r]} E_\k\left((D_\eta^\inv)_\# \unif\left[\frac{j}{r},\frac{j+1}{r}\right)\right)\\
                                             &\le E_\k\left((D_\eta^\inv)_\# \unif[0,1)\right) = \bE\k(R,R'),
  \end{talign}
  where in the last inequality we applied Jensen's inequality since $E_\k$ is convex by \cref{lem:Ik_cvx}.
  Hence we have shown
  \begin{talign}
    \bE[\mmd^2_{\k}(\bS_n^{w^\trunc},\bS_n^w)] - \bE[\mmd^2_{\k}(\bS_n^{w^\truncstrat},\bS_n^w)] \ge 0
  \end{talign}
  and the proof is complete.
\end{proof}

\section{Accelerated Debiased Compression}\label{sec:app_sub_quad_time}
In this section, we provide supplementary algorithmic details and deferred analyses for \LSKT (\cref{alg:LSKT}).
In \rpc (\cref{alg:rpc}), we provide details for the weighted extension of \citet[Alg.~2.1]{chen2022randomly} that is used extensively in our algorithms.
The details of \agm \citep[Alg.~14]{wang2023no} are provided in \cref{alg:agm}.
In \cref{subsec:LD_analysis}, we give the proof of \cref{thm:LD_guarantee} for the MMD error guarantee of \alrbc (\cref{alg:alrbc}).
In \cref{subsec:compresspp}, we provide details on \compresspptarget modified from \compresspp \citep{shetty2022distribution} to minimize MMD to $\bP$.
Finally, \cref{thm:LSKT_guarantee} is proved in \cref{sec:LSKT_guarantee_proof}.

\begin{algorithm}[htb]
  \caption{Weighted Randomly Pivoted Cholesky (\rpcnotag)  (extension of \citet[Alg.~2.1]{chen2022randomly})}
  \label{alg:rpc}
  \begin{algorithmic}\itemindent=-.7pc
    \STATE {\bf Input:} kernel $\k$, points $\cS_n=(x_i)_{i=1}^n$, simplex weights $w\in\simplex$, rank $r$
    \STATE $\tilde\k(i,j) \defeq \k(x_i,x_j)\sqrt{w_i}\sqrt{w_j}$ \COMMENT{reweighted kernel matrix function} %
    \STATE $F \gets \bm{0}_{n\times r}, \tS \gets \{\}, d \gets (\tilde \k(i,i))_{i\in[n]}$
    \FOR{$i = 1$ {\bf to } $r$}
    \STATE Sample $s \sim d/\sum_{j\in[n]} d_j$
    \STATE $\tS \gets \tS \cup \{s\}$
    \STATE $g \gets \tilde \k(:, s) - F(:, 1:i-1)F(s, 1:i-1)^\top$
    \STATE $F(:,i) \gets g/\sqrt{g_s}$
    \STATE $d \gets d - F(:, i)^2$ \COMMENT{$F(:, i)^2$ denotes a vector with entry-wise squared values of $F(:, i)$}
    \STATE $d \gets \max(d, 0)$ \COMMENT{numerical stability fix, helpful in practice}
    \ENDFOR
    \STATE $F \gets \diag((1/\sqrt{w_i})_{i\in[n]}) F$ \COMMENT{undo weighting; treat $1/\sqrt{w_i} = 0$ if $w_i = 0$}
    \STATE {\bf Return:} $\tS \subset [n]$ with $\abs{\tS} = r$ and $F \in \bR^{n\times r}$
  \end{algorithmic}
\end{algorithm}

\begin{algorithm}[htb]
  \caption{Accelerated Entropic Mirror Descent (\agmnotag) (modification of \citet[Alg.~14]{wang2023no})}
  \label{alg:agm}
  \begin{algorithmic}\itemindent=-.7pc
    \STATE {\bf Input:} kernel matrix $K \in \bR^{n\times n}$, number of steps $T$, initial weight $w_0\in \simplex$, aggressive flag \texttt{AGG}
    \STATE $\eta \gets \frac{1}{8 w_0^\top \diag(K)}$ if \texttt{AGG} else $\frac{1}{8\max_{i\in[n]} K_{ii}}$
    \STATE $v_0 \gets w_0$
    \FOR{$t=1$ {\bf to } $T$}
    \STATE $\beta_t \gets \frac{2}{t+1}$ 
    \STATE $z_t \gets (1-\beta_t)w_{t-1} + \beta_t v_{t-1}$
    \STATE $g \gets 2t\eta  Kz_t$
    \COMMENT{this is $\gamma_t\nabla f(z_t)$ in \citet[Alg.~14]{wang2023no} for $f(w) = w^\top Kw$}
    \STATE $v_t \gets v_{t-1} \cdot \exp(-g)$ \COMMENT{component-wise exponentiation and multiplication}
    \STATE $v_t \gets v_t / \norm{v_t}_1$ \COMMENT{$v_t = \argmin_{w\in\simplex} \langle g, w\rangle + D_{v_{t-1}}^\phi(w)$ for $\phi(w) = \sum_{i\in[n]}w_i\log w_i$}
    \STATE $w_t \gets (1-\beta_t)w_{t-1}+\beta_t v_t$
    \ENDFOR
    \STATE {\bf Return:} $w_T \in \simplex$
  \end{algorithmic}
\end{algorithm}

\subsection{\pcref{thm:LD_guarantee}}\label{subsec:LD_analysis}
We start with a useful lemma that bounds $w^\top(K-\hat K)w$ by $\tr(K-\hat K)$ for any simplex weights $w$.
\begin{lemma}\label{lem:quad_form_tr_bound}
    For any PSD matrix $A \in \bR^{n\times n}$ and $w\in\simplex$, we have
    \begin{talign}
        w^\top A w \le \tr(A^w)\le \max_{i\in[n]}A_{ii} \le \lambda_1(A),
    \end{talign}
    where $\lambda_1(A)$ denotes the largest eigenvalue of $A$.
\end{lemma}
\begin{proof}[Proof of \cref{lem:quad_form_tr_bound}]
  Note that
  \begin{talign}
    w^\top A w &= \sqrt{w}^\top \diag(\sqrt{w}) A \diag(\sqrt{w}) \sqrt{w} = \sqrt{w}^\top A^w \sqrt{w}.
  \end{talign}
  The condition that $w \in \simplex$ implies $\twonorm{\sqrt{w}}=1$, so that
  \begin{talign}
    \sqrt{w}^\top A^w \sqrt{w} \le \lambda_1(A^w) \le \tr(A^w).
  \end{talign}
  To see $\tr(A^w) \le \max{i\in[n]} A_{ii}$, note that $\tr(A^w) = \sum_{i\in [n]} A_{ii}w_{i} \le \max_{i\in[n]} A_{ii}$  since $w \in \simplex$..

  Since $\lambda_1(A) = \sup_{x: \twonorm{x}=1} x^\top A x$, if we let $i^* \defeq \argmin_{i\in[n]} A_{ii}$, then the simplex weight with $1$ on the $i^*$-th entry has two-norm $1$, so we see that $\max_{i\in[n]} A_{ii} \le \lambda_1(A)$.
\end{proof}

Our next lemma bounds the suboptimality of surrogate optimization of a low-rank plus diagonal approximation of $K$.
\begin{lemma}[Suboptimality of surrogate optimization]\label{lem:sub_quad_req}
  Let $\kp$ be a kernel satisfying \cref{assum:mean_zero_p}.
  Let $\cS_n =(x_1,\ldots,x_n) \subset \bR^d$ be a sequence of points. 
  Define $K \defeq \kp(\cS_n, \cS_n) \in \bR^{n\times n}$.
  Suppose $\widehat{K} \in \bR^{n\times n}$ is another PSD matrix such that $K \succeq \widehat{K}$.
  Define $D \defeq \diag(K - \widehat K)$, the diagonal part of $K - \widehat K$, and form $K' \defeq \widehat K + D$.
  Let $w' \in \argmin_{w\in\simplex} w'^\top K' w'$.
  Then for any $w \in \simplex$,
  \begin{talign}
    \mmd_{\kp}^2(\bS_n^w, \bP) \le \mmd_{\kp}^2(\bS_n^{\wopt}, \bP) + \tr((K - \widehat K)^w) + \max_{i\in[n]} (K-\widehat K)_{ii} + (w^\top K'w - w'^\top K' w').
    \label{eqn:sub_quad_req}
  \end{talign}
\end{lemma}
\begin{proof}[Proof of \cref{lem:sub_quad_req}]
    Since $K = K' + (K-\widehat K) - D$ by construction, we have
    \begin{talign}
        w^\top K w &= w^\top K' w + w^\top (K-\widehat K)w - w^\top D w \\
        &\le w^\top K' w + w^\top (K-\widehat K)w \\
        &= (w^\top K' w - w'^\top K' w') + w'^\top K' w' + w^\top (K-\widehat K)w \\
        &\le (w^\top K' w - w'^\top K' w') + w'^\top K' w' + \tr((K-\widehat K)^w),
    \end{talign}
    where we used the fact that $D \succeq 0$ and \cref{lem:quad_form_tr_bound}.
    Next, by the definition of $w'$, we have
    \begin{talign}
        w'^\top K' w' &\le (\wopt)^\top K' \wopt = (\wopt)^\top (K'-K) \wopt + (\wopt)^\top K \wopt \\\
    &=(\wopt)^\top (D - (K-\widehat K)) \wopt + (\wopt)^\top K \wopt \\
    &\le (\wopt)^\top D\wopt + (\wopt)^\top K \wopt \\
    &\le \max_{i\in[n]}(K-\widehat K)_{ii} +(\wopt)^\top K \wopt, 
    \end{talign}
    where we used the fact $K \succeq \widehat K$ in the penultimate step and \cref{lem:quad_form_tr_bound} in the last step.
  Hence we have shown our claim.
\end{proof}

\cref{lem:sub_quad_req} shows that to control $\mmd_{\kp}^2(\bS_n^w, \bP)$, it suffices to separately control the approximation error in terms of $\tr(K - \hat K)$ and the optimization error $(w^\top K'w - w'^\top K' w')$.
The next result establishes that using \rpc, we can obtain polynomial and exponential decay bounds for $\tr(K-\hat K)$ in expectation depending on the kernel growth of $\kp$.
\begin{proposition}[Approximation error of \rpc]
  \label{lem:rpc_approx_error}
  Let $\k$ be a kernel satisfying \cref{assum:kernel_growth}.
  Let  $\cS_\infty$ be an infinite sequence of points in $\bR^d$.
  For any $w \in \simplex$, let $F$ be the low-rank approximation factor output by $\rpc(\k, \cS_n, w, r)$.
  Define $K \defeq \k(\cS_n,\cS_n)$.
  If $r \ge (\frac{\cvrC R_n^\beta + 1}{\sqrt{\log 2}}+\sqrt{\log 2})^2 - \frac{1}{\log 2}$, then, with the expectation taken over the randomness in \rpc,
      \begin{talign}
        \bE\left[\tr\left((K - F F^\top)^w\right)\right] \le H_{n, r},
        \label{eqn:rpc_Hnr_bound}
      \end{talign}
      where $H_{n,r}$ is defined as
  \begin{talign}
    H_{n,r}\defeq
    \begin{cases}      8\sum_{\ell=\rpcq(r)}^n (\frac{\krcd_\k(R_n)}{\ell})^{\frac{2}{\alpha}}&\polygrowth(\cvrw,\cvrd), \\ 
     8\sum_{\ell=\rpcq(r)}^n \exp(1-(\frac{\ell}{\krcd_\k(R_n)})^{\frac{1}{\alpha}}) & \loggrowth(\cvrw,\cvrd),
    \end{cases}
  \label{eqn:lrbc_h}
      \end{talign}
      for $\krcd_\k$ defined in \eqref{eqn:krcd} and
\begin{talign}
    \rpcq(r) \defeq \floor{\sqrt{\frac{r+\frac{1}{\log 2}}{\log 2}} - \frac{1}{\log 2}}.
    \label{eqn:rpcq}
\end{talign}
    Moreover, $H_{n,r}$ satisfies the  bounds in \cref{thm:LD_guarantee}.
\end{proposition}

\begin{proof}[Proof of \cref{lem:rpc_approx_error}]
Recall the notation $\krcd_{\k}(R_n) = \frac{\cvrC R_n^\beta}{\log 2}$ from \eqref{eqn:krcd}.
Define $q \defeq \rpcq(r)$ so that $q$ is the biggest integer for which $r \ge 2q + q^2\log  2$.
The lower bound assumption of $r$ is chosen such that $q > \krcd_\k(R_n) > 0$. 
By \citet[Theorem 3.1]{chen2022randomly} with $\eps = 1$, we have
\begin{talign}
  \bE\left[\tr\left((K - F F^\top)^w\right)\right] \le 2\sum_{\ell = q+1}^n \lambda_\ell(K^{w}).
\label{eqn:rpc_approx_eigen_tail}
\end{talign}
Since $q > \krcd_k(R_n)$, we can apply \cref{cor:Kp_eigval_bound} to bound $\lambda_\ell(K^{w})$ for $\ell \ge q + 1$ and obtain \eqref{eqn:rpc_Hnr_bound} since $H_{n,r}$ \eqref{eqn:lrbc_h} is constructed to match the bounds when applying \cref{cor:Kp_eigval_bound} to \eqref{eqn:rpc_approx_eigen_tail}.
It remains to justify the bounds for $H_{n,r}$ in \cref{thm:LD_guarantee}.
  
  If $\k$ is $\polygrowth(\cvrw, \cvrd)$, by \cref{assum:kernel_growth} we have $\cvrw < 2$.
  Hence
  \begin{talign}
    H_{n,r}&= 8\sum_{\ell = q}^n \left(\frac{\krcd_\k(R_n)}{\ell}\right)^{\frac{2}{\cvrw}} 
                                         \le 8\krcd_\k(R_n)^{\frac{2}{\cvrw}} \int_{q-1}^\infty \ell^{-\frac{2}{\cvrw}} \dd \ell 
                                         = 8\krcd_\k(R_n)^{\frac{2}{\cvrw}} (q-1)^{1-\frac{2}{\cvrw}} 
                                         =O\left(\sqrt{r} (\frac{R_n^{2\beta}}{r})^{\frac{1}{\alpha}}\right),
  \end{talign}
  where we used the fact that $\int_{q-1}^\infty \ell^{-\frac{2}{\alpha}} \dd\ell = (q-1)^{1-\frac{2}{\alpha}}$ 
  for $\alpha < 2$, $\krcd_\k(R_n) = O(R_n^\beta)$, and $q = \Theta(\sqrt{r})$.

  If $\k$ is $\loggrowth(\cvrw, \cvrd)$, then
  \begin{talign}
    H_{n,r} &= 8 \sum_{\ell = q}^n \exp(1-\left(\frac{\ell}{\krcd_\k(R_n)}\right)^{\frac{1}{\cvrw}}) = 8e \sum_{\ell=q}^n c^{\ell^{1/\alpha}} \le 8e \int_{\ell=q-1}^\infty c^{\ell^{1/\alpha}},\label{eqn:rpc_approx_log_slack}
  \end{talign}
    where $c \defeq \exp(-\krcd_\k(R_n)^{-1/\alpha}) \in (0, 1)$.
    Defining $m \defeq -\log c > 0$ and $q'=q-1$, we have
    \begin{talign}
        \int_{x=q'}^\infty c^{x^{1/\alpha}}\dd x =\int_{x=q'}^\infty \exp(-mx^{1/\alpha})\dd x &= \alpha q'(mq'^{1/\alpha})^{-\alpha} \Gamma(\alpha, mq'^{1/\alpha})= \alpha m^{-\alpha} \Gamma(\alpha, mq'^{1/\alpha}),
        \label{eqn:rpc_approx_log_series}
    \end{talign}
    where $\Gamma(\alpha, x) \defeq \int_x^\infty t^{\alpha-1} e^{-t}\dd t$ is the incomplete gamma function.
    Since $\alpha > 0$, by \citet[Thm.~1.1]{pinelis2020exact}, we have
    \begin{talign}
    \Gamma(\alpha, mq'^{1/\alpha}) &\le \frac{(mq'^{1/\alpha} + b)^\alpha - (mq'^{1/\alpha})^\alpha}{\alpha b} e^{-mq'^{1/\alpha}},
    \end{talign}
    where $b$ is a known constant depending only on $\alpha$.
    By the equivalence of norms on $\bR^2$, there exists $C_\alpha > 0$ such that $(x+y)^\alpha \le C_\alpha(x^\alpha + y^\alpha)$ for any $x,y > 0$.
    Hence
    \begin{talign}
    \Gamma(\alpha, mq'^{1/\alpha}) &\le \frac{(mq'^{1/\alpha} + b)^\alpha}{\alpha b} e^{-mq'^{1/\alpha}} \le \frac{C_\alpha(m^\alpha q' + b^\alpha)}{\alpha b} e^{-mq'^{1/\alpha}}.
    \end{talign}
    Hence from \eqref{eqn:rpc_approx_log_series} we deduce 
    \begin{talign}
        \sum_{\ell=q'}^\infty c^{\ell^{1/\alpha}} &\le C_\alpha(q'b^{-1} + b^{\alpha-1}m^{-\alpha}) e^{-mq'^{1/\alpha}}.\label{eqn:rpc_approx_log_final}
    \end{talign}
    Since $m = -\log c  = \krcd_\k(R_n)^{-1/\alpha}$, we can bound the exponent by
    \begin{talign}
        -mq'^{1/\alpha} = -(L_k(R_n)^{-1} q')^{1/\alpha} = -(\frac{q'\log 2}{\cvrC R_n^\beta})^{1/\alpha} \le -(\frac{0.83\sqrt{r}-2.39}{\cvrC R_n^\beta})^{1/\alpha},
    \end{talign}
    where we used the fact that $q'\log 2 = (q-1)\log 2 \ge (\sqrt{\frac{r+\frac{1}{\log 2}}{\log 2}}- \frac{1}{\log 2} - 2)\log 2 \ge 0.83\sqrt{r}-2.39$.
    On the other hand, since $q'=q-1 \ge \krcd(R_n) = m^{-\alpha}$, we can absorb the $b^{\alpha-1}m^{-1}$ term in \eqref{eqn:rpc_approx_log_final} into $q$ and finally obtain the bounds for $H_{n,r}$ in \cref{thm:LD_guarantee}.
\end{proof}

The last piece of our analysis involves bounding the optimization error $(w^\top K'w - w'^\top K' w')$ in \eqref{eqn:sub_quad_req}.
\begin{lemma}[\agm guarantee for debiasing]\label{lem:AGM_guarantee}
  Let $K\in\bR^{n\times n}$ be an SPSD matrix. Let $f(w) \defeq w^\top K w$. Then the final iterate $x_T$ of Nesterov's 1-memory method \citep[Algorithm 14]{wang2023no} after $T$ steps with objective function $f(w)$, norm $\norm{\cdot} = \norm{\cdot}_1$,  distance-generating function $\phi(x) = \sum_{i=1}^n x_i\log x_i$, and initial point $w_0 = (\frac{1}{n}, \ldots, \frac{1}{n}) \in \simplex$ satisfies
  \begin{talign}
    f(w_T) - f(\wopt) \le \frac{16\log n\max_{i\in[n]}K_{ii}}{T^2},
  \end{talign}
  where $\wopt \in \argmin_{x\in\bR^n} f(x)$.
\end{lemma}
\begin{proof}[Proof of \cref{lem:AGM_guarantee}]
  We apply \citet[Theorem 14]{wang2023no}. Hence it remains to determine the smoothness constant $L > 0$ such that, for all $x, y \in \simplex$,
  \begin{talign}
    \norm{\nabla f(x) - \nabla f(y)}_\infty \le L \norm{x-y}_1,
  \end{talign}
  and an upper bound for the Bregman divergence $D^\phi_{w_0}(\wopt) = \sum_{i=1}^n \wopt_i \log\frac{\wopt_i}{(w_0)_i}= \sum_{i=1}^n \wopt_i \log n \wopt_i$.
  To determine $L$, note $\nabla f(w) = 2K w$, so we have, for any $x, y \in \simplex$,
  \begin{talign}
    \norm{\nabla f(x) - \nabla f(y)}_\infty &= 2\norm{K (x-y)}_\infty = 2 \max_{i\in [n]} \abs{K_{i, :} (x-y)} \\
                                            &\le 2 \max_{i \in [n]} \norm{K_{i, :}}_\infty \norm{x-y}_1 = 2\left(\max_{i \in [n]} K_{ii}\right) \norm{x-y}_1 = 2\left(\max_{i \in [n]} K_{ii}\right) \norm{x-y}_1,
  \end{talign}
  where we used the fact that the largest entry in an SPSD matrix appears on its diagonal.
  Thus we can take the smoothness constant to be
  \begin{talign}
    L = 2\max_{i\in[n]} K_{ii}.
  \end{talign}

  To bound $D^\phi_{w_0}({\wopt})$, note that by Jensen's inequality,
  \begin{talign}
    D^\phi_{w_0}({w}) = \sum_{i=1}^n w_i \log n w_i \le \log(\sum_{i=1}^n n w^2_i) = \log n + \log \norm{w}^2_2 \le \log n,
  \end{talign}
  where we used the fact that $\norm{w}_2^2 \le \norm{w}_1 = 1$ for $w \in \simplex$.
\end{proof}

With these tools in hand, we turn to the proof of \cref{thm:LD_guarantee}.
For the runtime of \alrbc, it follows from the fact that \rpc takes $O((\kev+r)nr)$ time and one step of \agm takes $O(nr)$ time.

The error analysis is different for the first adaptive iteration and the ensuing adaptive iterations.
Roughly speaking, we will show that the first adaptive iteration brings the MMD gap to the desired level, while the ensuing iterations do not introduce an excessive amount of error.

\para{Step 1. Bound $\Delta\mmd_{\kp}(w^{(1)})$}

  Let $K \defeq \kp(\cS_n,\cS_n)$ and $F$ denote the low-rank approximation factor generated by \rpc.
  Denote $\widehat K \defeq FF^\top$.
 Then $K' = \widehat K  + \diag(K - \widehat K)$.
  First, note that since $w^{(0)} = (\frac{1}{n},\ldots,\frac{1}{n})$, \stratresamp returns $\tilde w = w^{(0)}$ with probability one.
  By \cref{lem:sub_quad_req}, we have, using $\sqrt{a+b}\le\sqrt{a}+\sqrt{b}$ for $a,b \ge 0$ repeatedly and \cref{lem:quad_form_tr_bound} that $\tr((K-\widehat K)^w) \le \lambda_1(K-\widehat K)$ and $\max_{i\in[n]}(K-\widehat K)_{ii} \le \lambda_1(K-\widehat K)$,
  \begin{talign}
    \mmd_{\kp}(\bS_n^{w^{(1)}}, \bP) &\le \mmd_{\kp}(\bS_n^{\wopt}, \bP) + \sqrt{2\lambda_1(K-\widehat K)} + \sqrt{{w^{(1)}}^\top K' w^{(1)} - w'^\top K' w'} \\
                             &\le \mmd_{\kp}(\bS_n^{\wopt}, \bP) + \sqrt{2\lambda_1(K-\widehat K)} + \sqrt{\frac{16 \log n \kpnormSn}{T^2}},\label{eqn:lrbc_as_bound}
  \end{talign}
  where we applied \cref{lem:quad_form_tr_bound} and \cref{lem:AGM_guarantee} in the last inequality.
  Fix $\delta \in (0, 1)$.
  By Markov's inequality, we have
  \begin{talign}
    \Pr(\sqrt{\lambda_1(K-\widehat K)} > \sqrt{\frac{\bE\left[\lambda_1(K - \widehat K)\right]}{\delta}}) \le \delta.
  \end{talign}
  This means that with probability at least $1 - \delta$, we have 
  \begin{talign}
    \mmd_{\kp}(\bS_n^{w^{(1)}}, \bP) &\le \mmd_{\kp}(\bS_n^{\wopt}, \bP) + \sqrt{\frac{2\bE\left[\lambda_1(K - \widehat K)\right]}{\delta}} + \sqrt{\frac{16 \log n \kpnormSn}{T^2}}.
  \end{talign}
  Note that the lower bound condition on $r$ in \cref{assum:lr} implies the lower bound condition in \cref{lem:rpc_approx_error}.
  Hence, by \cref{lem:rpc_approx_error} with $w = (\frac{1}{n},\ldots,\frac{1}{n})$ and using the identity $\lambda_1(K-\widehat K) \le \tr(K-\widehat K)$ while noting that a factor of $n$ appears, we have
  \begin{talign}
    \mmd_{\kp}(\bS_n^{w^{(1)}}, \bP)&\le \mmd_{\kp}(\bS_n^{\wopt}, \bP) + \sqrt{\frac{2n H_{n,r}}{\delta}} + \sqrt{\frac{16  \kpnormSn\log n}{T^2}}.
  \end{talign}

\para{Step 2. Bound the error of the remaining iterations}

  Fix $\delta > 0$.
  The previous step shows that, with probability at least $1-\frac{\delta}{2}$,
  \begin{talign}
    \mmd_{\kp}(\bS_n^{w^{(1)}}, \bP) \le \mmd_{\kp}(\bS_n^{\wopt}, \bP) + \sqrt{\frac{4n H_{n,r}}{\delta}} + \sqrt{\frac{16\kpnormSn \log n}{T^2}}.
  \end{talign}
  Fix $q > 1$, and let $\tilde w$ be the resampled weight defined in the $q$-th iteration in \cref{alg:alrbc}.
  Without loss of generality, we assume $\tilde w_i > 0$ for all $i > 0$, since if $w_i = 0$ then index $i$ is irrelevant for the rest of the algorithm.
  Thus, thanks to \stratresamp, we have $\tilde w_i \ge 1/n$ for all $i\in[n]$.
  Let $a/b$ denote the entry-wise division between two vectors.
  As in the previous step of the proof, we let $K \defeq \kp(\cS_n,\cS_n)$, $F$ be the low-rank factor output by $\rpc(\kp, \cS_n, \tilde w, r)$, and $\widehat K = FF^\top$.
  For any $w \in \simplex$, recall the notation $K^w \defeq \diag(\sqrt w)K\diag(\sqrt w)$.
  Then we have
  \begin{talign}
  w^\top K w &= (w/\sqrt{\tilde w})^\top \diag(K^{\tilde w}) (w/\sqrt{\tilde w}) \\
  &= (w/\sqrt{\tilde w})^\top (\diag(\sqrt{\tilde w}) \widehat K \diag(\sqrt{\tilde w})) +  \diag(\sqrt{\tilde w}) (K-\widehat K) \diag(\sqrt{\tilde w})))(w/\sqrt{\tilde w}) \\
  &= w^\top \widehat K w + (w/\sqrt{\tilde w})^\top (\diag(\sqrt{\tilde w}) (K-\widehat K) \diag(\sqrt{\tilde w})))(w/\sqrt{\tilde w}) \\
  &\le w^\top \widehat K w + \max_{i\in[n]}(1/\tilde w_i)\tr(\diag(\sqrt{\tilde w}) (K-\widehat K) \diag(\sqrt{\tilde w})) \\
  &\le  w^\top \widehat K w + n\tr((K-\widehat K)^{\tilde w} ).
  \label{eqn:alrbc_rpc_bound_w}
  \end{talign}
  Note that 
  \begin{talign}
      K' = \widehat K + \diag(K-\widehat K) = K + (\widehat K - K)
 + \diag(K-\widehat K).
 \end{talign}
 Since $K' \succeq \widehat K$, we have
 \begin{talign}
     {w^{(q)}}^\top \widehat K w^{(q)} \le {w^{(q)}}^\top K' w^{(q)} \le \tilde w^\top K' \tilde w,
     \label{eqn:alrbc_rpc_bound_wq}
 \end{talign}
  where the last inequality follows from the if conditioning at the end of \cref{alg:alrbc}.
  In addition,
  \begin{talign}
      \tilde w^\top K' \tilde w &= \tilde w^\top (K + (\widehat K - K)
 + \diag(K-\widehat K)) \tilde w \\
 &\le \tilde w^\top K \tilde w + \tilde w^\top \diag(K-\widehat K )\tilde w \\
 &= \tilde w^\top K \tilde w + \sqrt{\tilde w}^\top \diag((K-\widehat K)^{\tilde w})\sqrt{\tilde w} \\
 &\le w^\top K \tilde w + \tr((K-\widehat K)^{\tilde w}),
  \end{talign}
  where we used the fact that $K \succeq \widehat K$ and $\twonorm{\sqrt{\tilde w}}=1$.
  Plugging the previous inequality into \eqref{eqn:alrbc_rpc_bound_wq} and then into \eqref{eqn:alrbc_rpc_bound_w} with $w=w^{(q)}$, we get
  \begin{talign}
    {w^{(q)}}^\top K w^{(q)} &\le \tilde w^\top K \tilde w + (n+1)\tr((K-\widehat K)^{\tilde w}).
    \label{eqn:alrbc_one_iter_bound}
  \end{talign}
  Taking square-root on both sides using $\sqrt{a+b} \le \sqrt{a}+\sqrt{b}$ for $a,b\ge 0$ and the triangle inequality, we get
  \begin{talign}
    \mmd_{\kp}(\bS_n^{w^{(q)}}, \bP) &\le \mmd_{\kp}(\bS_n^{\tilde w}, \bP)+  \sqrt{(n+1)\tr((K-\widehat K)^{\tilde w})} \\
                                     &\le \mmd_{\kp}(\bS_n^{w^{(q-1)}}, \bP)+  \mmd_{\kp}(\bS_n^{w^{(q-1)}}, \bS_n^{\tilde w}) + \sqrt{(n+1)\tr((K-\widehat K)^{\tilde w})}.
  \end{talign}
  By Markov's inequality, we have
  \begin{talign}
    \Pr(\mmd_{\kp}(\bS_n^{w^{(q-1)}}, \bS_n^{\tilde w}) > \sqrt{\frac{4Q \bE\left[\mmd^2_{\kp}(\bS_n^{w^{(q-1)}}, \bS_n^{\tilde w})\right]}{\delta}}) &\le \frac{\delta}{4Q} \\
    \Pr(\sqrt{\tr((K-\widehat K)^{\tilde w})} > \sqrt{\frac{4Q\bE\left[\tr((K-\widehat K)^{\tilde w})\right]}{\delta}}) &\le \frac{\delta}{4Q}.
  \end{talign}
  By \cref{prop:mmd_resample}\ref{itm:mmd_stratified_resample}, we have
  \begin{talign}
    \bE\left[\mmd^2_{\kp}(\bS_n^{w^{(q-1)}}, \bS_n^{\tilde w})\right] = \bE\left[\bE\left[\mmd^2_{\kp}(\bS_n^{w^{(q-1)}}, \bS_n^{\tilde w})\middle\vert w^{(q-1)}\right]\right] \le \frac{\kpnormSn}{n}.
  \end{talign}
  Thus by the union bound, with probability at least $1 - \frac{\delta}{2Q}$, using \cref{lem:rpc_approx_error} (recall low-rank approximation $\widehat K$ is obtained using $\tilde w$), we have
  
  \begin{talign}
    \mmd_{\kp}(\bS_n^{w^{(q)}}, \bP) &\le \mmd_{\kp}(\bS_n^{w^{(q-1)}}, \bP) + \sqrt{\frac{4Q\kpnormSn}{n\delta}} +  \sqrt{\frac{4Q(n+1)H_{n,r}}{\delta}}.
    \label{eqn:alrbc_one_iter_mmd}
  \end{talign}
  Finally, applying union bound and summing up the bounds for $q=1,\ldots,Q$, we get, with probability at least $1-\delta$,
  \begin{talign}
      \Delta\mmd_{\kp}(w^{(q)}) \le \sqrt{\frac{2n H_{n,r}}{\delta}} + \sqrt{\frac{16  \kpnormSn\log n}{T^2}}+ (Q-1)\left(\sqrt{\frac{4Q\kpnormSn}{n\delta}} +  \sqrt{\frac{4Q(n+1)H_{n,r}}{\delta}}\right).
  \end{talign}
  This matches the stated asymptotic bound in \cref{thm:LD_guarantee}.
  \qed

\subsection{Thinning with \compresspptargetnotag}\label{subsec:compresspp}
For compression with target distribution $\bP$, we modify the original KT-\compresspp algorithm of \citep[Ex.~6]{shetty2022distribution}: in \halve and \thin of \compresspp, we use \ktsplit with kernel $\kp$ without \ktswap (so our version of \compresspp outputs $2^\fg$ coresets, each of size $\sqrt{n}$), followed by \ktswaptarget to obtain a size $\sqrt{n}$ coreset.
We call the resulting thinning algorithm \compresspptarget.
We show in \cref{lem:cpp_guarantee_sg,cor:cpp_mmd_guarantee} that \compresspptarget satisfies an MMD guarantee similar to that of quadratic-time kernel thinning.

\begin{algorithm}[htb]
  \caption{\compresspptarget (modified \citet[Alg.~2]{shetty2022distribution} to minimize MMD to $\bP$)}
  \label{alg:compresspp_target}
  \begin{algorithmic}\itemindent=-.7pc
    \STATE {\bf Input:} kernel $\kp$ with zero-mean under $\bP$, input points $\cS_n=(x_i)_{i\in[n]}$, multiplicity $n'$ with $n'\in 4^\bN$, weight $w \in \simplex \cap (\frac{\bN_0}{n'})^n$, thinning parameter $\fg$, failure probability $\delta$
    \STATE $\tS \gets $ index sequence where $k \in [n]$ appears $n'w_k$ times
    \STATE $(\tI^{(\ell)})_{\ell\in[2^\fg]} \gets \compresspp(\fg, \cS_n[\tS])$ \COMMENT{\citet[Ex.~6]{shetty2022distribution} with KT substituted with \ktsplit in \halve and \thin.}
    \STATE $\tI^{(\ell)} \gets \tS[\tI^{(\ell)}]$ for each $\ell \in [2^\fg]$
    \STATE $\tI \gets \ktswaptarget(\kp, \cS_n, (\tI^{(\ell)})_{\ell\in[2^\fg]})$ 
    \STATE $w_\uCpp \gets \text{simplex weights corresponding to } \tI$ \COMMENT{$w_i = \frac{\text{number of occurrences of } i \text{ in } \tI}{\abs{\tI}}$}
    \STATE {\bf Return:} $w_{\uCpp} \in \simplex \cap (\frac{\bN_0}{\sqrt{n}})^n$ \COMMENT{Hence $\norm{w_{\uCpp}}_0 \le \sqrt{n}$}
  \end{algorithmic}
\end{algorithm}

\begin{lemma}[Sub-gaussian guarantee for $\compresspp$]\label{lem:cpp_guarantee_sg}
  Let $\cS_n$ be a sequence of $n$ points with $n \in 4^\bN$.
  For any $\delta \in (0, 1)$ and integer $\fg \ge \lceil \log_2\log(n+1)+3.1 \rceil$, consider the $\compresspp$ algorithm \citep[Algorithm 2]{shetty2022distribution}  with thinning parameter $\fg$, halving algorithm $\halve^{(k)} \defeq \symmetrize\footnote{Any halving algorithm can be converted into an unbiased one by symmetrization, i.e., returning either the output half or its complement with equal probability \citep[Remark 3]{shetty2022distribution}.}(\ktsplit(\k, \cdot, 1, \frac{n_k^2}{4n 2^\fg(\fg + (\beta_n+1)2^\fg)}\delta))$ for an input of $n_k \defeq 2^{\fg + 1 + k}\sqrt{n}$ points and $\beta_n \defeq \log_2\left(\frac{n}{n_0}\right)$, and with thinning algorithm $\thin \defeq \ktsplit(\k, \cdot, \fg, \frac{\fg}{\fg + (\beta_n+1)2^\fg}\delta)$.
  Then this instantiation of \compresspp compresses $\cS_n$ to $2^\fg$ coresets $(\cS_\out^{(i)})_{i\in[2^\fg]}$ of $\sqrt{n}$ points each.
  Denote the signed measure $\phi^{(i)} \defeq \frac{1}{n}\sum_{x\in\cS_n}\delta_x - \frac{1}{\sqrt{n}}\sum_{x\in \cS^{(i)}_{\out}} \delta_x$.
  Then for each $i \in [2^\fg]$, on an event $\Eequi^{(i)}$ with $\Pr(\Eequi^{(i)}) \ge 1-\frac{\delta}{2}$, $\phi^{(i)} = \tilde \phi^{(i)}$ for a random signed measure $\tilde \phi^{(i)}$ such that, for any $\delta' \in (0, 1)$,
\begin{talign}
  \Pr\left(\norm{\tilde \phi^{(i)} \k}_{\rkhs[\k]} \ge  a'_n \left(1 + \sqrt{\log(\frac{1}{\delta'})}\right)\right) \le \delta',\label{eqn:cpp_guarantee_sg}
\end{talign}
where
\begin{talign}
  a'_n=\frac{4}{\sqrt{n}}\left(2 + \sqrt{\frac{8}{3}\knormSn\log(\frac{6\sqrt{n}(\fg + (\frac{\log_2 n}{2} - \fg)2^\fg)}{\delta})\log(4\cvrnum{\k}\left(\balleuc(R_{n}), n^{-1/2}\right))}\right).
\end{talign}
\end{lemma}
\begin{proof}[Proof of \cref{lem:cpp_guarantee_sg}]
  This proof is similar to the one for \citet[Ex.~6]{shetty2022distribution} but with explicit constant tracking and is self-contained, invoking only \citet[Thm.~4]{shetty2022distribution} which gives MMD guarantees for \compresspp given the sub-Gaussian parameters of \halve and \thin.

  Recall that $n_k$ is the number of input points for the halving subroutine at recursion level $k$ in $\compresspp$, and $\beta_n$ is the total number of recursion levels.
  Let $\cS_\uC$ denote the output of $\compress$ \citep[Alg.~1]{shetty2022distribution} of size $2^\fg \sqrt{n}$.
  Fix $\delta,\delta' \in (0, 1)$.
  Suppose we use $\halve^{(k)} \defeq \symmetrize(\ktsplit(\k, \cdot, 1, \gamma_{k}\delta))$ for an input of $n_k$ points for $\gamma_k$ to be determined.
  Suppose we use $\thin \defeq \ktsplit(\k, \cdot, \fg, \gamma'\delta)$ for $\gamma'$ to be determined; this is the kernel thinning stage that thins $2^{\fg}\sqrt{n}$ points to $2^\fg$ coresets, each with $\sqrt{n}$ points. Since the analysis is the same for all coresets, we will fix an arbitrary coreset without superscript in the notation.

  By \cref{lem:kt_mmd_guarantee_sg}, with notation $t \defeq \log\frac{1}{\delta'}$, there exist events $\cE_{k,j}$, $\cE_{\uT}$, and random signed measures $\phi_{k,j}$, $\tilde \phi_{k,j}$, $\phi_\uT$, $\tilde \phi_\uT$ for $0\le k \le \beta_n$ and $j\in[4^k]$ such that
  \begin{enumerate}[label=(\alph*)]
    \item
      $\Pr(\cE_{k,j}^c) \le \frac{\gamma_k\delta}{2}$ and $\Pr(\cE_\uT^c) \le \frac{\gamma'\delta}{2}$,
    \item
      $\ind{\cE_{k,j}}\phi_{k,j} = \ind{\cE_{k,j}}\tilde \phi_{k,j}$ and $\ind{\cE_\uT}\phi_\uT = \ind{\cE_\uT} \tilde \phi_\uT$,
    \item
      We have
      \begin{talign}
        &\Pr\left(\norm{\tilde \phi_{k,j} \k}_{\rkhs[\k]} \ge a_{n_k} + v_{n_k}\sqrt{t}\middle\vert \{\tilde \phi_{k',j'}\}_{k'>k,j'\ge 1}, \{\tilde \phi_{k',j'}\}_{k',j'<j}\right) \le e^{-t} \\
        &\Pr\left(\norm{\tilde \phi_\uT \k}_{\rkhs[\k]} \ge a'_n + v'_n\sqrt{t}\middle\vert S_\uC \right) \le e^{-t},
      \end{talign}
      where, by \cref{lem:kt_mmd_guarantee_sg}, and by increasing the sub-Gaussian constants if necessary, we have
      \begin{talign}
        a_{n_k} &\defeq v_{n_k} \defeq a_{n_k, n_k/2} = \frac{2}{n_k}\left(2 + \sqrt{\frac{8}{3}\knormSn\log(\frac{3n_k}{\gamma_k\delta})  \log (4\cvrnum{\k}\left(\balleuc(R_{n}), \frac{2}{n_k}\right))} \right),\\
        a'_n &\defeq v'_n \defeq a_{2^\fg\sqrt{n},\sqrt{n}} = \frac{1}{\sqrt{n}}\left(2 + \sqrt{\frac{8}{3}\knormSn\log(\frac{6\fg \sqrt{n}}{\gamma'\delta})\log(4\cvrnum{\k}\left(\balleuc(R_{n}), n^{-1/2}\right))}\right), 
        \qtext{and}
      \end{talign}
    \item $\bE\left[\tilde \phi_{k,j} \k\middle\vert \{\tilde \phi_{k',j'}\}_{k'>k,j'\ge 1}, \{\tilde \phi_{k',j'}\}_{k',j'<j}\right] = 0.$ 
  \end{enumerate}
  Hence on the event $\cE = \bigcap_{k, j} \cE_{k, j} \cap \cE_\uT$, these properties hold simultaneously.
  We will choose $\{\gamma_k\}_k$ and $\gamma'$ such that $\Pr(\cE^c) \le \frac{\delta}{2}$.
  By the union bound,
  \begin{talign}
    \Pr(\cE^c) \le \Pr(\cE^c_\uT) + \sum_{k=0}^{\beta_n} \sum_{j=1}^{4^k}\Pr(\cE_{k,j}^c) \le \frac{\gamma'\delta}{2} + \sum_{k=0}^{\beta_n} 4^k \frac{\gamma_k\delta}{2}.\label{eqn:loc:PrEc}
  \end{talign}
  On the event $\cE$, we apply \citet[Thm.~4, Rmk.~7]{shetty2022distribution} to get a sub-Gaussian guarantee for $\mmd_{\k}(\bS_n, \bS_{\out})$.
  We want to choose $\gamma_k$, $\gamma'$ such that the rescaled quantities $\tilde\zeta_\uH \defeq \frac{n_0}{2}a_{n_0}$ and $\tilde\zeta_\uT \defeq \sqrt{n}a'_n$ satisfy $\tilde\zeta_\uH = \tilde\zeta_\uT$ \citep[Eq.~(13)]{shetty2022distribution}, which implies that we need
  \begin{talign}
    \frac{3n_0}{\gamma_0 \delta} = \frac{6\fg\sqrt{n}}{\gamma'\delta} \Longleftrightarrow \frac{\gamma_0}{\gamma'} = \frac{2^{\fg}}{\fg}.
    \label{eqn:loc:gamma_ratio}
  \end{talign}
  Hence if we take $\gamma' = \frac{\fg}{\fg + (\beta_n+1)2^\fg}$ and $\gamma_k = \frac{n_k^2}{4n 2^\fg(\fg + (\beta_n+1)2^\fg)}$, then \eqref{eqn:loc:gamma_ratio} holds and the upper bound in \eqref{eqn:loc:PrEc} becomes $\frac{\delta}{2}$.
  Note that $n_k a_{n_k}$ is non-decreasing in $n_k$, so by \citet[Theorem 4, Remark 7]{shetty2022distribution}, $\compresspp(\delta, \fg)$ outputs a signed measure $\phi$ that, on the event $\cE$ with $\Pr(\cE^c) \le \frac{\delta}{2}$, equals another signed measure $\tilde \phi$ that satisfies, for any $\delta' \in (0, 1)$,
  \begin{talign}
    \Pr(\norm{\tilde\phi\k}_{\rkhs[\k]} \ge \hat a_n + \hat v_n \sqrt{\log(\frac{1}{\delta'})}) \le \delta',
  \end{talign}
  where $\hat a_n, \hat v_n$ satisfy $\max(\hat a_n, \hat v_n) \le 4a'_n$ whenever $\fg \ge \lceil \log_2 \log (n+1) + 3.1 \rceil$.
\end{proof}

\begin{corollary}[MMD guarantee for \compresspp]\label{cor:cpp_mmd_guarantee}
  Let $\cS_\infty$ be an infinite sequence of points in $\bR^d$ and $\k$ a kernel.
  For any $\delta \in (0,1)$ and $n \in \bN$ such that $n \in 4^\bN$, consider the \compresspp with the same parameters as in \cref{lem:cpp_guarantee_sg} with $\fg \ge \lceil \log_2\log(n+1)+3.1 \rceil$.
  Then for any $i \in [\sqrt{n}]$, with probability at least $1 - \delta$,%
\begin{talign}
  \mmd_{\k}(\bS_n, \bS^{(i)}_{\out}) \le  \frac{4}{\sqrt{n}}\left(2 + \sqrt{\frac{8}{3}\knormSn\log(\frac{6\sqrt{n}(\fg + (\frac{\log_2 n}{2} - \fg)2^\fg)}{\delta})\log(4\cvrnum{\k}\left(\balleuc(R_{n}), n^{-1/2}\right))}\right)\left(1+\sqrt{\log\frac{2}{\delta}}\right).
\end{talign}
\end{corollary}
\begin{proof}
  After applying \cref{lem:cpp_guarantee_sg} with $\delta'=\frac{\delta}{2}$ and following the same argument as in the proof of \cref{prop:kt_mmd_guarantee}, we have, with probability at least $1-\delta$, 
  \begin{talign}
      \mmd_{\k}(\bS_n, \bS^{(i)}_{\out}) \le a_n'\left(1+\sqrt{\log\frac{2}{\delta}}\right).
  \end{talign}
  Plugging in the expression of $a_n'$ from \cref{lem:cpp_guarantee_sg} gives the claimed bound.
\end{proof}

\subsection{\pcref{thm:LSKT_guarantee}}\label{sec:LSKT_guarantee_proof}
First of all, the claimed runtime follows from the runtime of \alrbc (\cref{thm:LD_guarantee}), the $O(\kev 4^{\fg}n\log n)$ runtime of \compresspp, and the $O(\kev n^{1.5})$ runtime of \ktswaptarget.

  Without loss of generality assume $n \in 4^\bN$. 
  Fix $\delta \in (0, 1)$.
  Let $w^\db$ denote the output of $\alrbc$, and $w^\truncstrat$ denote the output of \stratresamp, both regarded as random variables.
  By \cref{thm:LD_guarantee}, we have, with probability at least $1-\frac{\delta}{3}$,
  \begin{talign}
    \mmd_{\kp}(\bS_n^{w^\db}, \bP) = \mmd_{\kp}(\bS_n^{\wopt}, \bP) + 
    O\left(\sqrt{\frac{nH_{n,r}}{\delta}}\right) + O\left(\sqrt{\frac{\kpnormSn \max(\log n,1/\delta)}{n}}\right).\label{eqn:loc:mmd_w_dagger_P}
  \end{talign}
  By \cref{prop:mmd_resample}\ref{itm:mmd_stratified_resample} with $\k=\kp$, we have the upper bound
  \begin{talign}
    \bE\left[\mmd^2_{\kp}(\bS_n^{w^\truncstrat}, \bS_n^{w^{\db}})\right] = \bE\left[\bE\left[\mmd^2_{\kp}(\bS_n^{w^\truncstrat}, \bS_n^{w^{\db}})\middle\vert w^\db\right]\right] \le \frac{\kpnormSn}{n}.
  \end{talign}
  Thus, by Markov's inequality,
  \begin{talign}
    \Pr(\mmd_{\kp}(\bS_n^{w^\truncstrat}, \bS_n^{w^{\db}}) \ge \sqrt{\frac{3\kpnormSn}{n\delta}}) \le \frac{\delta}{3}.
  \end{talign}
Hence, with probability at least $1-\frac{\delta}{3}$, we have
  \begin{talign}
    \mmd_{\kp}(\bS_n^{w^\truncstrat}, \bS_n^{w^{\db}}) \le \sqrt{\frac{3\kpnormSn}{n\delta}}.\label{eqn:loc:mmd_w_ts_w_dagger}
  \end{talign}
  Let $\bS_\out^{(i)}$ denote the $i$-th coreset output by \thin in \compresspptarget (\cref{alg:compresspp_target}). By \cref{cor:cpp_mmd_guarantee} with $\k=\kp$, we have, with probability at least $1-\frac{\delta}{3}$,
\begin{talign}
  \mmd_{\kp}(\bS^{w^\truncstrat}_n, \bS^{(i)}_{\out}) =  O\left(\sqrt{\frac{\kpnormSn \log n\log (e\cvrnum{\kp}(\balleuc(R_{n}), n^{-1/2})))}{n}} \log \frac{e}{\delta}\right).
\end{talign}
Since \ktswaptarget can never increase $\mmd_{\kp}(\cdot,\bP)$, we have, by the triangle inequality,
\begin{talign}
  \mmd_{\kp}(\bS_n^{w_{\LSKTnotag}},\bP) \le \mmd_{\kp}(\bS_\out^{(1)},\bP) \le \mmd_{\kp}(\bS_\out^{(1)},\bS_n^{w^\truncstrat}) + \mmd_{\kp}(\bS_n^{w^\truncstrat}, \bS_n^{w^\db}) + \mmd_{\kp}(\bS_n^{w^\db}, \bP).\label{eqn:loc:mmd_w_ts_S_out}
\end{talign}
By the union bound, with probability at least $1-\delta$, the bounds \eqref{eqn:loc:mmd_w_dagger_P}, \eqref{eqn:loc:mmd_w_ts_w_dagger}, \eqref{eqn:loc:mmd_w_ts_S_out} hold, so that the claim is shown by adding together the right-hand sides of these bounds and applying \cref{assum:kernel_growth}.
\qed

\section{Simplex-Weighted Debiased Compression}\label{sec:app_cvx}
In this section, we provide deferred analyses for \recomb and \GSR/\LSR, as well as the algorithmic details of \recombbfs (\cref{alg:recombbfs}) and \ktswapls (\cref{alg:kt_swap_ls}).
\begin{algorithm}[!htb]
  \caption{\recombbfsnotag (rephrasing of \citet[Alg. 1]{tchernychova2016caratheodory} that takes $O(nm+m^3\log n)$ time)}
  \label{alg:recombbfs}
  \begin{algorithmic}\itemindent=-.7pc
    \STATE {\bf Input:} matrix $A\in\bR^{m\times n}$ with $m < n$ and one row of $A$ all positive, a nonnegative vector $x_0\in\bR^{n}_{\ge 0}$.
    \FUNCTION{\textrm{FindBFS}($A, x_0$)}\itemindent=-.7pc 
    \STATE \COMMENT{The requirement of $A$ and $x_0$ are the same as the input. This subroutine takes $O(n^3)$ time.}
    \STATE $x \gets x_0$
    \STATE $U, S, V^\top \gets \textrm{SVD}(A)$ \COMMENT{any $O(n^3)$-time SVD algorithm that gives $U S V^\top = A$}
    \STATE $V \gets (V^\top)_{m+1:n}$ \COMMENT{$V\in \bR^{(n-m)\times n}$ so that the null space of $A$ is spanned by the rows of $V$}
    \FOR{$i=1$ {\bf to } $n-m$}\itemindent=-.7pc
    \STATE $v \gets V_i$
    \STATE $k \gets \argmin_{j\in[n]: v_j > 0} \frac{x_j}{v_j}$ \COMMENT{This must succeed because $Av=0$ and $A$ has an all-positive row, so one of the coordinates of $v$ must be positive.}
    \STATE $x \gets x - \frac{x_k}{v_k} v$ \COMMENT{This zeros out the $k$-th coordinate of $x$ while still ensuring $x$ is nonnegative.}
    \FOR{$j=i+1$ {\bf to } $n-m$}\itemindent=-.7pc
    \STATE $V_j \gets V_j - \frac{V_{j,k}}{v_k} v$ \COMMENT{$\{V_j\}_{j={i+1}}^{n-m}$ remain independent and have $0$ on the $k$-th coordinate.}
    \ENDFOR
    \ENDFOR
    \STATE {\bf return: } $x \in \bR^n_{\ge 0}$ such that $Ax=Ax_0$ and $\norm{x}_0\le m$. 
    \ENDFUNCTION
    \STATE $x \gets x_0$
    \WHILE{$\norm{x}_0 > 2m$}\itemindent=-.7pc
    \STATE Divide $\{i\in [n]: x_i > 0\}$ into $2m$ index blocks $\tI_{1},\ldots,\tI_{2m}$, each of size at most $\floor{\frac{\norm{x}_0}{2m}}$.
    \STATE $A_i \gets A_{:, \tI_i} x_{\tI_i} \in \bR^m, \forall i \in [2m]$ \COMMENT{Since $\norm{x}_0$ is halved in each iteration, overall this step takes $O(nm)$ time.}
    \STATE Form $\hat A$ to be the $m\times 2m$ matrix with columns $A_i$  \COMMENT{Hence, one row of $A$ contains all positive entries.}
    \STATE $\hat x \gets \textrm{FindFBS}(\hat A, \bm 1_{2m})$ \COMMENT{$\norm{\hat x}_0 \le n$ and $\hat A \hat x = \sum A_i \hat x_i = \sum A_i = \sum A_{:,\tI_i}x_{\tI_i} = Ax$.}
    \FOR{$i = 1$ \bf{ to } $2m$}\itemindent=-.7pc
    \STATE $x_{\tI_i} \gets \hat x_i \cdot x_{\tI_i}$ \bf{ if } $\hat x_i > 0$ \bf{ else } $0$
    \ENDFOR
    \STATE \COMMENT{After the update,  $\norm{x}_0$ is halved while the equality $Ax=Ax_0$ is maintained.}
    \ENDWHILE
    \IF{$\norm{x}_0 \ge m+1$}\itemindent=-.7pc
    \STATE $\tI \gets \{i \in [n]: x_i > 0\}$
    \STATE $x_{\tI} = \textrm{FindBFS}(A_{:,\tI}, x_{\tI})$
    \ENDIF
    \STATE {\bf Return:} $x \in \bR^n_{\ge 0}$ such that $Ax = Ax_0$ and $\norm{x}_0 \le m$. 
  \end{algorithmic}
\end{algorithm}
\begin{algorithm}[!htb]
  \caption{KT-Swap with Linear Search (\ktswapls)}
  \label{alg:kt_swap_ls}
  \begin{algorithmic}\itemindent=-.7pc
    \STATE {\bf Input:} kernel $\kp$ with zero-mean under $\bP$, input points $\cS_n=(x_i)_{i\in[n]}$, weights $w \in \simplex$, $\texttt{fmt} \in \{\textup{SPLX}, \textup{CP}\}$
    \STATE $\tS \gets \{i\in[n]: w_i \neq 0\}$
    \STATE 
    \COMMENT{Maintain two sufficient statistics: $g = Kw$ and $D = w^\top K w$.}
    \FUNCTION{\textrm{Add}($g,D,i,t$)}\itemindent=-.7pc
    \STATE $g \gets g + t \kp(\cS_n, x_i)$
    \STATE $D \gets D + 2t g_i + t^2 \kp(x_i, x_i)$
    \STATE {\bf return: } $(g,D)$
    \ENDFUNCTION
    \FUNCTION{\textrm{Scale}($g,D,\alpha$)}\itemindent=-.7pc
    \STATE $g \gets \alpha g$
    \STATE $D \gets \alpha^2 D$
    \STATE {\bf return: } $(g,D)$
    \ENDFUNCTION

    \STATE $\texttt{Kdiag} \gets \kp(\cS_n,\cS_n)$
    \STATE $g \gets \bm{0} \in \bR^n$
    \STATE $D \gets 0$
    \FOR{$i$ {\bf in } $\tS$}\itemindent=-.7pc
    \STATE $(g,D) \gets \textrm{Add}(g, D, i, w_i)$
    \ENDFOR
    \FOR{$i$ {\bf in } $\tS$}\itemindent=-.7pc
    \STATE {\bf if } $w_i = 1$ {\bf then continue}; \COMMENT{We cannot swap $i$ out if $\sum_{j\neq i} w_j = 0$!}
    \STATE \COMMENT{First zero out $w_i$.}
    \STATE $(g, D) \gets \textrm{Add}(g,D,i,-w_i)$
    \STATE $(g, D) \gets \textrm{Scale}(g, D, \frac{1}{1-w_i})$
    \STATE $w_i = 0$
    \STATE \COMMENT{Next perform line search to add back a point.}
    \STATE $\alpha = (D-g) ./ (D-2g+\texttt{Kdiag})$; \COMMENT{$\alpha_i = \argmin_t \mmd^2_{\kp}(\bS_n^{te_i+(1-t)w}, \bP) = \argmin_t (1-t)^2 D + 2t(1-t) g + t^2 \texttt{Kdiag}$}
    \IF{$\texttt{fmt} = \textup{SPLX}$}
    \STATE $\alpha = \texttt{clip}(\alpha, 0, 1)$; \COMMENT{Clipping $\alpha$ to $[0,1]$. This corresponds to $\argmin_{t\in[0,1]} \mmd^2_{\kp}(\bS_n^{te_i+(1-t)w}, \bP)$.}
    \ENDIF
    \STATE $D' \gets (1-\alpha)^2 D + 2\alpha(1-\alpha)g + \alpha^2\texttt{Kdiag}$  \COMMENT{multiplications are element-wise}
    \STATE $k \gets \argmin_i D'_i$
    \STATE $(g,D) \gets \textrm{Scale}(g, D, 1-\alpha_k)$
    \STATE $(g,D) \gets \textrm{Add}(g, D, k, \alpha_k)$
    \ENDFOR

    \STATE {\bf Return:} $w \in \simplex$ 
  \end{algorithmic}
\end{algorithm}

\subsection{MMD guarantee for \recomb}
We start by stating the MMD guarantee of \recomb, a result that might be of independent interest.
\begin{proposition}[\recomb guarantee]\label{prop:recomb_guarantee}
Under \cref{assum:mean_zero_p,assum:kernel_growth}, given $w \in \simplex$ and that $m \ge (\frac{\cvrC R_n^\cvrd + 1}{\sqrt{\log 2}} + \sqrt{\log 2})^2-\frac{1}{\log 2} + 1$, \recombfull (\cref{alg:recomb}) outputs  $w_{\recombnotag} \in \simplex$ with $\norm{w_{\recombnotag}}_0 \le m$ in $O((\kev+m)nm+m^3\log n)$ time such that
 with probability at least $1 - \delta$, 
  \begin{talign}
    \mmd_{\kp}(\bS_n^{w_{\recombnotag}}, \bP) \le \mmd_{\kp}(\bS_n^{w}, \bP) + \sqrt{\frac{2\kpnormSn}{n\delta}} +  \sqrt{\frac{2n H_{n,{m-1}}}{\delta}},
    \label{eqn:recomb_guarantee}
  \end{talign}
  where $H_{n,r}$ is defined in \eqref{eqn:lrbc_h}.
\end{proposition}
\begin{proof}[Proof of \cref{prop:recomb_guarantee}]
    The runtime follows from the $O((\kev+m)nm)$ runtime of \rpc, the $O(\kev nm)$ runtime of \ktswapls, and the $O(nm+m^3\log n)$ runtime of \recombbfs \citep{tchernychova2016caratheodory} which dominates the $O(m^3)$ weight optimization step. %

    Recall $w' \in \simplex$ from \recomb.
  The formation of $F$ in \cref{alg:recomb} is identical to the formation of $F$ (with $r=m-1$) in \cref{alg:alrbc} for $q > 1$.
  Thus by \eqref{eqn:alrbc_rpc_bound_w} with $w=w'$, $K = \kp(\cS_n, \cS_n)$,
  \begin{talign}
    {w'}^\top K w' \le {w'}^\top F F^\top w' + n\tr((K - FF^\top)^{\tilde w}),
  \end{talign}
  where $K = \kp(\cS_n, \cS_n)$.
  By construction of $w'$ using \recombbfs, we have $F^\top\tilde w = F^\top w'$.
  Since $K \succeq F F^\top$, we have
  \begin{talign}
    {w'}^\top K w' &\le {\tilde w}^\top F F^\top \tilde w + n\tr((K - FF^\top)^{\tilde w}) \le {\tilde w}^\top K \tilde w + n\tr((K - FF^\top)^{\tilde w}). 
  \end{talign}
  We recognize the right-hand side is precisely the right-hand side of \eqref{eqn:alrbc_one_iter_bound} aside from having a multiplier of $n$ instead of $n+1$ in front of the trace and that $F$ is rank $m-1$.
  Now applying \eqref{eqn:alrbc_one_iter_mmd} with $Q=\frac{1}{2}$, $w^{(q)} = w'$, $w^{(q-1)} = w$, $r=m-1$, and noticing that \ktswapls and the quadratic-programming solve at the end cannot decrease the objective, we obtain \eqref{eqn:recomb_guarantee} with probability at least $1 - \delta$.
  Note that the lower bound of $m$ in \cref{assum:lr} makes $r=m-1$ satisfy the lower bound for $r$ in \cref{lem:rpc_approx_error}.
\end{proof}

\subsection{\pcref{thm:SR_guarantee}}\label{sec:thm:SR_guarantee_proof}
    The claimed runtime follows from the runtime of \GBCfull (\cref{alg:GBC}) or \alrbc (\cref{thm:LD_guarantee}) plus the runtime of \recomb (\cref{prop:recomb_guarantee}).

    Note the lower bound for $m$ in \cref{assum:lr} implies the lower bound condition in \cref{prop:recomb_guarantee}.
  For the case of \GSR, we proceed as in the proof of \cref{thm:GSKT_guarantee} and use \cref{prop:recomb_guarantee}.
  For the case of \LSR, we proceed as in the proof of \cref{thm:LSKT_guarantee} and use \cref{thm:LD_guarantee} and \cref{prop:recomb_guarantee}.
  \qed

\section{Constant-Preserving Debiased Compression}
\label{sec:app_cp}
In this section, we provide deferred analyses for \cholthin and \GSC/\LSC.

\subsection{MMD guarantee for \cholthin}
\begin{proposition}[\cholthin guarantee]\label{prop:cholthin_guarantee}
Under \cref{assum:mean_zero_p,assum:kernel_growth}, given $w \in \simplex$ and $m \ge (\frac{\cvrC R_n^\cvrd + 1}{\sqrt{\log 2}} + \frac{2}{\sqrt{\log 2}})^2 - \frac{1}{\log 2}$, \cholthin outputs  $w_{\cholthinnotag}\in\reals^n$ with $\bm{1}_n^\top w_{\cholthinnotag} = 1$ and $\norm{w_{\cholthinnotag}}_0 \le m$ in $O((\kev+m)nm + m^3)$ time such that, for any $\delta \in (0, 1)$, with probability $1-\delta$,
  \begin{talign}
\mmd_{\kp}(\bS_n^{w_{\cholthinnotag}}, \bP) \le 2\mmd_{\kp}(\bS_n^{w}, \bP)  + \sqrt{\frac{4H_{n,m'}}{\delta}},
\end{talign}
where $H_{n, m}$ is defined in \eqref{eqn:lrbc_h} and $m' \defeq m + \log 2 - 2\sqrt{m\log 2 + 1}$.
\end{proposition}
\begin{proof}[Proof of \cref{prop:cholthin_guarantee}]
The runtime follows from the $O((\kev+m)nm)$ runtime of \rpc, the $O(nm)$ runtime of \ktswapls, and the $O(m^3)$ runtime of matrix inversion in solving the two minimization problems using \eqref{eqn:cp_kkt_opt}.

To improve the clarity of notation, we will use $w^\db$ to denote the input weight $w$ to \cholthin.
For index sequences $\tI, \tJ \subset [n]$ and a kernel $\k$, we use $\k(\tI, \tJ)$ to indicate the matrix $\k(\cS_n[\tI],\cS_n[\tJ])=[\k(x_{i}, x_{j})]_{i\in\tI, j\in\tJ}$, and similarly for a function $f:\bR^n\to\bR$, we use $f(\tI)$ to denote the vector $(f(x_i))_{i\in\tI}$.

Recall the regularized kernel is $\kc \defeq \kp + c$.
Suppose for now that $c > 0$ is an arbitrary constant.
Let $\tI$ denote the indices output by \rpc in \cholthin.
Let
\begin{talign}
  w^c \defeq \argmin_{w: \supp(w) \subset \tI} \mmd^2_{\kc}(\bS_n^w, \bS_n^{w^\db}).
\end{talign}
Note that $w^c$ is not a probability vector and may not sum to $1$.

\para{Step 1. Bound $\mmd^2_{\kc}(\bS_n^{w^c}, \bS_n^{w^\db})$ in terms of \rpcnotag approximation error}

We start by using an argument similar to that of \citet[Prop. 3]{epperly2024kernel} to exploit the optimality condition of $w^c$.
Since
\begin{talign}
  \argmin_{w:\supp(w)\subset \tI}\mmd^2_{\kc}(\bS_n^w, \bS_n^{w^\db}) &= \argmin_{w:\supp(w)\subset\tI} w_\tI^\top \kc(\tI, \tI) w_\tI - 2 {w^{\db}}^\top \kc(\cS_n, \tI) w_{\tI},
\end{talign}
by optimality, $w^c$ satisfies, 
\begin{talign}
  \kc(\tI, \tI) w^c_\tI = \bS_n^{w^\db} \kc(\tI).
\end{talign}
We comment that the index sequence $\tI$ returned by \rpc makes $\kc(\tI,\tI)$ invertible with probability $1$: by the Guttman rank additivity formula of Schur complement \citep[Eq.~(6.0.4)]{zhang2006schur}, each iteration of \rpc chooses a pivot with a non-zero diagonal and thus increases the rank of the low-rank approximation matrix, which is spanned by the columns of pivots, by $1$. 
Hence
\begin{talign}
  \bS_n^{w^c}\kc(\cdot) &= \kc(\cdot, \cS_n)w^c 
                        = \kc(\cdot, \tI)w^c_\tI 
                        = \kc(\cdot, \tI)\kc(\tI, \tI)^{-1} \kc(\tI, \tI)w^c_\tI \\
                        &= \kc(\cdot, \tI)\kc(\tI, \tI)^{-1} \bS_n^{w^\db}\kc(\tI) 
                        = \bS_n^{w^\db}\kc_\tI(\cdot),
\end{talign}
where $\kc_\tI(x,y) \defeq \kc(x, \tI)\kc(\tI, \tI)^{-1}\kc(\tI, y)$.
Then
\begin{talign}
  \mmd^2_{\kc}(\bS_n^{w^c}, \bS_n^{w^\db}) &= \norm{\bS_n^{w^\db}\kc - \bS_n^{w^c}\kc}^2_{\kc} 
  = \norm{\bS_n^{w^\db}\kc - \bS_n^{w^\db}\kc_\tI}^2_{\kc} 
  = {w^\db}^\top (\kc - \kc_{\tI})(\cS_n,\cS_n) w^\db.
\end{talign}
Recall the index set $\tI$ consists of the $m$ pivots selected by \rpc on the input matrix
\begin{talign}
  K^\db_c \defeq \kc(\cS_n, \cS_n)^{w^\db}.
\end{talign}
Define 
\begin{talign}
  \widehat{K}^\db_c \defeq \kc_\tI(\cS_n,\cS_n)^{w^\db}.
\end{talign}
Thus, by \cref{lem:quad_form_tr_bound},
\begin{talign}
  \mmd^2_{\kc}(\bS_n^{w^c}, \bS_n^{w^\db}) &= {w^\db}^\top (\kc - \kc_{\tI})(\cS_n,\cS_n) w^\db 
  = \sqrt{w^\db}^\top (K^\db_c -\widehat K^\db_c) \sqrt{w^\db} 
  \le \lambda_1 (K^\db_c - \widehat K^\db_c) \le \tr(K^\db_c - \widehat K^\db_c).
\end{talign}

\para{Step 2. Bound $\tr(K^\db_c - \widehat K^\db_c)$ using the trace bound of the unregularized kernel}

Let $\bestapprox{A}_r$ denote the best rank-$r$ approximation of an SPSD matrix $A \in \bR^{n\times n}$ in the sense that
\begin{talign}
  \bestapprox{A}_r &\defeq \argmin_{\substack{X\in\bR^{n\times n}\\X =X^\top\\ A\succeq X \succeq 0 \\ \rank(X) \le r}} \tr(A - X). \label{eqn:best_rank_r_approx}
\end{talign}
By the Eckart-Young-Mirsky theorem applied to symmetric matrices \citep[Theorem 19]{dax2014low}, the solution to \eqref{eqn:best_rank_r_approx} is given by $r$-truncated eigenvalue decomposition of $A$, so that 
\begin{talign}
 \tr(A - \bestapprox{A}_r) = \sum_{\ell=r+1}^n \lambda_\ell(A).
\end{talign}
Let $q \defeq \rpcq(m)$ where $\rpcq$ is defined in \eqref{eqn:rpcq}, so that by \citet[Thm. 3.1]{chen2022randomly} with $\eps=1$, we have
\begin{talign}
\bE\left[\tr(K_c^\db - \widehat{K}_c^\db)\right] \le 2\tr(K_c^\db - \bestapprox{K_c^\db}_{q}).
\end{talign}
We know one specific rank-$q$ approximation of $K_c^\db$:
\begin{talign}
  \widetilde K_c^\db\defeq \bestapprox{K^\db}_{q-1} +  \diag(\sqrt{w^\db})c\bm 1_n \bm 1_n^\top \diag(\sqrt{w^\db}),
\end{talign}
which satisfies
\begin{talign}
  K^\db_c - \widetilde K_c^\db &= K^\db + \diag(\sqrt{w^\db})c\bm 1_n \bm 1_n^\top \diag(\sqrt{w^\db}) - \widetilde K_c^\db = K^\db - \bestapprox{K^\db}_{q-1}.
\end{talign}
Thus by the variational definition in \eqref{eqn:best_rank_r_approx}, we have
\begin{talign}
  \tr(K^\db_c - \bestapprox{K^\db_c}_{q}) \le \tr(K^\db_c - \widetilde K^\db_c) = \tr(K^\db - \bestapprox{K^\db}_{q-1}) = \sum_{\ell=q}^n \lambda_\ell(K^\db).
\end{talign}
Note the last bound does not depend on $c$.
The tail sum of eigenvalues in the last expression is the same (up to a constant multiplier) as the one in \eqref{eqn:rpc_approx_eigen_tail} except for an off-by-1 difference in the summation index.
A simple calculation shows that for $m' \defeq m + \log 2 - 2\sqrt{m\log 2 + 1}$, we have $\rpcq(m') = \rpcq(m)-1$.
Another simple calculation shows that $m \ge (\frac{\cvrC R_n^\cvrd + 1}{\sqrt{\log 2}} + \frac{2}{\sqrt{\log 2}})^2 - \frac{1}{\log 2}$ implies that $m'$ satisfies the lower bound requirement of $r$ in \cref{lem:rpc_approx_error}.
Thus, arguing as in the proof that follows \eqref{eqn:rpc_approx_eigen_tail}, we get
\begin{talign}
\bE\left[\tr(K_c^\db - \widehat K_c^{\db})\right] \le H_{n,m'}.
\end{talign}
Thus so far we have shown
\begin{talign}
\bE[\mmd_{\kc}^2(\bS_n^{w^c}, \bS_n^{w^\db})] \le
\bE\left[\tr(K_c^\db - \widehat{K}_c^\db)\right] \le H_{n,m'}.
\end{talign}
By Markov's inequality, with probability at least $1 - \delta$, we have
\begin{talign}
\mmd_{\kc}(\bS_n^{w^c}, \bS_n^{w^\db}) \le \sqrt{\frac{H_{n,m'}}{\delta}}.
\label{eqn:loc:CT_proof_H_bound}
\end{talign}
Recall that $\mmd_{\k}(\mu,\nu) = \norm{(\mu-\nu)\k}_{\k}$ for signed measures $\mu, \nu$.
By the triangle inequality, we have
\begin{talign}
\mmd_{\kc}(\bS_n^{w^c}, \bP) &\le \mmd_{\kc}(\bS_n^{w^c}, \bS_n^{w^\db}) + \mmd_{\kc}(\bS_n^{w^\db}, \bP) \\
&= \mmd_{\kc}(\bS_n^{w^c}, \bS_n^{w^\db}) + \mmd_{\kp}(\bS_n^{w^\db}, \bP),
\end{talign}
where we used that fact that $\sum_{i\in [n]}w^\db_i=1$ to get the identity $\mmd_{\kc}(\bS^{w^\db}, \bP) = \mmd_{\kp}(\bS^{w^\db}, \bP)$.
Hence, with probability at least $1 - \delta$,
\begin{talign}
\mmd_{\kc}(\bS_n^{w^c}, \bP) \le \mmd_{\kp}(\bS_n^{w^\db}, \bP) + \sqrt{\frac{H_{n,m'}}{\delta}}.
\label{eqn:wc_to_wdb_bound}
\end{talign}

\para{Step 3. Incorporating sum-to-one constraint}

We now turn $w^c$ into a constant-preserving weight while not inflating the MMD by much. Define
\begin{talign}
  w^1 \defeq \argmin_{w: \supp(w) \subset \tI, \sum_{i\in[n]} w_i=1} \mmd^2_{\kp}(\bS_n^{w}, \bP).
  \label{eqn:cp_sum_1_opt}
\end{talign}
Note $w^1$ is the weight right before \ktswapls step in \cholthin.
Let $K_\tI = \kp(\tI,\tI)$.
Let $\bm 1_\tI$ denote the $\abs{\tI}$-dimensional all-one veector.
The Karush-Kuhm-Tucker condition \citep[Sec.~4.7]{ghojogh2021kkt} applied to \eqref{eqn:cp_sum_1_opt} implies that, the solution $w^1$ is a stationary point of the Lagrangian function 
\begin{talign}
    L(w_{\tI}, \lambda) \defeq w_{\tI}^\top K_{\tI} w_{\tI} + \lambda (\bm 1_{\tI}^\top w_{\tI} - 1).
\end{talign}
Then $\nabla_{w_{\tI}}L(w_\tI^1, \lambda) = 0$ implies $2K_{\tI}w^1_{\tI} - \lambda \bm 1_{\tI} = 0$, so $w^1_{\tI} = \frac{\lambda K_\tI^{-1} \bm 1_\tI}{2}$.
The Lagrangian multiplier $\lambda$ is determined by the constraint $\bm 1_\tI^\top w_\tI = 1$, so we find
\begin{talign}
  w^1_\tI = \frac{K_\tI^{-1}\bm{1}_\tI}{\bm{1}_\tI^\top K_\tI^{-1}\bm{1}_\tI}.\label{eqn:cp_kkt_opt}
\end{talign}
Define
\begin{talign}
  w^{c,\bP} \defeq \argmin_{w: \supp(w) \subset \tI} \mmd^2_{\kc}(\bS_n^{w}, \bP).
\end{talign}
Since $w^{c,\bP}$ is optimized to minimize $\mmd_{\kc}$ to $\bP$ on the same support as $w^c$, we have
\begin{talign}
    \mmd_{\kc}(\bS_n^{w^{c,\bP}}, \bP) \le \mmd_{\kc}(\bS_n^{w^{c}}, \bP).
\end{talign}
The optimality condition for $w^{c,\bP}$ is 
\begin{talign}
  (K_\tI + c\bm{1}_\tI\bm{1}_\tI^\top) w - c \bm{1}_\tI = 0
  ,
\end{talign} 
and hence by the Sherman–Morrison formula,
\begin{talign}
  w^{c,\bP}_\tI = (K_\tI + c\bm 1_\tI \bm 1_\tI^\top)^{-1}c\bm{1}_\tI = \left(K_\tI^{-1} - \frac{cK^{-1}_\tI \bm 1_\tI \bm 1_\tI^\top K_\tI^{-1}}{1+c\bm 1_\tI^\top K_\tI^{-1}\bm 1_\tI}\right) c\bm 1_\tI = \frac{K_\tI^{-1}\bm 1_\tI}{1/c + \bm{1}_\tI^\top K_\tI^{-1}\bm{1}_\tI}.
\end{talign}
Let $\rho_c \defeq \frac{\bm{1}_\tI^\top K_\tI^{-1}\bm 1_\tI}{1/c + \bm{1}_\tI^\top K_\tI^{-1}\bm{1}_\tI}$, so that $w^{c,\bP}_\tI = \rho_c w^1_\tI$.
In particular, $w^1$ and $w^{c,\bP}$ are scalar multiples of one another.
To relate $\mmd_{\kp}(\bS_n^{w^1}, \bP)$ and $\mmd_{\kc}(\bS_n^{w^{c,\bP}}, \bP)$, note that
\begin{talign}
  \mmd^2_{\kp}(\bS_n^{w^1}, \bP) 
  &= {w^1_\tI}^\top K_\tI w^1_\tI 
  = \frac{{w^{c,\bP}_\tI}^\top K_\tI w^{c,\bP}_\tI}{\rho^2_c} 
  = \frac{{w^{c,\bP}_\tI}^\top (K_\tI+c\bm 1_\tI\bm 1_\tI^\top) w^{c,\bP}_\tI - c(\bm 1_\tI^\top w_\tI^c)^2}{\rho^2_c} \\
    &= \frac{\mmd_{\kc}^2(\bS_n^{w^{c,\bP}},\bP) + 2c\bm{1}_\tI^\top w_\tI^c -c  - c(\bm 1_\tI^\top w_\tI^c)^2}{\rho^2_c} = \frac{\mmd_{\kc}^2(\bS_n^{w^{c,\bP}},\bP) - c(\rho_c-1)^2}{\rho^2_c}.
\end{talign}

So far the argument does not depend on any particular choice of $c>0$. 
Let us now discuss how to choose $c$.
Note that
\begin{talign}
\bm 1_\tI^\top K_\tI^{-1} \bm 1_\tI &= m \frac{\bm 1_\tI}{\sqrt{m}}^\top K_\tI^{-1} \frac{\bm 1_\tI}{\sqrt{m}} 
\ge m \lambda_m(K_\tI^{-1}) \ge \frac{m}{\lambda_1(K_\tI)} \ge \frac{m}{\tr(K_\tI)} \ge \frac{m}{\sum_{i\in[m]} \diag(K)^{\downarrow}_i},
\end{talign}
where $\diag(K)^\downarrow$ denote the diagonal entries of $K=\kp(\cS_n,\cS_n)$ sorted in descending order.
Thus
\begin{talign}
    \rho_c &= \frac{1}{\frac{1}{c \bm 1^\top_\tI K_\tI^{-1} \bm 1_\tI} + 1} \ge \frac{1}{\frac{\sum_{i\in[m]} \diag(K)^{\downarrow}_i}{mc}+1}.
\end{talign}
Hence we can choose $c$ to make sure $\rho_c$ is bounded from below by a positive value.
Recall in \cholthin, we take
\begin{talign}
    c = \frac{\sum_{i\in[m]} \diag(K)^{\downarrow}_i}{m},
\end{talign}
so that $\rho_c \ge \frac{1}{2}$ and 
\begin{talign}
  \mmd^2_{\kp}(\bS_n^{w^1}, \bP) &= \frac{\mmd_{\kc}^2(\bS_n^{w^{c,\bP}},\bP) - c(\rho_c-1)^2}{\rho^2_c} \le 4\mmd_{\kc}^2(\bS_n^{w^{c,\bP}},\bP).
\end{talign}
Hence by \eqref{eqn:wc_to_wdb_bound} and the fact that \ktswapls and the final reweighting in \cholthin only improves MMD, we have, with probability at least $1 - \delta$,
\begin{talign}
\mmd_{\kp}(\bS_n^{w_{\cholthinnotag}}, \bP) \le \mmd_{\kp}(\bS_n^{w^{1}}, \bP) &\le 2\mmd_{\kc}(\bS_n^{w^{c,\bP}}, \bP) 
\le 2\mmd_{\kc}(\bS_n^{w^{c}}, \bP) 
\le  2\mmd_{\kp}(\bS_n^{w^\db}, \bP) + 2\sqrt{\frac{H_{n,m'}}{\delta}},
\end{talign}
where we use \cref{eqn:wc_to_wdb_bound} in the last inequality.
\end{proof}

\subsection{\pcref{thm:SC_guarantee}}\label{sec:thm:SC_guarantee_proof}
    The claimed runtime follows from the runtime of \GBCfull (\cref{alg:GBC}) or \alrbc (\cref{thm:LD_guarantee}) plus the runtime of \cholthin(\cref{prop:cholthin_guarantee}).

    Note the lower bound for $m$ in \cref{assum:lr} implies the lower bound condition in \cref{prop:cholthin_guarantee}.
  For the case of \GSC, we proceed as in the proof of \cref{thm:GSKT_guarantee} and use \cref{prop:cholthin_guarantee}.
  For the case of \LSC, we proceed as in the proof of \cref{thm:LSKT_guarantee} by invoking \cref{thm:LD_guarantee} and \cref{prop:cholthin_guarantee}.
\qed

\section{Implementation and Experimental Details}\label{sec:app_experiment}
In this section, we collect experimental details that were deferred from \cref{sec:experiments}.

\subsection{$O(d)$-time Stein kernel evaluation}
\label{subsec:O_d_ks_eval}
In this section, we show that for $\cS_n=(x_i)_{i\in[n]}$, each Stein kernel evaluation $\ksm(x_i,x_j)$ for a radially analytic base kernel (\cref{def:analytic_k}) can be done in $O(d)$ time after computing certain sufficient statistics in $O(nd^2+d^3)$ time.
Let $M \in \bR^{d\times d}$ be a positive definite preconditioning matrix for $\ksm$.
Let $L$ be the Cholesky decomposition of $M$ which can be done in $O(d^3)$ time so that $M=LL^\top$.
From the expression \eqref{eqn:ks_alt_form_analytic},  we can achieve $O(d)$ time evaluation if we can compute $\norm{x-y}_M^2$ and $M\nabla\log p(x)$ in $O(d)$ time.
For $M\nabla\log p(x)$, we can simply precompute $M\nabla\log p(x_i)$ for all $i\in[n]$.
For $\norm{x-y}_M^2$, we have
\begin{align}
  \norm{x-y}_M^2 =  (x-y)^\top M^{-1} (x-y) = (x-y)^\top (LL^\top)^{-1} (x-y) = \twonorm{L^{-1}(x-y)}^2.
\end{align}
Hence it suffices to precompute $L^{-1} x_i$ for all $i\in[n]$, and we can precompute the inverse $L^{-1}$ in $O(d^3)$ time. %

\subsection{Default parameters for algorithms}
For \alrbc, we always use $Q = 3$. 
To ensure that the  guarantees of \cref{lem:AGM_guarantee,thm:LD_guarantee} hold while achieving fast convergence in practice, we take the step size of \agm to be $1/(8\kpnormSn)$ in the first adaptive round and $1/(8 \sum_{i\in[n]} w_i^{(q-1)}\kp(x_i,x_i))$ in subsequent adaptive rounds. 
We use $T=7\sqrt{n_0}$ for \agm in all experiments.

We implemented our modified versions of \kttargetfull and \compresspptarget in JAX \citep{jax2018github}
so that certain subroutines can achieve a speedup using just-in-time compilation and the parallel computation power of GPUs.
For \compresspp, we use $\fg = 4$ in all experiments as in \citet{shetty2022distribution}.
For both \kttargetfull and \compresspptarget, we use choose $\delta = 1/2$ as in the \texttt{goodpoints} library.

Each experiment was run with a single NVIDIA RTX 6000 GPU and an AMD EPYC 7513 32-Core CPU.

\subsection{Correcting for burn-in details}\label{subsec:exp_goodwin_app}
We use the four MCMC chains provided by \citet{riabiz2022optimal} that include both the sample points and their scores.
The reference chain used to compute the energy distance is the same one used in \citet{riabiz2022optimal} for the energy distance and was kindly provided by the authors.

In \cref{tab:goodwin_runtime}, we collect the runtime for the burn-in correction experiments.

\cref{fig:goodwin_eq_app}, \cref{fig:goodwin_cvx_app}, \cref{fig:goodwin_cp_app}, display the results of the burn-in correction experiment of \cref{subsec:exp_goodwin} repeated with three other MCMC algorithms: MALA without preconditioning, random walk (RW), and adaptive random walk (ADA-RW).
The results of P-MALA from \cref{subsec:exp_goodwin}  are also included for completeness.
For all four chains, our methods reliably achieve better quality coresets when compared with the baseline methods.

\begin{figure}[h]
    \centering
    \includegraphics[width=\textwidth]{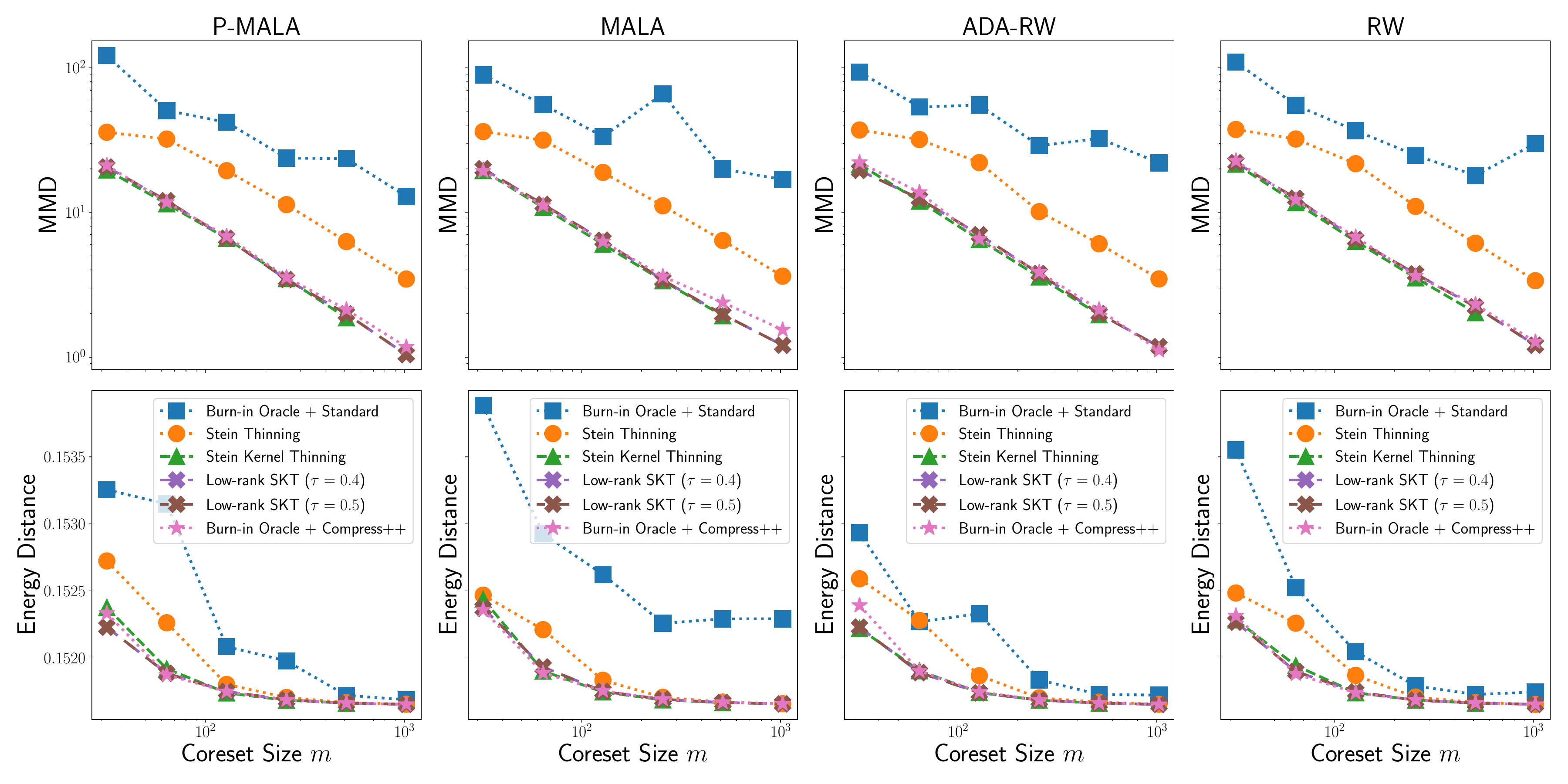}
    \caption{\textbf{Correcting for burn-in with equal-weighted compression.} For each of four MCMC algorithms and using only one chain, our methods consistently outperform the Stein and standard thinning baselines and match the 6-chain oracle.}
    \label{fig:goodwin_eq_app}
\end{figure}
\begin{figure}[h]
    \centering
    \includegraphics[width=\textwidth]{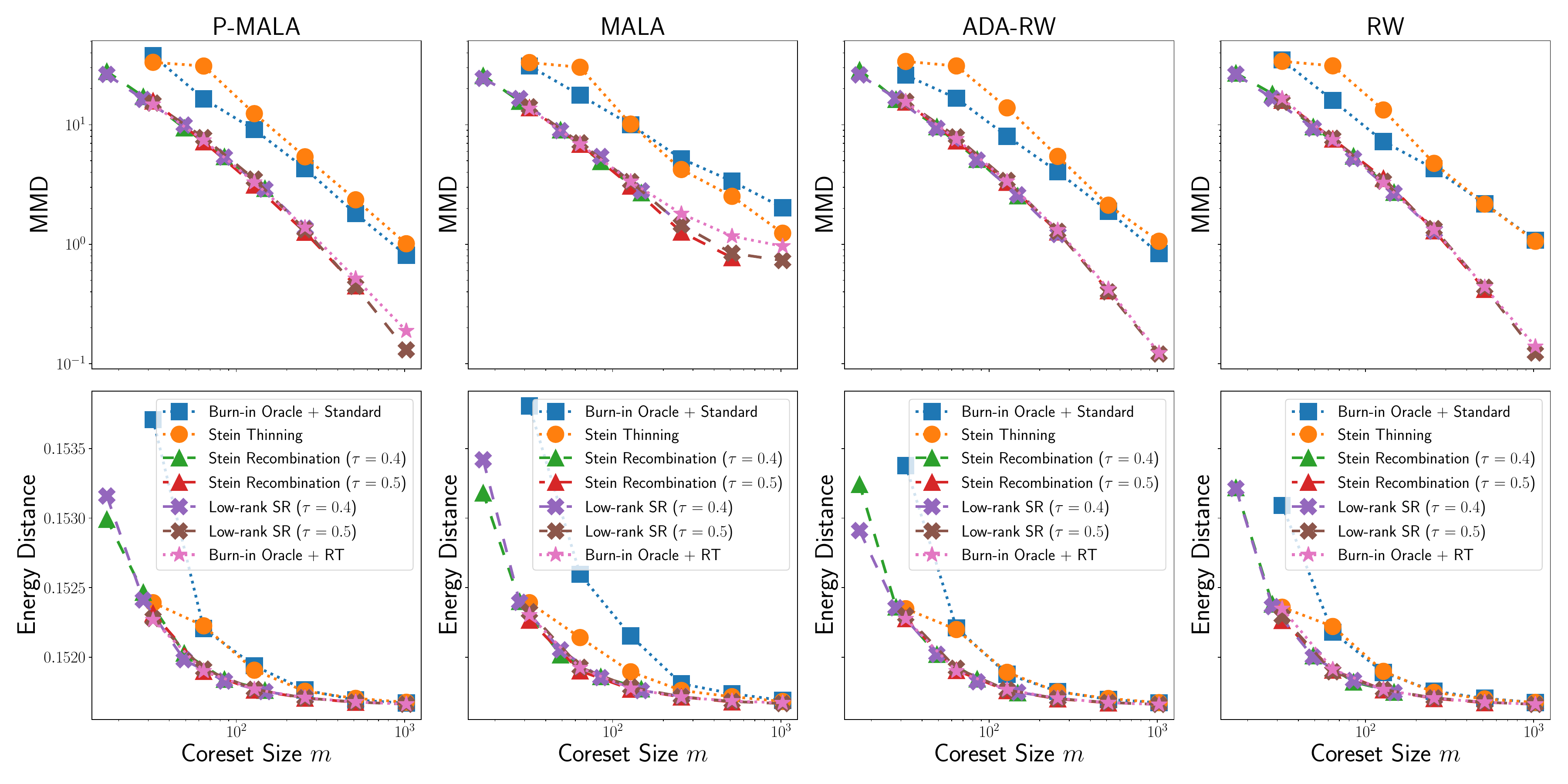}
    \caption{\textbf{Correcting for burn-in with simplex-weighted compression.} For each of four MCMC algorithms and using only one chain, our methods consistently outperform the Stein and standard thinning baselines and match the 6-chain oracle.}
    \label{fig:goodwin_cvx_app}
\end{figure}
\begin{figure}[h]
    \centering
    \includegraphics[width=\textwidth]{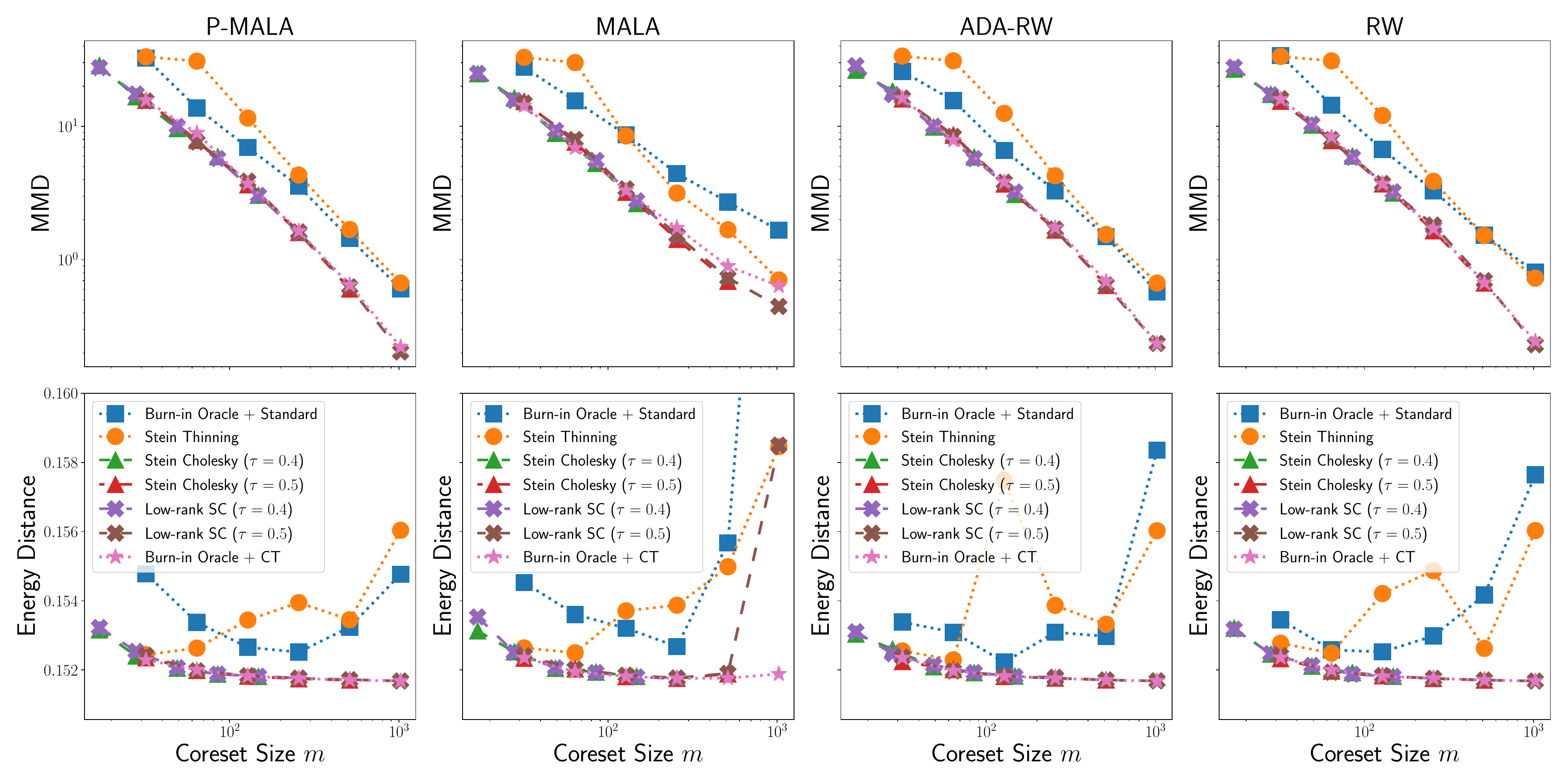}
    \caption{\textbf{Correcting for burn-in with constant-preserving compression.} For each of four MCMC algorithms and using only one chain, our methods consistently outperform the Stein and standard thinning baselines and match the 6-chain oracle.}
    \label{fig:goodwin_cp_app}
\end{figure}

\begin{table}[htbp]
  \centering
  \begin{tabular}{c|ccc|cccccc}
    \toprule
    $\mbi{n_0}$ & \textbf{\GBCnotag} & \textbf{\alrbcnotag ($0.5$)} & \textbf{\alrbcnotag ($0.4$)} & \textbf{KT} & \textbf{\kt-\compresspp} & \textbf{\recombnotag ($0.5$)} & \textbf{\recombnotag ($0.4$)} & \textbf{\cholthinnotag ($0.5$)} & \textbf{\cholthinnotag ($0.4$)} \\ \midrule
    $2^{14}$ & 2.50  & 13.22 & 12.88& 7.31 & 3.49 & 0.79 &0.60 & 2.06 &  1.96 \\ \midrule
    $2^{16}$ & 8.48 & 16.15 & 15.82 & 20.77 & 5.90 & 2.59 &1.68 & 3.66 &  3.04 \\ \midrule
    $2^{18}$ & 111.06 & 32.14& 20.60  & 193.03 & 11.73 & 11.16 &  2.63 & 6.48& 3.67  \\ \midrule
    $2^{20}$ & - & 314.67 & 131.31 & - & 35.99  & 113.71 & 11.06 & 51.14& 8.42  \\ \bottomrule
  \end{tabular}
\caption{\tbf{Breakdown of runtime (in seconds)} for the burn-in correction experiment $(d=4)$ of \cref{sec:experiments}. $n_0$ is the input size after standard thinning from the length $n=2\times 10^6$ chain (\cref{rmk:standard_thin_n_n0}). Each runtime is the median of 3 runs. 
KT and \kt-\compresspp output $m = \sqrt{n_0}$ equal-weighted points. \recombnotag and  \cholthinnotag respectively output $m=n_0^\tau$ points with simplex or constant-preserving weights for $\tau$ shown in parentheses. 
In addition, \alrbcnotag, \recombnotag, and  \cholthinnotag use the rank $n_0^\tau$.
\GBCnotag and KT took longer than 30 minutes for $n_0=2^{20}$ and hence their numbers are not reported.} 
\label{tab:goodwin_runtime}
\end{table}

\subsection{Correcting for approximate MCMC details}\label{subsec:exp_covtype_app}

\para{Surrogate ground truth}
Following \citet{liu2017black}, we took the first 10,000 data points and generated $2^{20}$ surrogate ground truth sample points using NUTS \citep{hoffman2014no} for the evaluation. To generate the surrogate ground truth using NUTS, we used \texttt{numpyro} \citep{phan2019composable}.
It took $12$ hours to generate the surrogate ground truth points using the GPU implementation, and we estimate it would have taken $200$ hours using the CPU implementation. %

\para{SGFS}
For SGFS, we used batch size 32 and the step size schedule $\eta/(1+t)^{0.55}$ where $t$ is the step count and $\eta$ is the initial step size.
We chose $\eta$ from $\{10.0, 5.0, 1.0, 0.5, 0.1, 0.05, 0.01\}$, found $\eta = 1.0$ gave the best standard thinning MMD to get a coreset size of $m=2^{10}$ 
, and hence we fixed $\eta = 1.0$ in all experiments.
We used the version of SGFS \citep[SGFS-f]{ahn2012bayesian} that involves inversion of $d\times d$ matrices --- we found the faster version (SGFS-d) that inverts only the diagonal resulted in significantly worse mixing.
We implemented SGFS in \texttt{numpy} and ran it on the CPU.

\para{Runtime} The SGFS  chain of length $2^{24}$ took approximately 2 hours to generate using the CPU.
Remarkably, all of our low-rank methods finish within 10 minutes for $n_0=2^{20}$, which is orders of magnitude faster than the time taken to generate the NUTS surrogate ground truth.

\para{Additional results}
In \cref{fig:covtype_mse}, we plot the posterior mean mean-squared error (MSE) for each compression method in the approximate MCMC experiment of \cref{subsec:exp_covtype}.

\begin{figure}[h]
    \centering
    \includegraphics[width=0.3\textwidth]{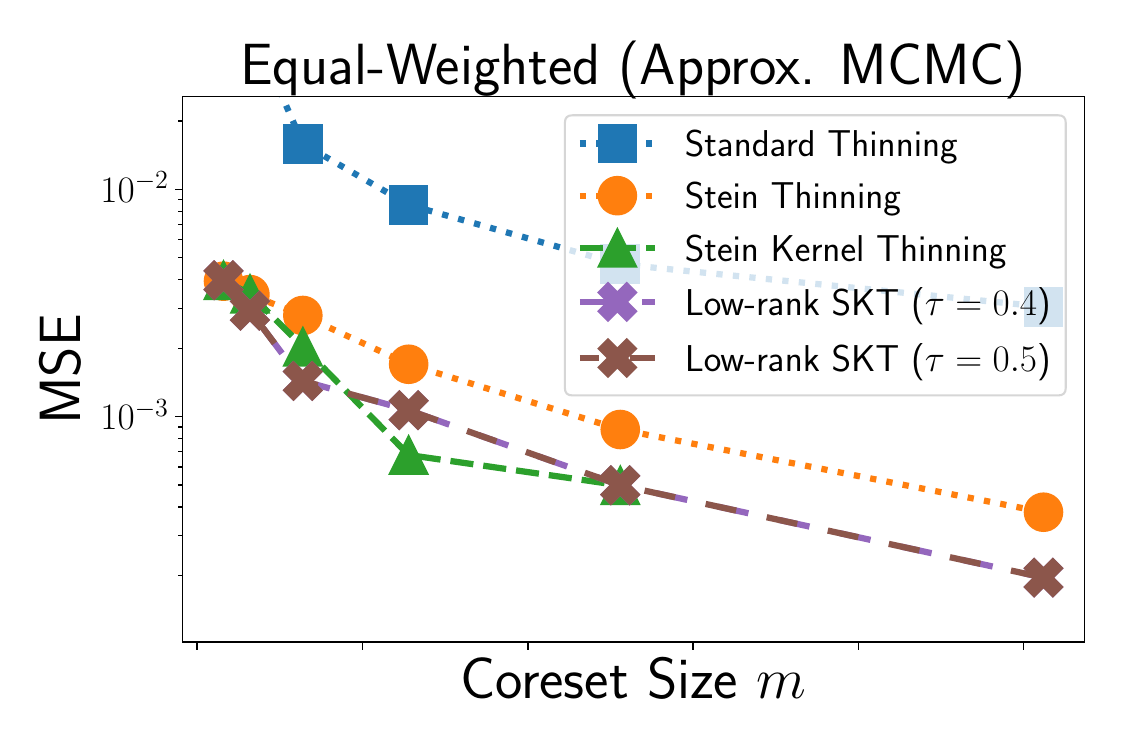}
    \includegraphics[width=0.3\textwidth]{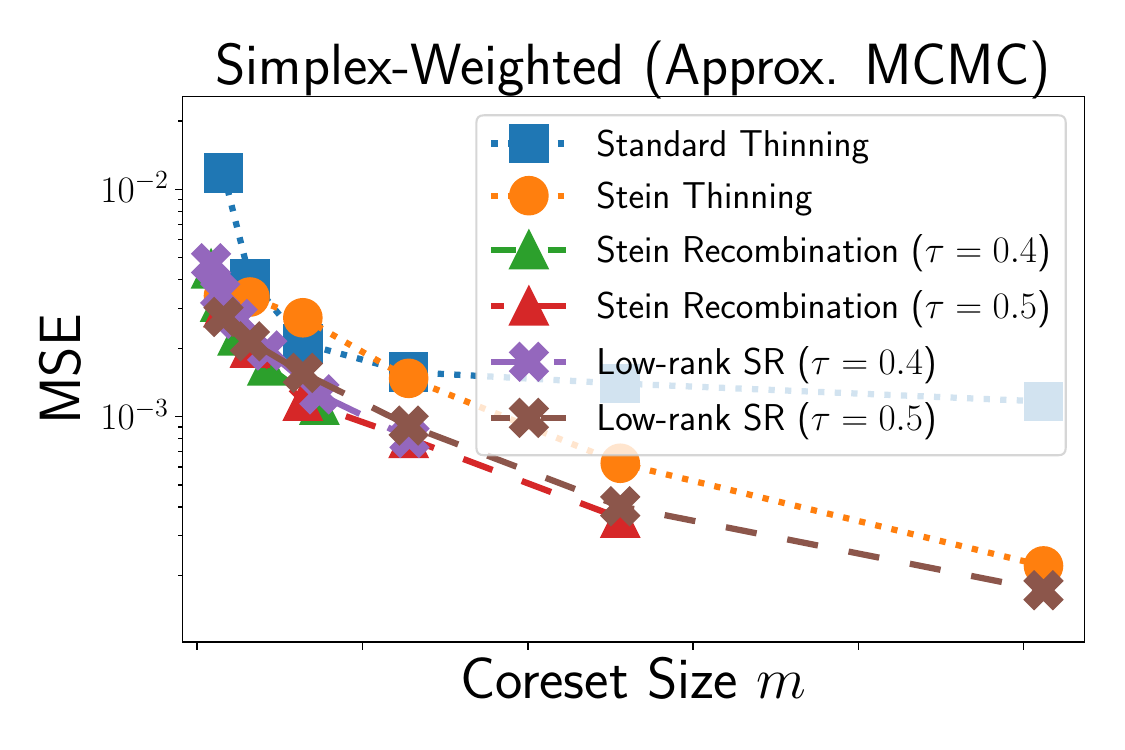}
    \includegraphics[width=0.3\textwidth]{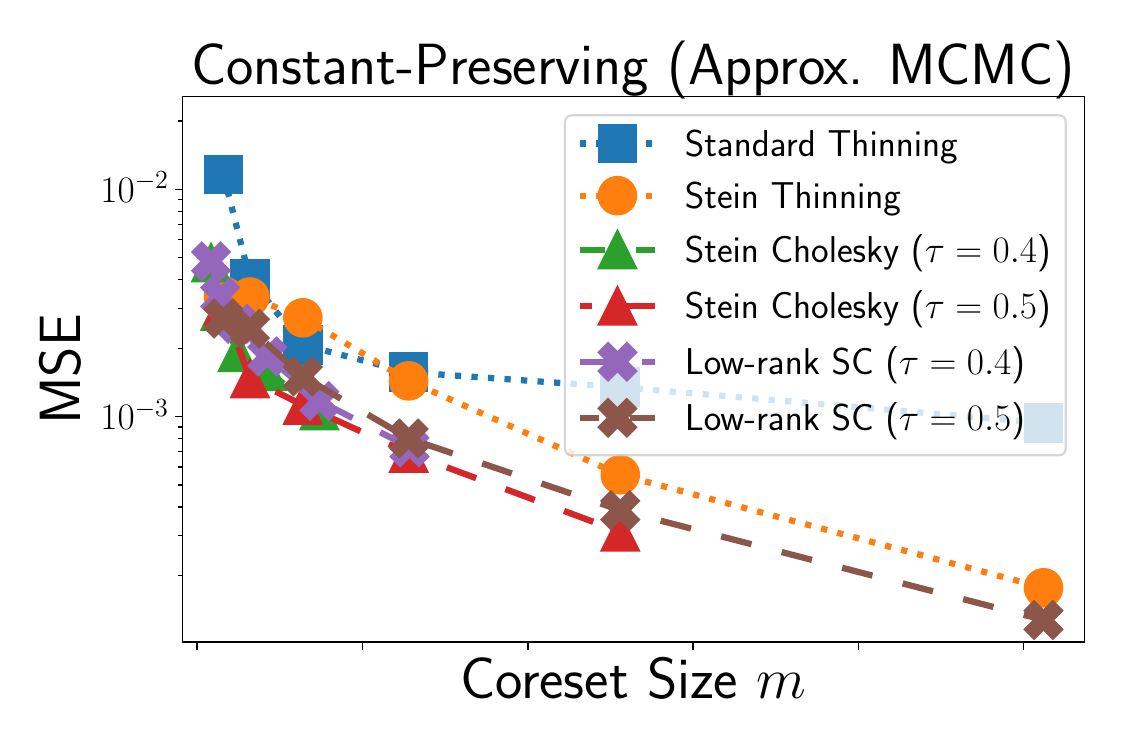}
    \caption{\textbf{Posterior mean mean-squared error (MSE)} for the approximate MCMC compression experiment of \cref{subsec:exp_covtype}.
MSE is computed as $\staticnorm{\hat{\bE}_\bP Z - \sum_{i\in[n_0]} w_i x_i}_M^2/d$ where $\hat{\bE}_\bP Z$ is the mean of the surrogate ground truth NUTS sample.
    }
    \label{fig:covtype_mse}
\end{figure}

\subsection{Correcting for tempering details}
In the data release of \citet{DVN/MDKNWM_2020}, we noticed there were 349 sample points for which the provided scores were NaNs, so we removed those points at the recommendation of the authors.

\end{document}